%% file: main.tex
\documentclass[10pt,twocolumn,letterpaper]{article}

\usepackage[pagenumbers]{cvpr} 

\input{preamble}

\definecolor{cvprblue}{rgb}{0.21,0.49,0.74}
\usepackage[pagebackref,breaklinks,colorlinks,allcolors=cvprblue]{hyperref}

\def\paperID{XXXXX}
\def\confName{CVPR}
\def\confYear{2026}

\title{Robust Continual Unlearning \\against Knowledge Erosion and Forgetting Reversal}

\author{Eun-Ju Park$^1$
\qquad
Youjin Shin$^2$\thanks{Corresponding Authors}
\qquad
Simon S. Woo$^{1,3}$\footnotemark[1]\\
$^1$Sungkyunkwan University \\
$^2$The Catholic University of Korea\\
$^3$Secure Machines Lab\\
{\tt\small 
\{eunju.park, swoo\}@g.skku.edu,
yj.shinn@catholic.ac.kr}
}

\begin{document}

\maketitle
\input{sec/0_abstract}    
\input{sec/1_intro}
\input{sec/2_relwork}
\input{sec/3_prelim}
\input{sec/4_method}

\input{sec/5_experiment}

\input{sec/6_analysis}

\input{sec/7_conclusion}

{
    \small

}

\input{sec/8_suppl}

\end{document}

%% file: preamble.tex
\usepackage{kotex}   
\usepackage{subcaption}
\usepackage[table]{xcolor}
\usepackage{multirow}
\usepackage{booktabs}

\usepackage{amsmath}
\usepackage{amsfonts}
\usepackage{mathtools}
\usepackage{xcolor}
\usepackage{amsthm}
\usepackage{comment}
\usepackage{soul}
\usepackage{accents}

\newtheorem{proposition}{Proposition}



%% file: sec/0_abstract.tex
\begin{abstract}
As a means to balance the growth of the AI industry with the need for privacy protection, machine unlearning plays a crucial role in realizing the ``right to be forgotten'' in artificial intelligence. This technique enables AI systems to remove the influence of specific data while preserving the rest of the learned knowledge.
Although it has been actively studied, most existing unlearning methods assume that unlearning is performed only once. In this work, we evaluate existing unlearning algorithms in a more realistic scenario where unlearning is conducted repeatedly, and in this setting, we identify two critical phenomena:
(1) Knowledge Erosion, where the accuracy on retain data progressively degrades over unlearning phases, and
(2) Forgetting Reversal, where previously forgotten samples become recognizable again in later phases.
To address these challenges, we propose SAFER (StAbility-preserving Forgetting with Effective Regularization), a continual unlearning framework that maintains representation stability for retain data while enforcing negative logit margins for forget data.
Extensive experiments show that SAFER mitigates not only knowledge erosion but also forgetting reversal, achieving stable performance across multiple unlearning phases.
Source code: \url{https://github.com/DASH-Lab/SAFER}
\end{abstract}


%% file: sec/1_intro.tex
\section{Introduction}
\label{sec:intro}

With large-scale datasets, deep learning has achieved significant progress across a wide range of computer vision tasks, including classification~\cite{krizhevsky2012imagenet,deng2009imagenet,he2016deep,dosovitskiy2021an}.
While applying such technology to real-world problems can bring significant convenience to daily life, it inevitably requires collecting and processing more personal data, thereby intensifying concerns over privacy violations~\cite{bourtoule2021machine,carlini2021extracting}.
To address these concerns, data protection regulations such as the GDPR~\cite{GDPR} guarantee individuals the ``right to be forgotten,'' which grants them the right to request the deletion of information related to themselves.
This extends beyond merely removing the original data from databases; it also encompasses erasing the influence of that data from models that have already been trained on it.

Therefore, deep learning systems should be equipped with the capability to selectively forget specific user data. 
Enabling a model to forget certain information is essential not only for addressing privacy concerns but also for correcting unintended learning outcomes resulting from erroneous data during training. 
In this respect, machine unlearning has become an important research area for ensuring the reliability, security, and fairness of AI systems.

Most prior studies on machine unlearning~\cite{foster2024fast,fan2024salun, chen2023boundary,kurmanji2023towards} implicitly assume that the unlearning operation is processed only once.
In real-world applications, however, data deletion requests may arise continuously over time, requiring models to support ongoing unlearning throughout their entire life cycle.
In other words, models should be capable of maintaining their utility even when data removal is performed repeatedly.
Thus, developing robust and efficient approaches for continual unlearning is essential to ensure both regulatory compliance and long-term model reliability.

In our analysis of applying existing unlearning methods to multi-phase unlearning, we identify two critical issues that undermine both model reliability and privacy compliance.
The first issue, \textit{Knowledge Erosion}, refers to the progressive degradation of accuracy on retain data as unlearning proceeds across multiple phases. 
The second issue, \textit{Forgetting Reversal}, describes the phenomenon in which samples that were previously forgotten become recognizable again in later unlearning phases. 

To mitigate these issues, we propose SAFER (StAbility-preserving Forgetting with Effective Regularization), a continual unlearning framework designed to maintain model utility while preventing previously forgotten samples from being recognized again in later unlearning phases.
SAFER simultaneously enhances the clusterability of retain data representations by reducing intra-class variation and increasing inter-class separation, and encourages forgotten samples to maintain negative unlearning margins across phases.
By combining these two mechanisms, SAFER reduces the likelihood that forgotten samples re-enter class decision regions as correct predictions as the unlearning process repeats.
In brief, our contributions can be summarized as follows:
\begin{itemize}
    \item We propose \textbf{SAFER}, a continual unlearning framework that enhances the clusterability of retain data representations, while encouraging forgotten samples to maintain negative unlearning margins across phases. This mechanism prevents previously forgotten information from being reactivated even after multiple unlearning phases.
    \item Through comprehensive evaluations on diverse datasets and continual unlearning configurations, we show that SAFER maintains strong model utility while robustly preventing forgetting reversal, outperforming existing state-of-the-art approaches.
    \item We provide detailed analyses of representation clusterability and unlearning-margin behaviors to uncover the factors that enable reliable continual unlearning.
\end{itemize}


%% file: sec/2_relwork.tex
\section{Related Work}
\label{sec:relwork}

Cao and Yang~\cite{cao2015towards} introduced the notion of machine unlearning, focusing on the challenge of removing a subset of training data from a trained model. 
Earlier research on machine unlearning, including SISA~\cite{bourtoule2021machine} and Amnesiac learning~\cite{graves2021amnesiac}, explored strategies for training models under the assumption that data deletion occurs. 
More recent studies~\cite{chen2023boundary,fan2024salun,foster2024fast} have instead focused on approaches that modify the weights of an already trained model to minimize the influence of the forget data.

This category of unlearning methods, known as \textit{approximate unlearning}, aims to statistically ensure that an unlearned model obtained through post-hoc methods remains indistinguishable from models that achieve exact unlearning, such as retrained models. 
In other words, approximate unlearning is regarded as effective when it achieves performance comparable to retrained models, under comprehensive evaluation criteria.

These methods typically fine-tune the model to reduce the influence of the forget data~\cite{chundawat2023can,kurmanji2023towards,fan2024salun,chen2023boundary} or estimate the degree to which model parameters are affected when updating forget data or train data and adjust them accordingly~\cite{foster2024fast,golatkar2020eternal}. 
However, approximate unlearning does not inherently guarantee the complete removal of the forget samples’ influence.
Therefore, its effectiveness is typically evaluated by comparison with retrained models.
Even though most unlearning studies achieve high levels of forgetting, they generally assume that the unlearning process is performed only once and evaluate their methods using forget data as a single batch.

In this work, we study continual unlearning, also referred to as iterative unlearning~\cite{grimes2024gone} or online unlearning~\cite{xu2024machine}.
Recent works~\cite{chatterjee2024unified,adhikari2025unlearning,huang2025unified,tang2025acu} consider scenarios where learning and unlearning occur at different stages of the training process, or where, after multi-stage learning, unlearning removes knowledge acquired at a specific learning stage.
In contrast, our setting assumes a fully trained model, after which unlearning requests may arise on arbitrary data in either class-aligned or class-misaligned settings. Thus, unlearning is not tied to a specific training phase.
GS-LoRA~\cite{Zhao_2024_CVPR,zhao2026practical} presents a LoRA-based approach that adapts Transformer-specific FFN modules to address continual forgetting.
In the context of continual unlearning, we characterize two key issues, \textit{Knowledge Erosion} and \textit{Forgetting Reversal}, and propose corresponding metrics to quantify them. 
Furthermore, we introduce an architecture-agnostic framework that addresses these challenges.

%% file: sec/3_prelim.tex
\section{Preliminaries}
\label{sec:prelim}

\subsection{Problem Setting and Notation}
\label{subsec:problem}

\begin{figure}[!tbp]
    \centering
    \includegraphics[width=\columnwidth]{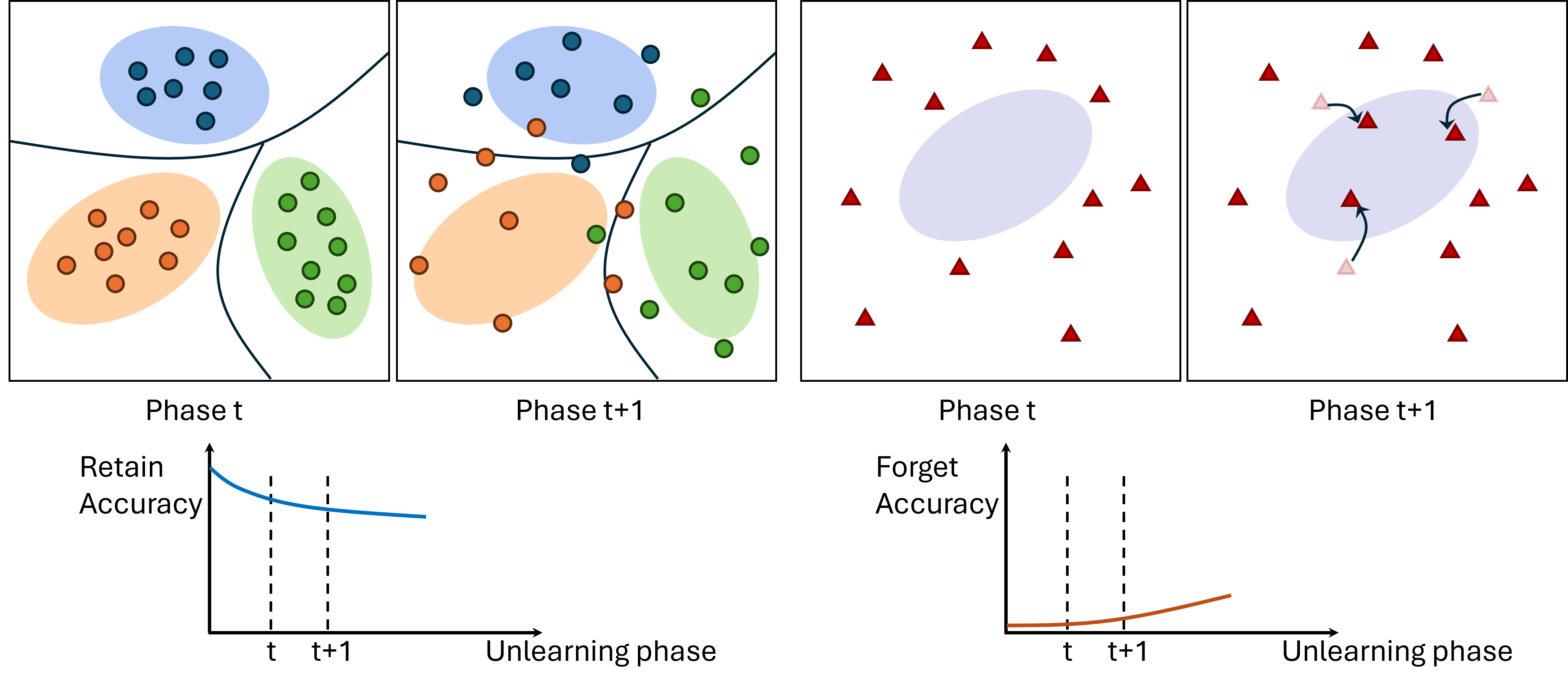}
    \caption{Illustration of knowledge erosion and forgetting reversal. The left figure illustrates knowledge erosion, where the model's performance on the retain data progressively degrades as unlearning is repeated. The right figure depicts forgetting reversal, where previously forgotten samples become recognizable again in later unlearning phases.}
    \label{fig:problem}
    \vspace{-1.em}
\end{figure}



Let the entire training dataset used for a classification model \(f_{\theta} : \mathcal{X} \rightarrow \mathcal{Y}\) be \(\mathcal{D} = \{(\boldsymbol{\mathrm{x}}^{(i)}, \mathrm{y}^{(i)})\}_{i=1}^N \subseteq \mathcal{X} \times \mathcal{Y}\), where each input \(\boldsymbol{\mathrm{x}}^{(i)} \in \mathcal{X}\) maps to a label \(\mathrm{y}^{(i)} \in \mathcal{Y}\). 
The set of labels is denoted as \(\mathcal{Y} = \{l_1, l_2, ..., l_K\}\) where \(K\) is the number of classes. 

We denote the original classification model trained on \(\mathcal{D}\) as \(f_{\theta_0}\).
When a subset of the training data should be unlearned from \(f_{\theta_0}\) in the first phase, we denote the forgetting data as \(\mathcal{D}_{\text{forget}}^{(1)}\), and the remaining data as \(\mathcal{D}_{\text{retain}}^{(1)} = \mathcal{D} \setminus \mathcal{D}_{\text{forget}}^{(1)}\). 
After the first unlearning, we obtain the unlearned model \(f_{\theta_1}\), where the parameters are updated from \(f_{\theta_0}\) to \(f_{\theta_1}\) during the unlearning process. 

When additional deletion requests arise in the second phase, we denote the new forgetting data as \(\mathcal{D}_{\text{forget}}^{(2)}\), and the remaining data as \(\mathcal{D}_{\text{retain}}^{(2)}\). 
Importantly, the forgetting data from the first phase becomes the forgotten data in the second phase, denoted as \(\mathcal{D}_{\text{forgot}}^{(2)}\).
This implies that the forgotten data has already been removed in the first unlearning phase and therefore remains inaccessible. 
The unlearned model after the second unlearning phase is denoted as \(f_{\theta_2}\), obtained using \(f_{\theta_1}\), \(\mathcal{D}_{\text{forget}}^{(2)}\), and \(\mathcal{D}_{\text{retain}}^{(2)}\).

Let the model at phase \(t\) be denoted as \(f_{\theta_t}\), which has undergone unlearning with respect to a forget set \(\mathcal{D}_{\text{forget}}^{(t)}\).
When a new forget request \(\mathcal{D}_{\text{forget}}^{(t+1)}\) arise, the model \(f_{\theta_t}\) serves as the starting point for the next unlearning phase, producing \(f_{\theta_{t+1}}\).  
The continual unlearning process can therefore be represented as:
\[
f_{\theta_0} \xrightarrow[]{\text{unlearn on } D_{\text{forget}}^{(1)}} 
f_{\theta_1} \xrightarrow[]{\text{unlearn on } D_{\text{forget}}^{(2)}} 
\cdots \xrightarrow[]{\text{unlearn on } D_{\text{forget}}^{(t)}} f_{\theta_t}.
\]
We also denote the relationship among \(\mathcal{D}\), \(\mathcal{D}_{\text{forget}}^{(t)}\), \(\mathcal{D}_{\text{retain}}^{(t)}\), and \(\mathcal{D}_{\text{forgot}}^{(t)}\) as follows:
\begin{equation}
\label{eq:data_def}
    \mathcal{D} = \mathcal{D}_{\text{forget}}^{(t)} \cup \mathcal{D}_{\text{retain}}^{(t)}
    \cup \mathcal{D}_{\text{forgot}}^{(t)},
\end{equation}
where \(\mathcal{D}_{\text{forgot}}^{(1)} = \emptyset\) and \(\mathcal{D}_{\text{forgot}}^{(t)} = \bigcup_{i=1}^{t - 1}\mathcal{D}_{\text{forget}}^{(i)}\).


Let \(f_{\accentset{\ast}{\theta_t}}\) denote a model retrained from scratch on \(\mathcal{D}_{\text{retain}}^{(t)}\). 
The objective for \(f_{\theta_t}\) does not imply that \(\theta_t\) should be similar to \(\accentset{\ast}{\theta_t}\). 
Instead, we expect the behavior of \(f_{\theta_t}\) to be analogous to that of \(f_{\accentset{\ast}{\theta_t}}\). 
Hence, our objective for continual unlearning is to (i) prevent forgotten samples from being classified as their original labels again, while (ii) preserving performance on the retain dataset at any unlearning phase \(t\).
\subsection{Challenges in Continual Unlearning}
\label{subsec:challenges}

Applying existing unlearning algorithms in the continual unlearning setup, as described in~\cref{subsec:problem}, tends to result in two fundamental challenges that threaten both performance and reliability.

\paragraph{(1) Knowledge Erosion.}
We define \textit{Knowledge Erosion} as the phenomenon where the model’s performance on retain data gradually degrades as unlearning is repeated across phases:
\begin{equation}
    \text{Acc}_{t}(\mathcal{D}_{\text{retain}}^{(t)}) > \text{Acc}_{t+k}(\mathcal{D}_{\text{retain}}^{(t+k)}), 
\end{equation}
where \(k \geq 1\). 
Here, \(\text{Acc}_{t}(\cdot)\) and \(\text{Acc}_{t+k}(\cdot)\) denote the accuracy measured after completing the \(t\)-th and (\(t+k\))-th unlearning phase, using the models \(f_{\theta_t}\) and \(f_{\theta_{t+k}}\), respectively.
This degradation suggests that the model’s ability to preserve the discriminative characteristics of retain data weakens as unlearning is repeated. 
As a result, the model becomes increasingly unstable and less reliable in continual unlearning scenarios.

\paragraph{(2) Forgetting Reversal.}
We identify another critical issue, \textit{Forgetting Reversal}, in which samples that were successfully forgotten in earlier phases become recognizable again in subsequent ones:
\begin{equation}
    \text{Acc}_{t}(\mathcal{D}_{\text{forget}}^{(t)}) < \text{Acc}_{t+k}(\mathcal{D}_{\text{forget}}^{(t)}),
\end{equation}
where \(k \geq 1\).
This implies that during later unlearning optimization, the representations of previously forgotten data drift toward newly adjusted decision boundaries, causing them to re-enter recognizable regions. 

These two phenomena, \textit{Knowledge Erosion} and \textit{Forgetting Reversal}, represent distinct yet interrelated challenges in continual unlearning.  
The former pertains to the stability of retained knowledge, whereas the latter concerns the reliability of forgetting.  
Consequently, an effective continual unlearning algorithm should aim to achieve both:
(i) robust preservation of knowledge learned from \(\mathcal{D}_{\text{retain}}\), and  
(ii) consistent unrecognizability of both \(\mathcal{D}_{\text{forget}}\) and \(\mathcal{D}_{\text{forgot}}\).  
In the next section, we present a method that explicitly addresses these challenges by optimizing intra/inter-class feature variation for retain data while enforcing negative logit margins for forget data.

%% file: sec/4_method.tex
\section{Method}
\label{sec:method}

In this section, we propose SAFER (\textit{StAbility-preserving Forgetting with Effective Regularization}), 
a continual unlearning framework designed to address two critical issues: 
\textit{Knowledge Erosion}, the progressive degradation of performance on retain data across unlearning phases, and 
\textit{Forgetting Reversal}, the unintended reclassification of previously forgotten data as correct predictions. 

To overcome these challenges, continual unlearning should satisfy two key requirements.
First, as unlearning is repeated, the model should preserve the structural integrity of retain data representations:
maintaining compact intra-class distributions while keeping distinct class boundaries prevents accumulated representation drift and helps sustain predictive utility over phases.
Second, forgotten samples should remain unrecognizable in future updates so that forgetting does not unintentionally reverse.
This motivates constraining the model’s confidence for the ground-truth class of forget data to remain lower than that of retain classes, even as unlearning continues across phases.

To effectively counteract the two challenges, SAFER introduces two optimization objectives:
(i) improving the clusterability of retain data representations to protect retained knowledge from degradation, and
(ii) encouraging negative logit margins for forget data to suppress the possibility of their re-entering decision regions in later phases.

\paragraph{Improving Representation Clusterability.}
To mitigate knowledge erosion, SAFER aims to maintain and reinforce the clusterability of retain data representations across phases. Our key idea is to encourage (i) low intra-class variation and (ii) well-separated cluster structures over unlearning phases.

We introduce a class-conditioned latent-variable module that operates exclusively on retain samples. Given a feature vector \(x \in \mathbb{R}^d\) with class label \(c\), the encoder produces the parameters of a latent distribution conditioned on the input feature \(x\) and class label \(c\), from which a latent embedding is sampled as:
\begin{equation}
    z = \mu + \sigma \odot \epsilon, \quad \epsilon \sim \mathcal{N}(0,I),
\end{equation}
where \(\mu\) and \(\sigma\) are the mean and standard deviation of the latent distribution. A decoder reconstructs the feature from \(z\) as \(\hat{x}\), and the stabilized feature forwarded to the classifier is computed as:
\begin{equation}
    x' = \frac{x + \hat{x}}{2}.    
\end{equation}
Averaging the original and reconstructed representations encourages feature smoothness and consistency between encoder and decoder, reducing noisy intra-class variability.
We then optimize retain samples with the following objective:
\begin{equation}
\label{eq:retain_loss}
\begin{aligned}
    \mathcal{L}_{\text{retain}} &=
    \underbrace{\mathcal{L}_{\text{CE}}(f(x'), c)}_{\text{classification consistency}} 
    + \underbrace{\| x - \hat{x} \|^2}_{\text{intra-class compactness}} \\
    & + \underbrace{\mathrm{KL}(q(z|x,c)\parallel\mathcal{N}(0,I))}_{\text{smooth and bounded latent structure}} 
    + \underbrace{\lambda (\mathbb{E}[\| \mu - \mu^{\text{EMA}} \|_2])^{-1}}_{\text{cluster separation}}, 
\end{aligned}
\end{equation}
where \(\mu^{\text{EMA}}\) denotes a global exponential moving average of latent means computed over mini-batches.
The final term encourages separation between latent representations and the global mean, preventing collapse of class clusters.
The regularization terms in~\cref{eq:retain_loss} reduce intra-class variability and maintain well-separated cluster structures across unlearning phases, thereby preserving retain-data utility.

\paragraph{Negative Logit Margin for Forget Data.}
In continual unlearning, our objective is to ensure that forgotten data remain unrecognizable even after multiple phases.
To characterize this behavior, we introduce the unlearning margin, inspired by the concept of the logit margin~\cite{ngnawe2024detecting}.
For a sample $\mathbf{x}$ with original class $y$, we define
\begin{equation}
\label{eq:unlearning_margin}
    UM(\mathbf{x}) = \ell_y - \max_{k \neq y} \ell_k,
\end{equation}
where \(\ell_y\) is the logit of the original class and \(\max_{k \neq y}\ell_k\) is the maximum competing logit.

SAFER encourages negative unlearning margins for forget samples by suppressing the logit of the original class and assigning probability mass to retain classes.
For each forget sample at phase \(t\), we define a target vector \(\mathbf{q}\) such that only retain classes receive non-zero supervision: 
\begin{equation}
q_i =
\begin{cases}
0, & i \notin \mathcal{D}_{\text{retain}}^{(t)}, \\
r_i, & i \in \mathcal{D}_{\text{retain}}^{(t)}, 
\end{cases}
\end{equation}
where \(r_i\) is independently sampled from a uniform distribution over [0, 1]. Then, the target vector is normalized as: 
\begin{equation}
\label{eq:random_logitdist}
q_i \leftarrow \frac{q_i}{\sum_j q_j}.
\end{equation}

Since the original label \(y\) of a forget sample does not belong to \(\mathcal{D}_{\text{retain}}^{(t)}\), it receives zero target probability.
Minimizing KL divergence using \(\mathbf{q}\) drives the model toward
\begin{equation}
\label{eq:negative_logit_margin_objective}
    \ell_y < \max_{i \in \mathcal{D}_{\text{retain}}^{(t)}} \ell_i
    \quad \Rightarrow \quad UM(\mathbf{x}) < 0.
\end{equation}
As a result, forgotten classes remain suboptimal in the decision space throughout unlearning phases, preventing inadvertent reclassification of forgotten samples.
A formal analysis of the objective based on the negative logit margin is provided in~\cref{sec:lms_theory}.

%% file: sec/5_experiment.tex
\section{Experiment}
\label{sec:experiment}

\subsection{Experimental Settings}
\label{subsec:setup}

\noindent\textbf{Dataset and Models.}
For unlearning evaluation, we use three datasets: CIFAR-100~\cite{krizhevsky2009learning}, VGGFace2~\cite{cao2018vggface2datasetrecognisingfaces}, and MUFAC~\cite{choi2024towards}.
CIFAR-100 is a benchmark dataset comprising 100 object classes, each containing 600 images, of which 500 are used for training and 100 for testing.
VGGFace2 is a facial dataset containing 100 identities, with the number of images per identity ranging from 87 to 843.
Lastly, MUFAC consists of facial images from a total of 1,216 subjects, with 8 to 24 photographs per individual, taken at different ages and categorized into eight distinct age groups. 
We use ResNet-18, ResNet-50~\cite{he2016deep}, and Vision Transformer (ViT)~\cite{dosovitskiy2021an} as backbone models. 


\input{tables/table_acc_c100}

\input{tables/table_acc_vgg2}

\input{tables/table_acc_mufac}

\noindent\textbf{Continual Unlearning Setup.}
In our experiments, we evaluate the performance of unlearning algorithms under scenarios where the unlearning process is performed multiple times.
To simulate diverse continual unlearning situations, we consider two settings based on the alignment between the task label and the forgetting unit:
\begin{itemize}
\item \textit{Class-Aligned Continual Unlearning}: the forgetting unit directly corresponds to the model’s classification labels (CIFAR-100, VGGFace2).
\item \textit{Class-Misaligned Continual Unlearning}: unlearning requests occur at the individual level, while the task objective is attribute classification, resulting in a misalignment between the forgetting unit and the classification target (MUFAC).
\end{itemize}
For each setting, we conduct experiments under the \textit{N CLASSES-3PHASES} configuration, which consists of three consecutive unlearning processes with \textit{N} different classes.
All experiments begin with a pre-trained classification model \(f_{\theta_0}\).
At each phase \textit{t}, a new deletion request \(\mathcal{D}_{\text{forget}}^{(t)}\) is processed using \(f_{\theta_{t-1}}\), \(\mathcal{D}_{\text{forget}}^{(t)}\), and \(\mathcal{D}_{\text{retain}}^{(t)}\), resulting in an updated model \(f_{\theta_t}\).

\begin{figure*}[!tbp]
\centering
    \begin{subfigure}{\textwidth} 
        \centering 
        \includegraphics[width=0.33\textwidth]{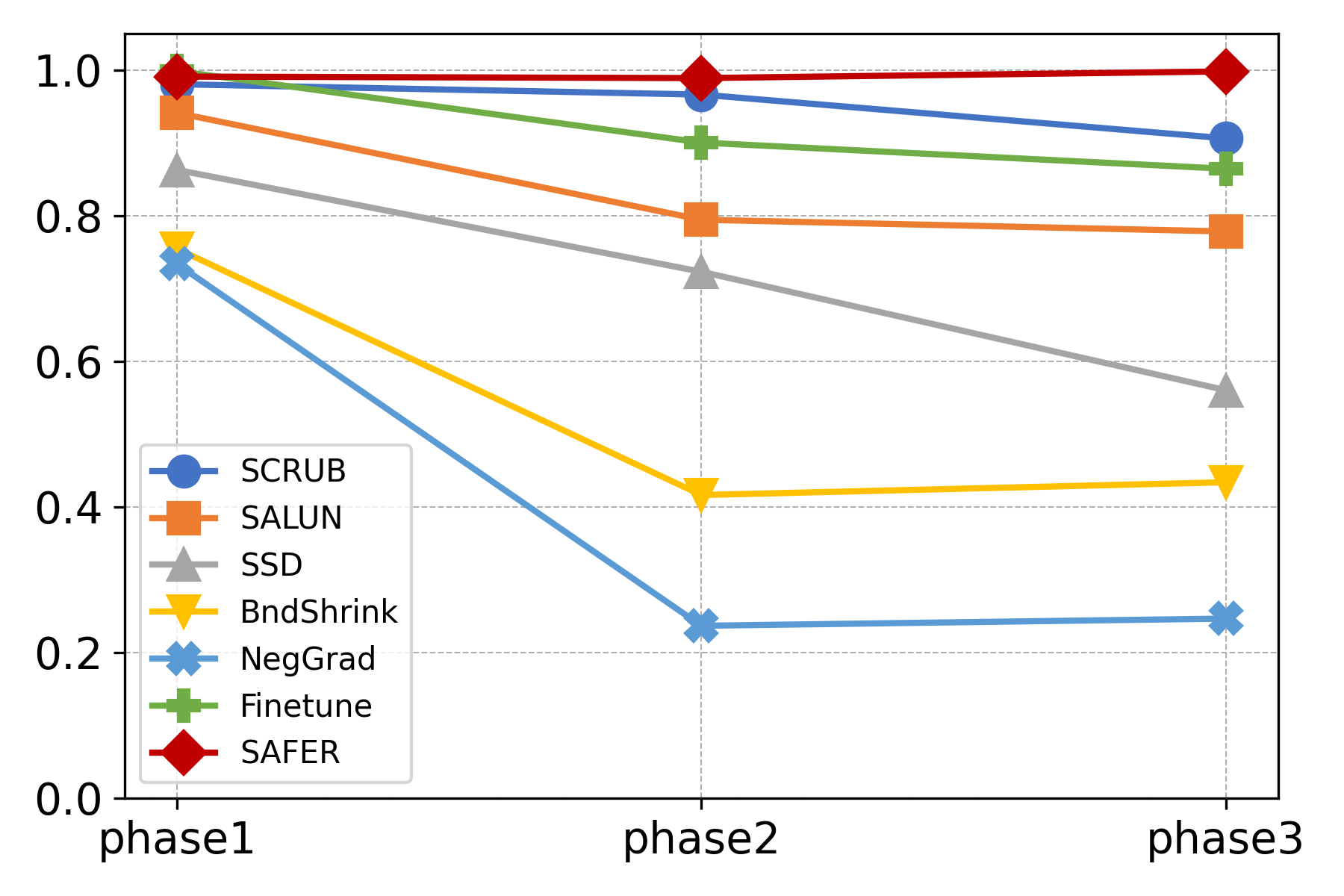}
        \includegraphics[width=0.33\textwidth]{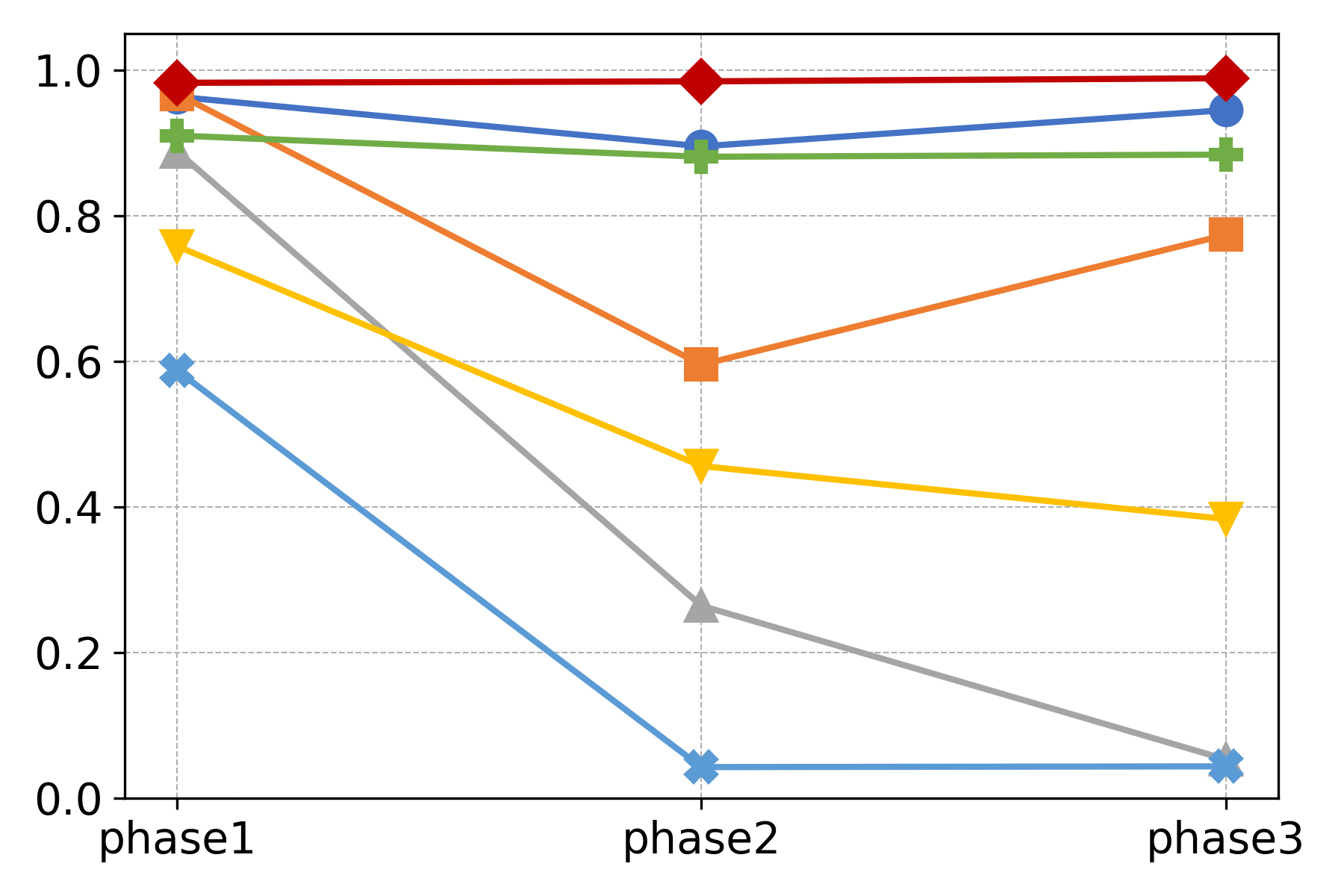}
        \includegraphics[width=0.33\textwidth]{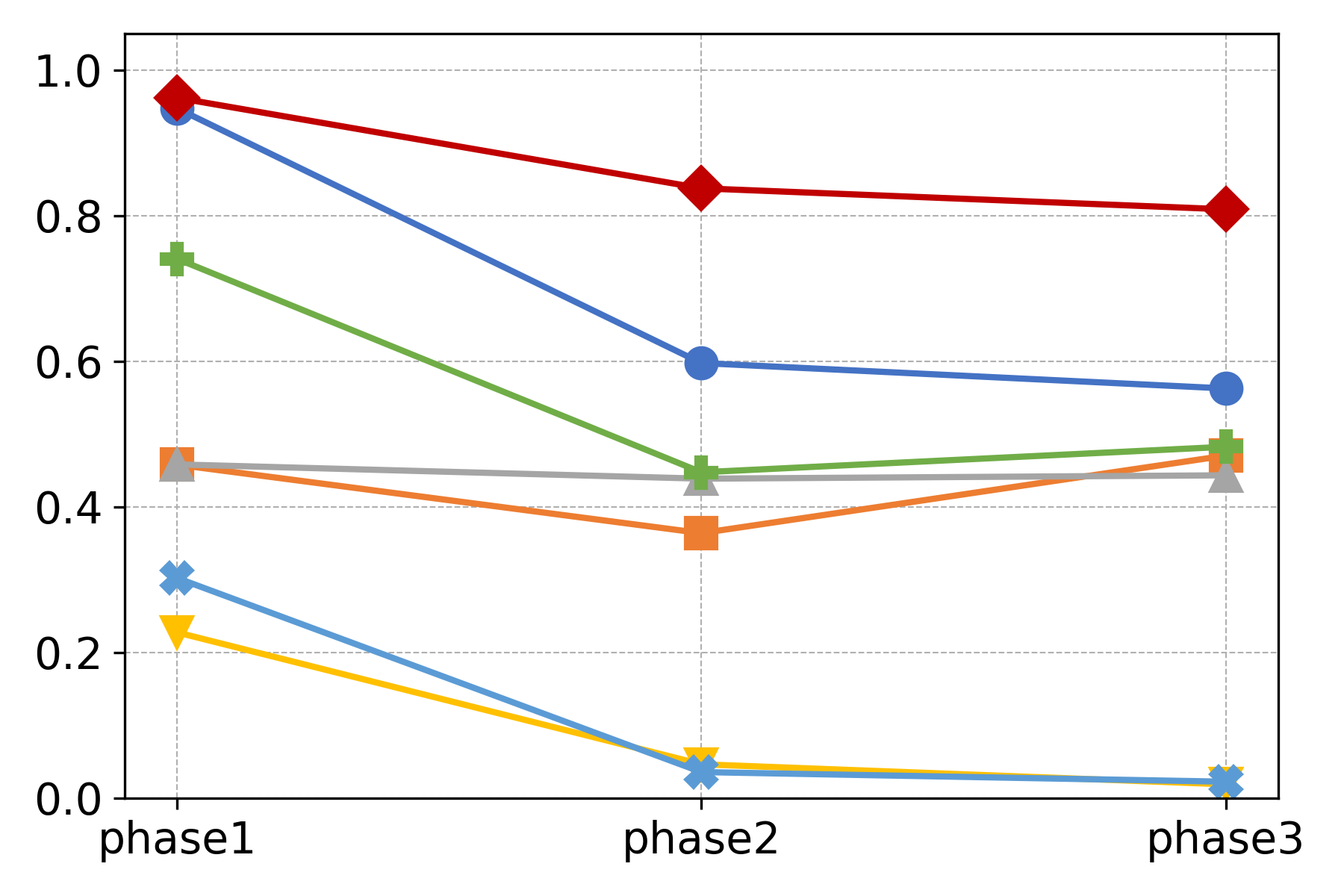}
        \caption{Unlearning Efficiency (ToW)}
        \label{fig:tow_all}
    \end{subfigure}

    \hfill

    \begin{subfigure}[b]{0.49\textwidth}
    \includegraphics[width=\textwidth]{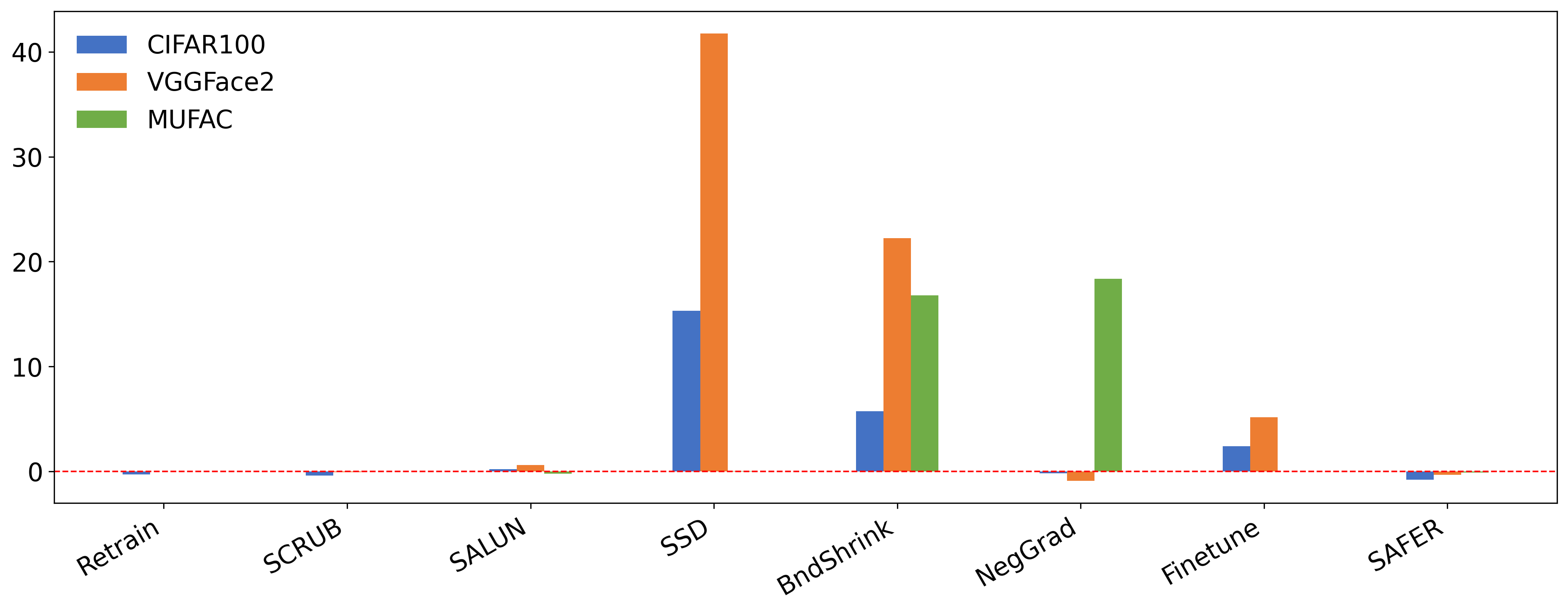}
    \caption{Knowledge Erosion} 
    \label{fig:ke_all}
    \end{subfigure}
    \begin{subfigure}[b]{0.49\textwidth}
    \includegraphics[width=\textwidth]{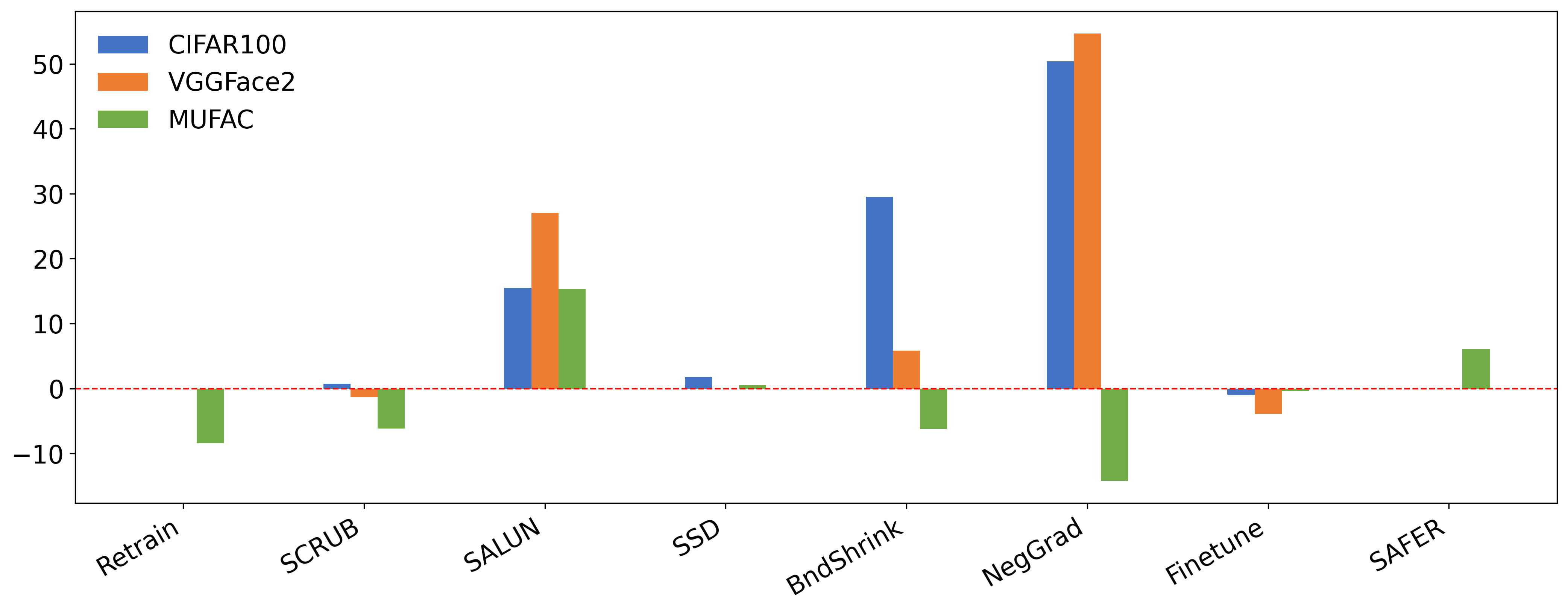}
    \caption{Forgetting Reversal} 
    \label{fig:fr_all}
    \end{subfigure}
    
\caption{Measured continual unlearning metrics. In (a), ToW results, shown from left to right for CIFAR-100, VGGFace2, and MUFAC, indicate that higher values are better. For both (b) knowledge erosion and (c) forgetting reversal, lower values indicate better performance.}
\label{fig:continual_measurement}
\end{figure*}


\noindent\textbf{Evaluation.}
At each unlearning phase, we measure and report the accuracy using the following evaluation sets: 
\begin{itemize}
\item \textit{Class-Aligned Continual Unlearning}: \(\mathcal{D}_{tr}\) (retain class data of test set), \(\mathcal{D}_{tf}\) (forget class data of test set), and \(\mathcal{D}_{tg}\) (forgotten class data of test set).
\item \textit{Class-Misaligned Continual Unlearning}: \(\mathcal{D}_{r}\) (retain data), \(\mathcal{D}_{f}\) (forget data), \(\mathcal{D}_{t}\) (test data), and \(\mathcal{D}_{g}\) (forgotten data).
\end{itemize}
We note that the measurement of forgotten data is included starting from the second unlearning process.

To provide a concise and comparable assessment across multiple unlearning phases, we introduce two metrics that quantify the two key phenomena observed in continual unlearning: Knowledge Erosion and Forgetting Reversal.

\noindent\textit{Measurement of Knowledge Erosion.} 
Let \(\text{Acc}_{t}(\mathcal{D}_{\text{retain}}^{(t)})\) denote the accuracy on retain data measured after completing the \textit{t}-th unlearning phase. 
Knowledge erosion characterizes the progressive decline of predictive performance on retain data across phases. 
Accordingly, we define
\begin{equation}
\label{eq:m_ke}
    m(\mathrm{KE}) \coloneqq \frac{1}{T-1}\sum_{t=1}^{T-1}\Big(\text{Acc}_{t}(\mathcal{D}_{\text{retain}}^{(t)}) - \text{Acc}_{t+1}(\mathcal{D}_{\text{retain}}^{(t+1)})\Big).
\end{equation}
Since smaller changes in the performance on retain data indicate that the model better preserves its utility as unlearning progresses, lower values correspond to more stable performance.

\noindent\textit{Measurement of Forgetting Reversal.}
We measure the extent to which samples forgotten at phase-(\textit{t}) become recognizable in subsequent phases.
To quantify forgetting reversal, we compare \(\text{Acc}_{t}(\mathcal{D}_{\text{forget}}^{(t)})\) with \(\text{Acc}_{t+1}(\mathcal{D}_{\text{forgot}}^{(t+1)})\), since 
\(\mathcal{D}_{\text{forgot}}^{(t+1)}\) contains \(\mathcal{D}_{\text{forget}}^{(t)}\), as described in~\cref{eq:data_def}.  
Thus, we define
\begin{equation}
\label{eq:m_fr}
    m(\mathrm{FR}) \coloneqq \frac{-1}{T-1}\sum_{t=1}^{T-1}\Big(\text{Acc}_{t}(\mathcal{D}_{\text{forget}}^{(t)}) - \text{Acc}_{t+1}(\mathcal{D}_{\text{forgot}}^{(t+1)})\Big).
\end{equation}
Since it is desirable for the performance on forget data to remain low even after multiple unlearning processes, lower values indicate more robust unlearning efficiency.

In addition, we adopt the Tug of War (ToW) metric to comprehensively rank unlearning methods for each phase and dataset, following~\cite{zhao2024makes}. In brief, ToW quantifies the performance discrepancy between the unlearned model (\(f_{\theta_t}\)) and the retrained model (\(h_\theta\)) as follows:
\begin{equation}
    ToW(t) = \Pi_{\mathcal{D}^{(t)}}(1 - |\text{Acc}(\mathcal{D}^{(t)}; f_{\theta_t}) - \text{Acc}(\mathcal{D}^{(t)}; h_\theta)|).
\end{equation}


\noindent\textbf{Baselines.}
We compare our proposed method against various state-of-the-art baselines for continual unlearning.
The baselines used in our experiments are as follows. \textit{Retrain} is a model retrained from scratch on \(\mathcal{D}_{\text{retain}}\) and serves as an oracle reference for continual unlearning.
\textit{SCRUB}~\cite{kurmanji2023towards} optimizes the distance between the student and teacher models, decreasing it for the retain data while increasing it for the forgetting data.
\textit{SALUN}~\cite{fan2024salun} identifies model weights where the gradient of the forgetting loss exceeds a threshold and applies random labeling to update those selected weights.
\textit{SSD}~\cite{foster2024fast} selectively updates weights of \(f_{\theta}\) that are important for the forgetting data but not for the retain data, using the Fisher information matrix. 
\textit{BndShrink}~\cite{chen2023boundary} fine-tunes \(f_{\theta}\) with the nearest but incorrect labels for the forgetting data, determined through adversarial perturbations.
\textit{NegGrad}~\cite{golatkar2020eternal} fine-tunes the model in the direction that increases the loss for \(\mathcal{D}_{\text{forget}}\). 
\textit{Finetune}~\cite{golatkar2020eternal} updates the weights of \(f_{\theta}\) to minimize the loss only on the retain data.



\subsection{Results}
\label{subsec:results}

\begin{figure*}[!tbp]
\centering
    \begin{subfigure}[b]{0.32\textwidth} 
        \centering 
        \includegraphics[width=\textwidth]{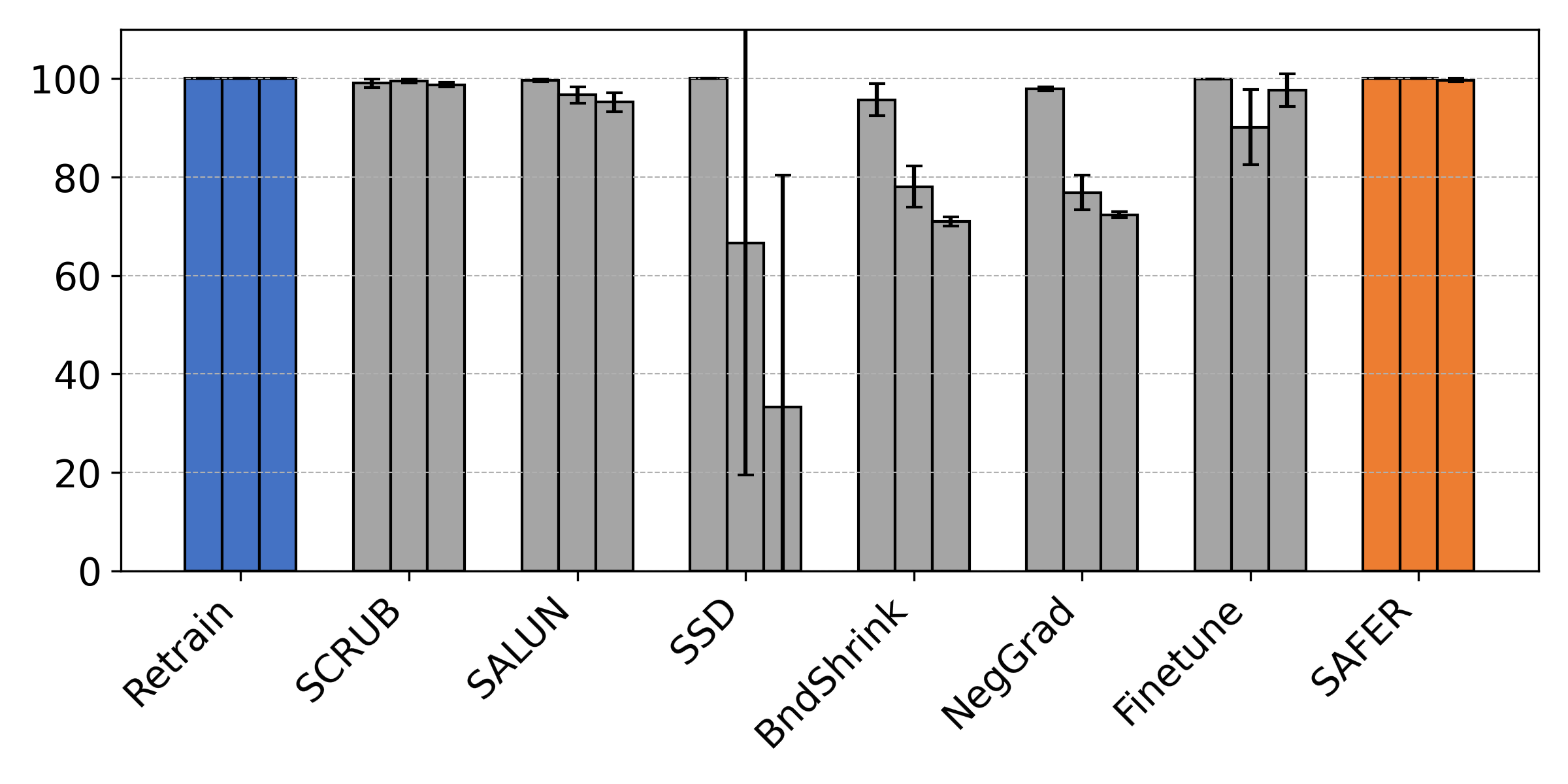}
        \caption{CIFAR-100}
        \label{fig:mia_c100}
    \end{subfigure}
    \hfill
    \begin{subfigure}[b]{0.32\textwidth} 
        \centering 
        \includegraphics[width=\textwidth]{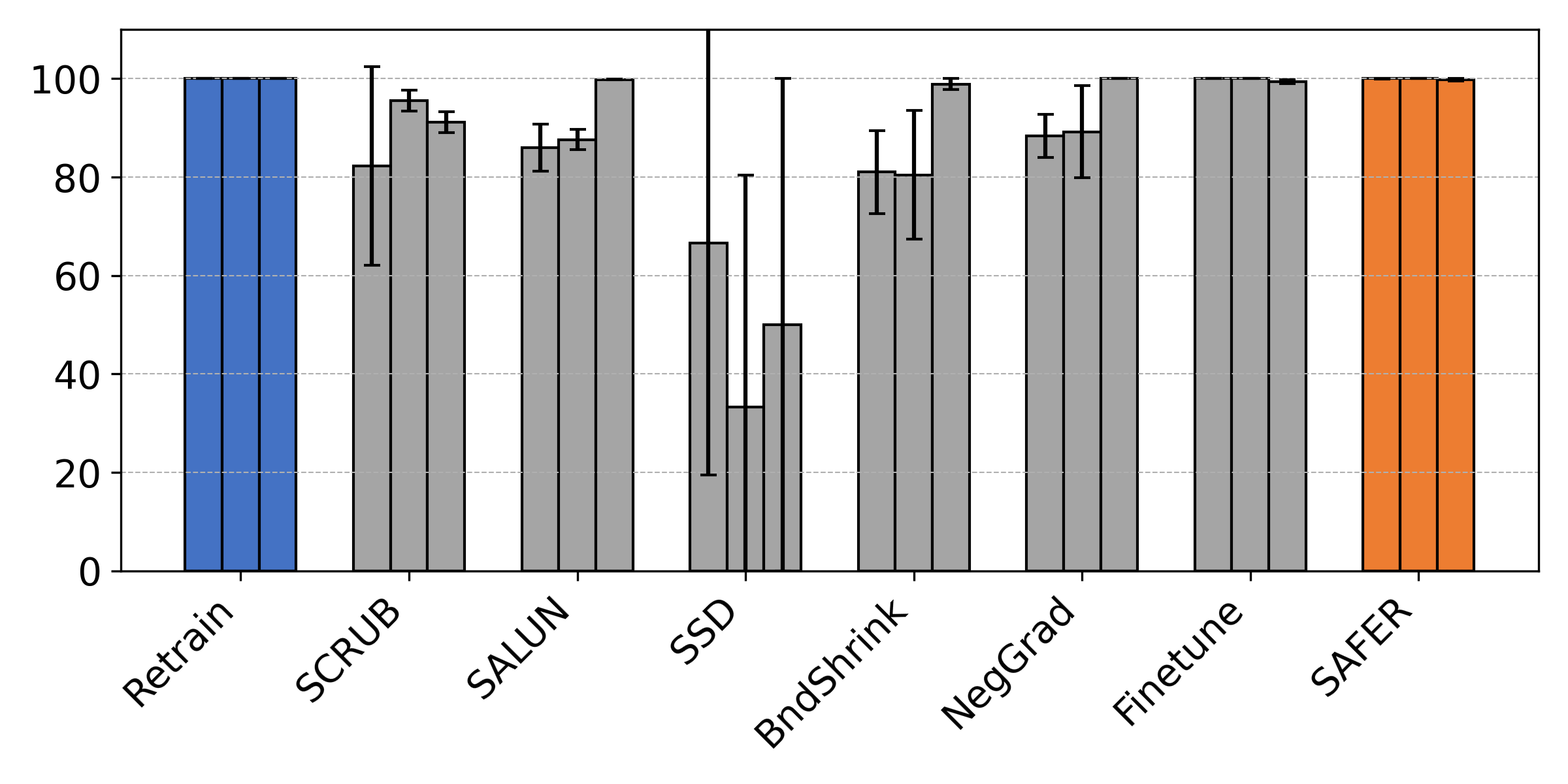}
        \caption{VGGFace2}
        \label{fig:mia_vgg2}
    \end{subfigure}
    \hfill
    \begin{subfigure}[b]{0.32\textwidth} 
        \centering 
        \includegraphics[width=\textwidth]{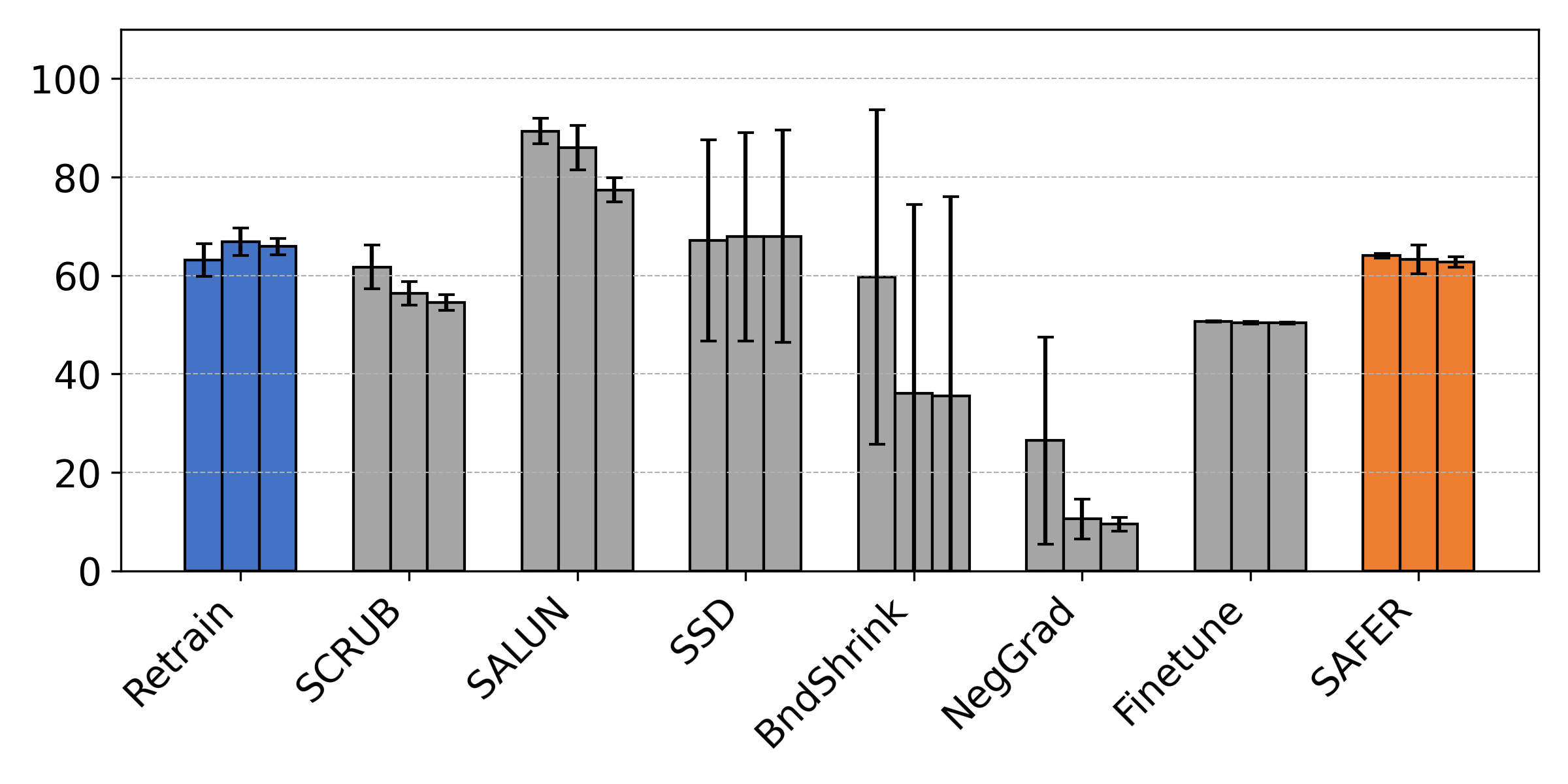}
        \caption{MUFAC}
        \label{fig:mia_mufac}
    \end{subfigure}

\caption{MIA results across phases. For each method, the three bars correspond to the MIA scores at phases 1, 2, and 3.}
\label{fig:mia}
\end{figure*}

\noindent\textbf{Model Performance.}
In~\cref{tab:c100_perf},~\cref{tab:vgg2_perf}, and~\cref{tab:mufac_perf}, we present the accuracy results for CIFAR-100, VGGFace2, and MUFAC after each algorithm processes multiple unlearning requests.
For each dataset, we conduct three experiments under the N CLASSES-3PHASES setting, using different \(\mathcal{D}_{\text{forget}}^{(t)}\) at each phase.
We report the average accuracy and standard deviation across the three runs.

In~\cref{tab:c100_perf} and ~\cref{tab:vgg2_perf}, Retrain shows zero accuracy for both \(\mathcal{D}_{tf}\) and \(\mathcal{D}_{tg}\), while demonstrating almost identical performance on \(\mathcal{D}_{tr}\) in each unlearning phase.
On the other hand, the results in~\cref{tab:mufac_perf} show that the accuracy on forget and forgotten data is comparable to that of the test data within the margin of error, whereas the accuracy on retain data remains nearly unchanged, staying close to 100\%.

We now describe the performance results of the remaining unlearning methods using~\cref{fig:continual_measurement}.
In our experiments, SCRUB exhibits almost no knowledge erosion or forgetting reversal. However, as the number of phases increases, its ToW value, which measures similarity to the retrain performance, tends to gradually decrease. SALUN shows pronounced forgetting reversal, and this effect leads to a divergence from the retrain performance as reflected in its ToW scores. In contrast, SSD and Finetune display signs of knowledge erosion. BndShrink and NegGrad exhibit both substantial knowledge erosion and forgetting reversal, resulting in extremely low values of ToW.

For SAFER, the performance on retain data remains nearly unchanged across phases for all datasets, showing virtually no knowledge erosion as illustrated in~\cref{fig:ke_all}. In~\cref{tab:c100_perf} and~\cref{tab:vgg2_perf}, the accuracy on forget data and forgotten data remains close to zero. 
However, in~\cref{tab:mufac_perf}, the accuracy of unlearned samples in earlier phases becomes higher when evaluated in later phases. Reflecting these observations,~\cref{fig:fr_all} shows that forgetting reversal does not occur for CIFAR-100 or VGGFace2, whereas it does occur for MUFAC. Correspondingly, ToW scores for CIFAR-100 and VGGFace2 stay close to 1, while those for MUFAC exhibit a downward trend in ToW as the number of phases increases. 
To investigate whether this trend persists over extended phases, we provide the results of additional experiments in~\cref{sec:add_analysis}.

\noindent\textbf{Unlearning Efficacy.}
In~\cref{fig:mia}, we present MIA results for CIFAR-100, VGGFace2, and MUFAC.
To assess unlearning efficacy, we compute the proportion of forget samples that an attacker determines to be absent from the training set of the unlearned model, following the MIA protocol in~\cite{fan2024salun, zhang2025toward}. 
As shown in~\cref{fig:mia_c100} and~\cref{fig:mia_vgg2}, the retrained model in the class-aligned continual unlearning setting maintains an MIA score of 100\% across all phases, indicating that an attacker identifies none of the unlearned samples as belonging to the model's training set. In contrast, in the class-misaligned continual unlearning setting, illustrated in~\cref{fig:mia_mufac}, the retrained model yields scores near 60 across all phases.
Other algorithms exhibit varying MIA scores across phases under both the class-aligned and class-misaligned continual unlearning settings. However, SAFER consistently achieves MIA scores similar to the retrained model for all datasets and across all phases.

\noindent\textbf{Efficiency.}
\begin{figure}[t]
\centering
\includegraphics[width=0.97\columnwidth]{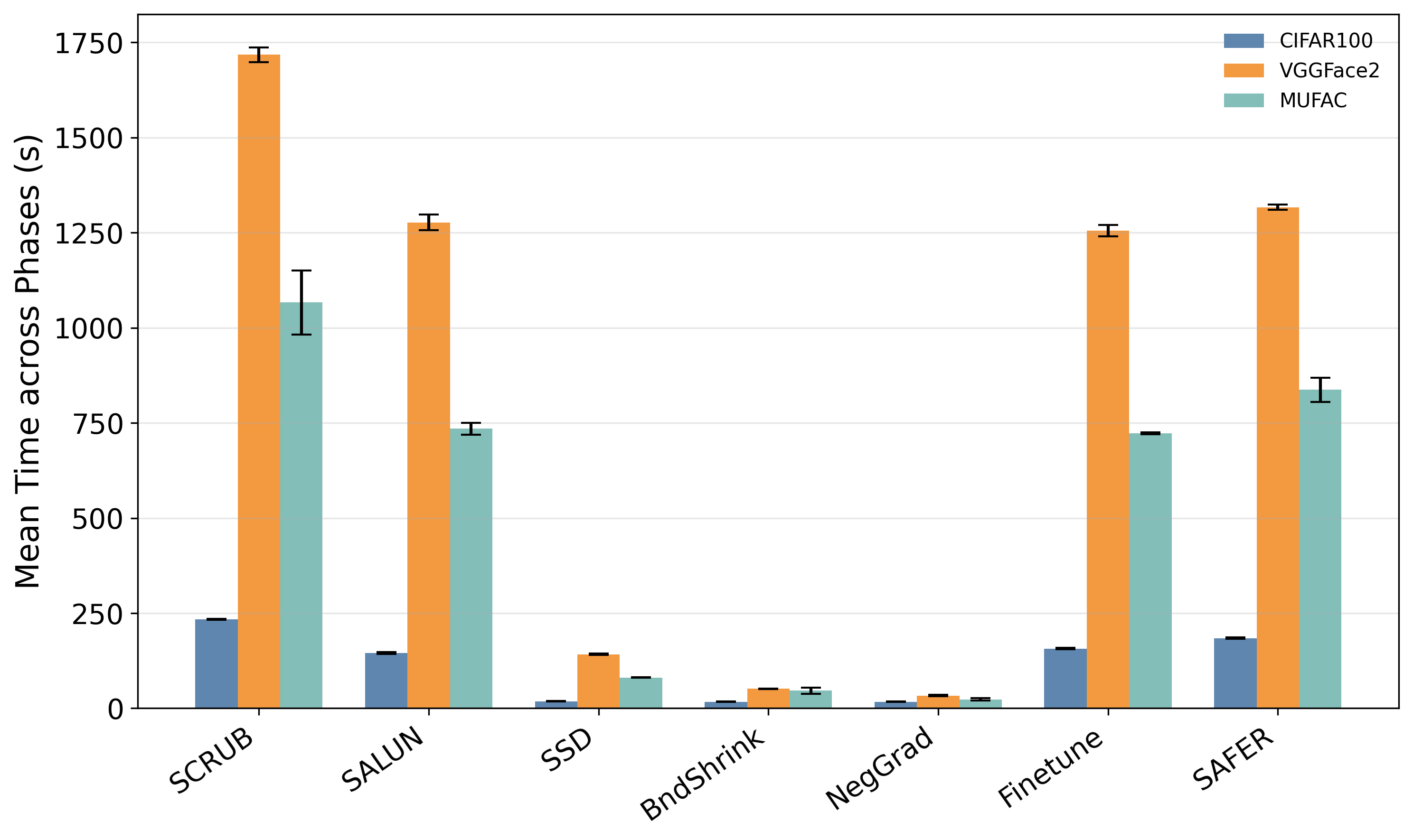}
\caption{Time efficiency of each unlearning method.}
\vspace{-1. em}
\label{fig:time-efficiency}
\end{figure}
We evaluate the computational efficiency of SAFER and other baselines in continual unlearning settings, as shown in Fig.~\ref{fig:time-efficiency}. Each value corresponds to the mean runtime across phases, and the error bar reflects the runtime variation between phases.
SSD, BndShrink, and NegGrad exhibit faster unlearning times across all datasets. 
Although SAFER requires more time than these methods, its time efficiency remains comparable to most approaches that fine-tune the model using retain data.
This implies that SAFER is able to maintain stable performance across phases without introducing meaningful computational overhead, making it practical for real-world continual unlearning scenarios.



%% file: tables/table_acc_c100.tex
\begin{table*}[h!]
    \centering
    \setlength{\tabcolsep}{9pt}
    \resizebox{\textwidth}{!}{
    \begin{tabular}{c|cc|cccccccc}
        \toprule
        
        & Phase & Acc. & Retrain & SCRUB & SALUN & SSD & BndShrink & NegGrad & Finetune & SAFER \\
         
        \midrule
        \multirow{8}{*}{\rotatebox{90}{3CLASSES-3PHASES}} & \multirow{2}{*}{1} 
        
        & \(\mathcal{D}_{tr}\) 
        & 76.60{\footnotesize $\pm$ 0.54} & 77.22{\footnotesize $\pm$ 0.30} & 71.97{\footnotesize $\pm$ 0.80} & 68.49{\footnotesize $\pm$ 8.23} & 63.22{\footnotesize $\pm$ 1.83} & 64.28{\footnotesize $\pm$ 3.03} & 76.73{\footnotesize $\pm$ 0.33} & 75.69{\footnotesize $\pm$ 0.81} \\
        
        & & \(\mathcal{D}_{tf}\)
        & 0.00{\footnotesize $\pm$ 0.00} & 1.33{\footnotesize $\pm$ 0.98} & 1.33{\footnotesize $\pm$ 0.82} & 6.11{\footnotesize $\pm$ 8.64} & 12.89{\footnotesize $\pm$ 7.82} & 16.33{\footnotesize $\pm$ 7.32} & 0.00{\footnotesize $\pm$ 0.00} & 0.00{\footnotesize $\pm$ 0.00} \\
        
        \cline{2-11}
        & \multirow{3}{*}{2} 
        
        & \(\mathcal{D}_{tr}\) 
        & 76.71{\footnotesize $\pm$ 0.24} & 77.78{\footnotesize $\pm$ 0.41} & 71.93{\footnotesize $\pm$ 0.42} & 53.80{\footnotesize $\pm$ 24.49} & 61.41{\footnotesize $\pm$ 2.40} & 68.14{\footnotesize $\pm$ 2.14} & 68.48{\footnotesize $\pm$ 6.88} & 77.82{\footnotesize $\pm$ 0.16} \\
        
        & & \(\mathcal{D}_{tf}\)
        & 0.00{\footnotesize $\pm$ 0.00} & 0.33{\footnotesize $\pm$ 0.27} & 0.56{\footnotesize $\pm$ 0.42} & 0.00{\footnotesize $\pm$ 0.00} & 10.11{\footnotesize $\pm$ 2.61} & 17.11{\footnotesize $\pm$ 2.98} & 1.89{\footnotesize $\pm$ 2.67} & 0.00{\footnotesize $\pm$ 0.00} \\
        
        & & \(\mathcal{D}_{tg}\)
        & 0.00{\footnotesize $\pm$ 0.00} & 2.00{\footnotesize $\pm$ 1.09} & 16.11{\footnotesize $\pm$ 4.90} & 6.22{\footnotesize $\pm$ 8.80} & 45.33{\footnotesize $\pm$ 19.81} & 68.78{\footnotesize $\pm$ 16.65} & 0.00{\footnotesize $\pm$ 0.00} & 0.00{\footnotesize $\pm$ 0.00} \\
        
        \cline{2-11}
        & \multirow{3}{*}{3} 
        
        & \(\mathcal{D}_{tr}\) 
        & 77.20{\footnotesize $\pm$ 0.32} & 78.03{\footnotesize $\pm$ 0.31} & 71.55{\footnotesize $\pm$ 0.54} & 37.89{\footnotesize $\pm$ 29.77} & 51.72{\footnotesize $\pm$ 2.05} & 64.65{\footnotesize $\pm$ 1.45} & 71.95{\footnotesize $\pm$ 5.73} & 77.29{\footnotesize $\pm$ 0.61} \\
        
        & & \(\mathcal{D}_{tf}\)
        & 0.00{\footnotesize $\pm$ 0.00} & 7.56{\footnotesize $\pm$ 6.00} & 0.89{\footnotesize $\pm$ 0.42} & 4.44{\footnotesize $\pm$ 6.28} & 8.00{\footnotesize $\pm$ 1.19} & 18.56{\footnotesize $\pm$ 3.27} & 8.78{\footnotesize $\pm$ 11.03} & 0.11{\footnotesize $\pm$ 0.16} \\
        
        & & \(\mathcal{D}_{tg}\)
        & 0.00{\footnotesize $\pm$ 0.00} & 1.17{\footnotesize $\pm$ 0.76} & 16.78{\footnotesize $\pm$ 2.28} & 3.39{\footnotesize $\pm$ 4.79} & 36.72{\footnotesize $\pm$ 4.36} & 65.39{\footnotesize $\pm$ 3.75} & 0.00{\footnotesize $\pm$ 0.00} & 0.00{\footnotesize $\pm$ 0.00} \\
        
        \bottomrule
    \end{tabular}
}
\caption{Performance for continual unlearning on CIFAR-100 with ResNet-18.}
\label{tab:c100_perf}
\end{table*}

%% file: tables/table_acc_vgg2.tex
\begin{table*}[h!]
    \centering
    \setlength{\tabcolsep}{9pt}
    \resizebox{\textwidth}{!}{
    \begin{tabular}{c|cc|cccccccc}
        \toprule
        
        & Phase & Acc. & Retrain & SCRUB & SALUN & SSD & BndShrink & NegGrad & Finetune & SAFER \\
        
        \midrule
        \multirow{8}{*}{\rotatebox{90}{3CLASSES-3PHASES}} & \multirow{2}{*}{1} 
        
        & \(\mathcal{D}_{tr}\) 
        & 96.66{\footnotesize $\pm$ 0.35} & 95.30{\footnotesize $\pm$ 0.03} & 94.05{\footnotesize $\pm$ 0.32} & 85.49{\footnotesize $\pm$ 4.85} & 85.15{\footnotesize $\pm$ 3.36} & 88.91{\footnotesize $\pm$ 1.69} & 95.39{\footnotesize $\pm$ 0.51} & 94.90{\footnotesize $\pm$ 0.03} \\
        
        & & \(\mathcal{D}_{tf}\)
        & 0.00{\footnotesize $\pm$ 0.00} & 2.34{\footnotesize $\pm$ 1.85} & 0.67{\footnotesize $\pm$ 0.62} & 0.00{\footnotesize $\pm$ 0.00} & 14.33{\footnotesize $\pm$ 5.59} & 36.32{\footnotesize $\pm$ 4.78} & 7.84{\footnotesize $\pm$ 3.22} & 0.00{\footnotesize $\pm$ 0.00} \\
        
        \cline{2-11}
        & \multirow{3}{*}{2} 
        
        & \(\mathcal{D}_{tr}\) 
        & 96.81{\footnotesize $\pm$ 0.15} & 95.26{\footnotesize $\pm$ 0.17} & 93.07{\footnotesize $\pm$ 0.48} & 23.22{\footnotesize $\pm$ 17.63} & 68.65{\footnotesize $\pm$ 2.65} & 89.48{\footnotesize $\pm$ 2.01} & 84.89{\footnotesize $\pm$ 2.50} & 95.25{\footnotesize $\pm$ 0.16} \\
        
        & & \(\mathcal{D}_{tf}\)
        & 0.00{\footnotesize $\pm$ 0.00} & 6.85{\footnotesize $\pm$ 5.17} & 0.66{\footnotesize $\pm$ 0.47} & 0.00{\footnotesize $\pm$ 0.00} & 12.60{\footnotesize $\pm$ 7.40} & 39.90{\footnotesize $\pm$ 10.63} & 0.00{\footnotesize $\pm$ 0.00} & 0.00{\footnotesize $\pm$ 0.00} \\
        
        & & \(\mathcal{D}_{tg}\)
        & 0.00{\footnotesize $\pm$ 0.00} & 2.39{\footnotesize $\pm$ 2.13} & 37.75{\footnotesize $\pm$ 3.04} & 0.00{\footnotesize $\pm$ 0.00} & 27.37{\footnotesize $\pm$ 1.84} & 92.36{\footnotesize $\pm$ 1.56} & 0.00{\footnotesize $\pm$ 0.00} & 0.00{\footnotesize $\pm$ 0.00} \\
        
        \cline{2-11}
        & \multirow{3}{*}{3} 
        
        & \(\mathcal{D}_{tr}\) 
        & 96.67{\footnotesize $\pm$ 0.10} & 95.42{\footnotesize $\pm$ 0.12} & 92.83{\footnotesize $\pm$ 0.21} & 1.98{\footnotesize $\pm$ 1.37} & 40.69{\footnotesize $\pm$ 3.52} & 90.69{\footnotesize $\pm$ 0.81} & 85.04{\footnotesize $\pm$ 0.71} & 95.53{\footnotesize $\pm$ 0.14} \\
        
        & & \(\mathcal{D}_{tf}\)
        & 0.00{\footnotesize $\pm$ 0.00} & 0.18{\footnotesize $\pm$ 0.25} & 2.31{\footnotesize $\pm$ 2.89} & 0.00{\footnotesize $\pm$ 0.00} & 1.98{\footnotesize $\pm$ 2.43} & 32.20{\footnotesize $\pm$ 3.28} & 0.00{\footnotesize $\pm$ 0.00} & 0.00{\footnotesize $\pm$ 0.00} \\
        
        & & \(\mathcal{D}_{tg}\)
        & 0.00{\footnotesize $\pm$ 0.00} & 4.14{\footnotesize $\pm$ 2.83} & 17.65{\footnotesize $\pm$ 3.53} & 0.00{\footnotesize $\pm$ 0.00} & 11.18{\footnotesize $\pm$ 2.95} & 93.17{\footnotesize $\pm$ 5.15} & 0.00{\footnotesize $\pm$ 0.00} & 0.00{\footnotesize $\pm$ 0.00} \\
        
        \bottomrule
    \end{tabular}
}
\caption{Performance for continual unlearning on VGGFace2 with ResNet-50.}
\label{tab:vgg2_perf}
\end{table*}

%% file: tables/table_acc_mufac.tex
\begin{table*}[h!]
    \centering
    \setlength{\tabcolsep}{9pt}
    \resizebox{\textwidth}{!}{
    \begin{tabular}{c|cc|cccccccc}
        \toprule
        
        & Phase & Acc. & Retrain & SCRUB & SALUN & SSD & BndShrink & NegGrad & Finetune & SAFER \\
        \midrule
        \multirow{11}{*}{\rotatebox{90}{20PERSONS-3PHASES}} & \multirow{3}{*}{1} 
        
        & \(\mathcal{D}_{r}\)
        & 99.97{\footnotesize $\pm$ 0.00} & 99.94{\footnotesize $\pm$ 0.02} & 99.04{\footnotesize $\pm$ 0.58} & 64.29{\footnotesize $\pm$ 42.49} & 52.95{\footnotesize $\pm$ 12.63} & 53.82{\footnotesize $\pm$ 31.04} & 99.97{\footnotesize $\pm$ 0.00} & 99.73{\footnotesize $\pm$ 0.07} \\
        
        & & \(\mathcal{D}_{f}\)
        & 73.69{\footnotesize $\pm$ 6.66} & 76.53{\footnotesize $\pm$ 8.84} & 21.33{\footnotesize $\pm$ 5.15} & 65.69{\footnotesize $\pm$ 40.87} & 29.66{\footnotesize $\pm$ 6.42} & 45.24{\footnotesize $\pm$ 31.26} & 98.71{\footnotesize $\pm$ 0.30} & 71.94{\footnotesize $\pm$ 1.02} \\
        
        & & \(\mathcal{D}_{t}\)
        & 67.02{\footnotesize $\pm$ 0.39} & 69.48{\footnotesize $\pm$ 0.42} & 64.03{\footnotesize $\pm$ 4.15} & 44.46{\footnotesize $\pm$ 30.43} & 43.75{\footnotesize $\pm$ 9.92} & 45.48{\footnotesize $\pm$ 20.44} & 68.24{\footnotesize $\pm$ 0.64} & 68.88{\footnotesize $\pm$ 0.47} \\
        
        \cline{2-11}
        & \multirow{4}{*}{2}
        
        & \(\mathcal{D}_{r}\)
        & 99.98{\footnotesize $\pm$ 0.01} & 99.98{\footnotesize $\pm$ 0.01} & 99.81{\footnotesize $\pm$ 0.06} & 64.23{\footnotesize $\pm$ 42.39} & 29.16{\footnotesize $\pm$ 8.88} & 21.46{\footnotesize $\pm$ 5.83} & 99.97{\footnotesize $\pm$ 0.01} & 99.87{\footnotesize $\pm$ 0.04} \\
        
        & & \(\mathcal{D}_{f}\)
        & 71.63{\footnotesize $\pm$ 5.61} & 99.60{\footnotesize $\pm$ 0.57} & 28.21{\footnotesize $\pm$ 11.59} & 62.91{\footnotesize $\pm$ 44.04} & 17.58{\footnotesize $\pm$ 5.02} & 21.38{\footnotesize $\pm$ 8.65} & 99.73{\footnotesize $\pm$ 0.38} & 70.49{\footnotesize $\pm$ 4.39} \\
        
        & & \(\mathcal{D}_{t}\)
        & 66.91{\footnotesize $\pm$ 0.27} & 69.30{\footnotesize $\pm$ 0.60} & 64.30{\footnotesize $\pm$ 0.41} & 44.62{\footnotesize $\pm$ 30.52} & 27.08{\footnotesize $\pm$ 9.18} & 23.29{\footnotesize $\pm$ 7.17} & 68.17{\footnotesize $\pm$ 0.19} & 67.95{\footnotesize $\pm$ 1.00} \\
        
        & & \(\mathcal{D}_{g}\)
        & 61.66{\footnotesize $\pm$ 6.81} & 76.67{\footnotesize $\pm$ 8.90} & 27.81{\footnotesize $\pm$ 6.83} & 65.40{\footnotesize $\pm$ 40.70} & 19.16{\footnotesize $\pm$ 2.13} & 20.91{\footnotesize $\pm$ 7.87} & 98.62{\footnotesize $\pm$ 0.88} & 75.95{\footnotesize $\pm$ 7.65} \\
        
        \cline{2-11}
        & \multirow{4}{*}{3}
        
        & \(\mathcal{D}_{r}\)
        & 99.98{\footnotesize $\pm$ 0.01} & 99.98{\footnotesize $\pm$ 0.01} & 99.46{\footnotesize $\pm$ 0.38} & 64.19{\footnotesize $\pm$ 42.29} & 19.39{\footnotesize $\pm$ 3.14} & 17.12{\footnotesize $\pm$ 0.20} & 99.98{\footnotesize $\pm$ 0.01} & 99.92{\footnotesize $\pm$ 0.01} \\
        
        & & \(\mathcal{D}_{f}\)
        & 72.37{\footnotesize $\pm$ 10.39} & 100.00{\footnotesize $\pm$ 0.00} & 29.21{\footnotesize $\pm$ 12.23} & 63.81{\footnotesize $\pm$ 45.26} & 10.46{\footnotesize $\pm$ 2.35} & 22.97{\footnotesize $\pm$ 3.62} & 100.00{\footnotesize $\pm$ 0.00} & 64.68{\footnotesize $\pm$ 2.85} \\
        
        & & \(\mathcal{D}_{t}\)
        & 66.95{\footnotesize $\pm$ 0.57} & 69.33{\footnotesize $\pm$ 0.51} & 64.21{\footnotesize $\pm$ 1.30} & 44.48{\footnotesize $\pm$ 30.45} & 19.19{\footnotesize $\pm$ 1.85} & 18.22{\footnotesize $\pm$ 0.00} & 68.61{\footnotesize $\pm$ 0.34} & 67.51{\footnotesize $\pm$ 0.54} \\
        
        & & \(\mathcal{D}_{g}\)
        & 66.80{\footnotesize $\pm$ 5.50} & 87.17{\footnotesize $\pm$ 4.70} & 52.34{\footnotesize $\pm$ 11.98} & 64.19{\footnotesize $\pm$ 42.18} & 15.64{\footnotesize $\pm$ 1.38} & 17.32{\footnotesize $\pm$ 2.97} & 99.00{\footnotesize $\pm$ 0.60} & 78.65{\footnotesize $\pm$ 3.28} \\

        \bottomrule
    \end{tabular}
}
\caption{Performance for continual unlearning on MUFAC with ViT-B/16.}
\vspace{-1.em}
\label{tab:mufac_perf}
\end{table*}

%% file: sec/6_analysis.tex
\section{Analysis}
\label{sec:analysis}

\paragraph{Representation Clustering Quality.}
To assess how repeated unlearning affects the structure of learned representations, we analyze the representation stability of retain data using the Davies–Bouldin Index (DBI)~\cite{davies2009cluster}, a widely used data clusterability metric.
SAFER aims to stabilize feature representations by reducing intra-class variation and reinforcing inter-class separation. To quantitatively assess whether the retain data remain well-clustered throughout continual unlearning, we adopt the DBI metric.
Formally, for each class (\textit{i}), we compute its DBI as
\begin{equation}
    \mathrm{DBI}_{i} = \max_{i \neq j} \frac{d_{intra}(i) + d_{intra}(j)}{d_{inter} (i,j)},
\end{equation}
where \(d_{intra}(i)\) denotes the intra-class variance of cluster (\textit{i}), computed as the mean squared distance of its samples from the corresponding class centroid.
\(d_{inter} (i,j)\) represents the distance between the class centroids of clusters (\textit{i}) and (\textit{j}).
The overall DBI score is obtained by averaging \(\mathrm{DBI}_i\) across all classes in the retain data at phase \(t\).
A lower DBI value indicates tighter clusters and greater separation.

~\cref{fig:retain_dbi} shows the DBI results for the retain data across all phases on CIFAR-100, VGGFace2, and MUFAC.
In all datasets, Retrain maintains consistently low DBI values across phases, reflecting compact intra-class structures and well-separated class centers.
This behavior represents the ideal target for stable representation geometry and serves as the reference point against which all unlearning methods are evaluated.
Compared to the retrain model, NegGrad shows higher DBI values. This suggests that error-maximizing unlearning methods that do not account for retain data may induce drift in retrain data representation, leading to negative effects on model utility. 
On the contrary, our proposed method maintains low DBI values across all phases, remaining close to the retrain model. This suggests that its feature-regularization mechanism effectively constrains representational drift over phases.

\begin{figure}[!tbp]
\centering
    \begin{subfigure}[b]{0.32\columnwidth}
    \includegraphics[width=\columnwidth] {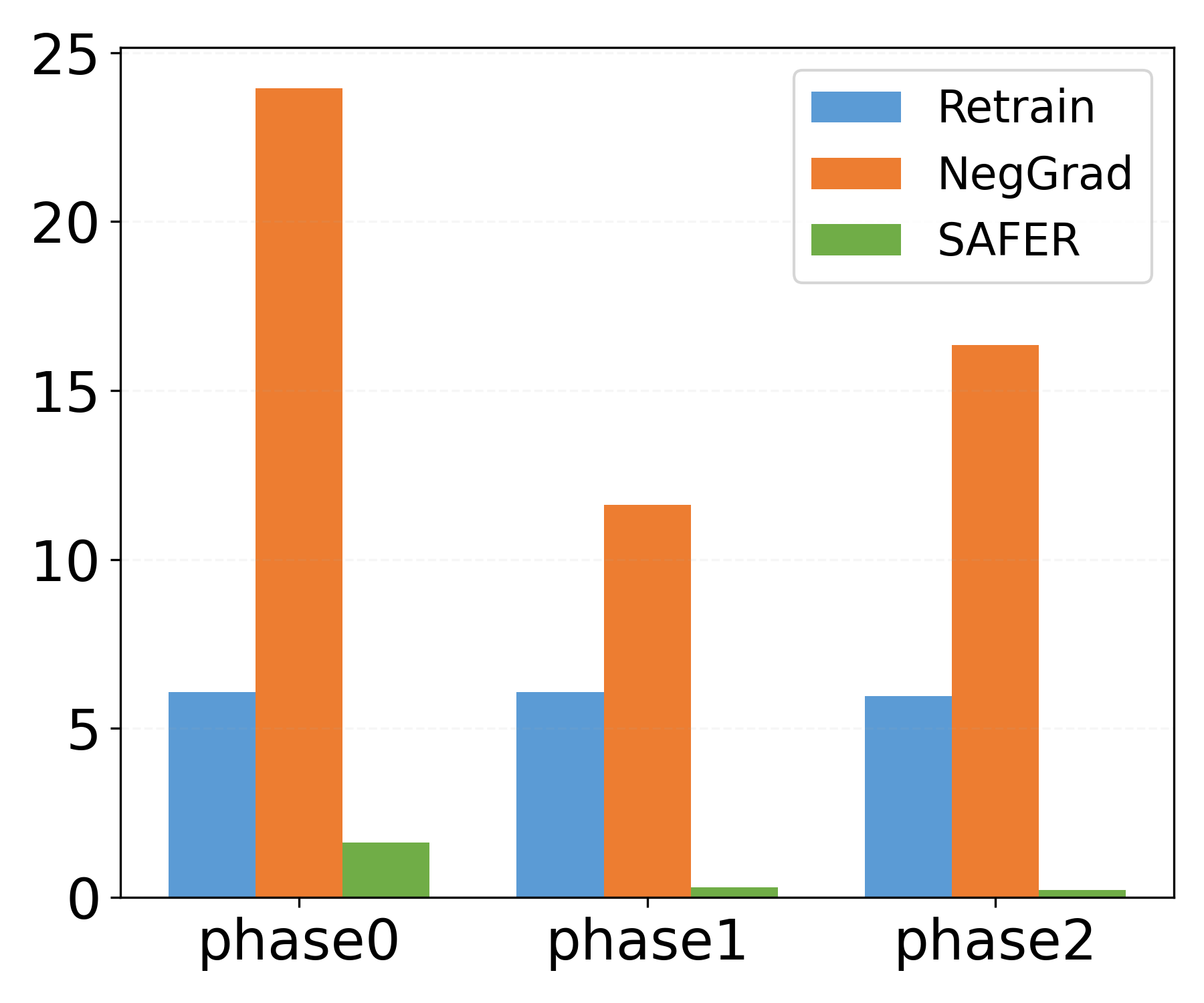}
    \caption{CIFAR-100}
    \end{subfigure}
    \begin{subfigure}[b]{0.32\columnwidth}
    \includegraphics[width=\columnwidth]{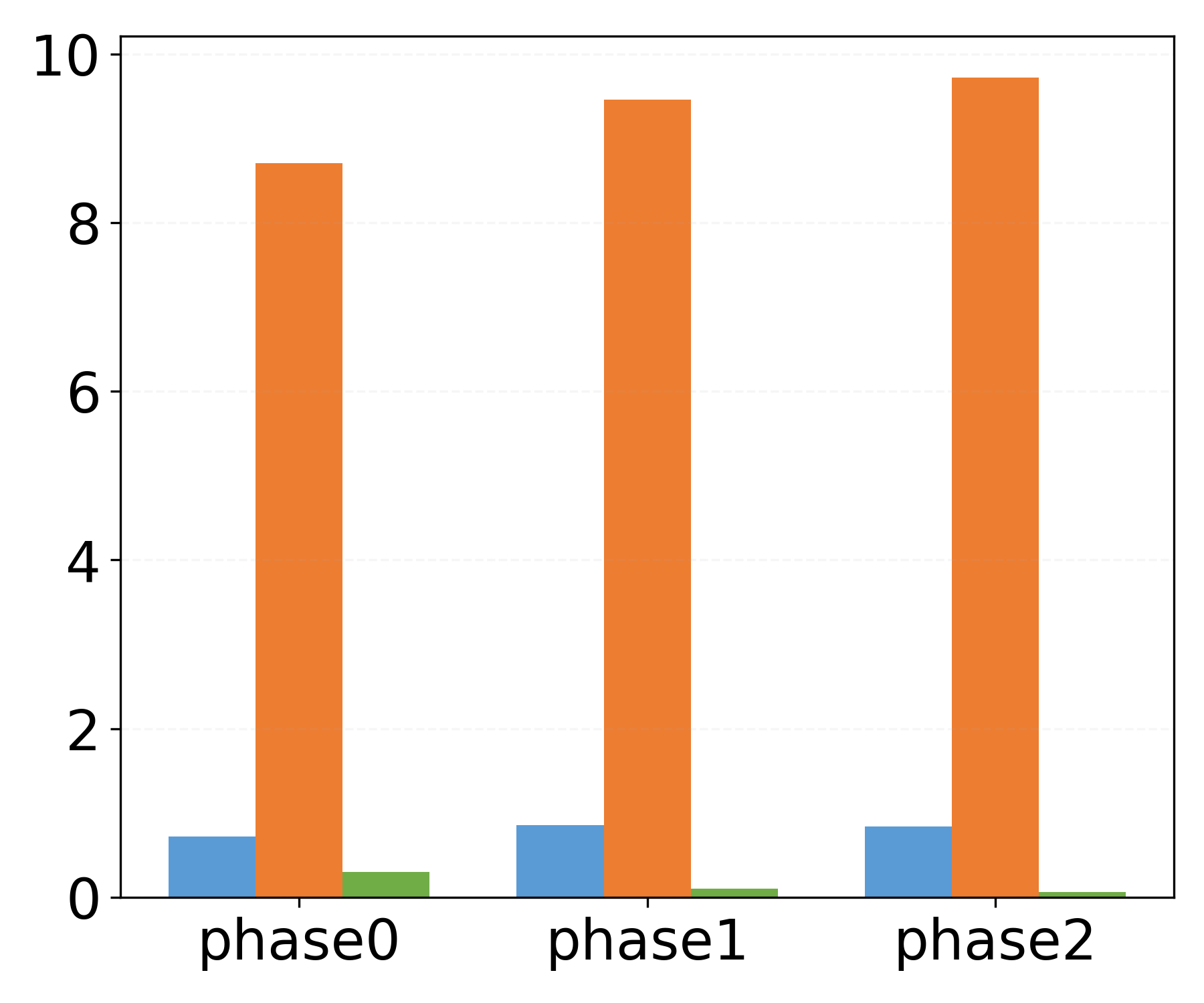}
    \caption{VGGFace2}
    \end{subfigure}
    \begin{subfigure}[b]{0.32\columnwidth}
    \includegraphics[width=\columnwidth]{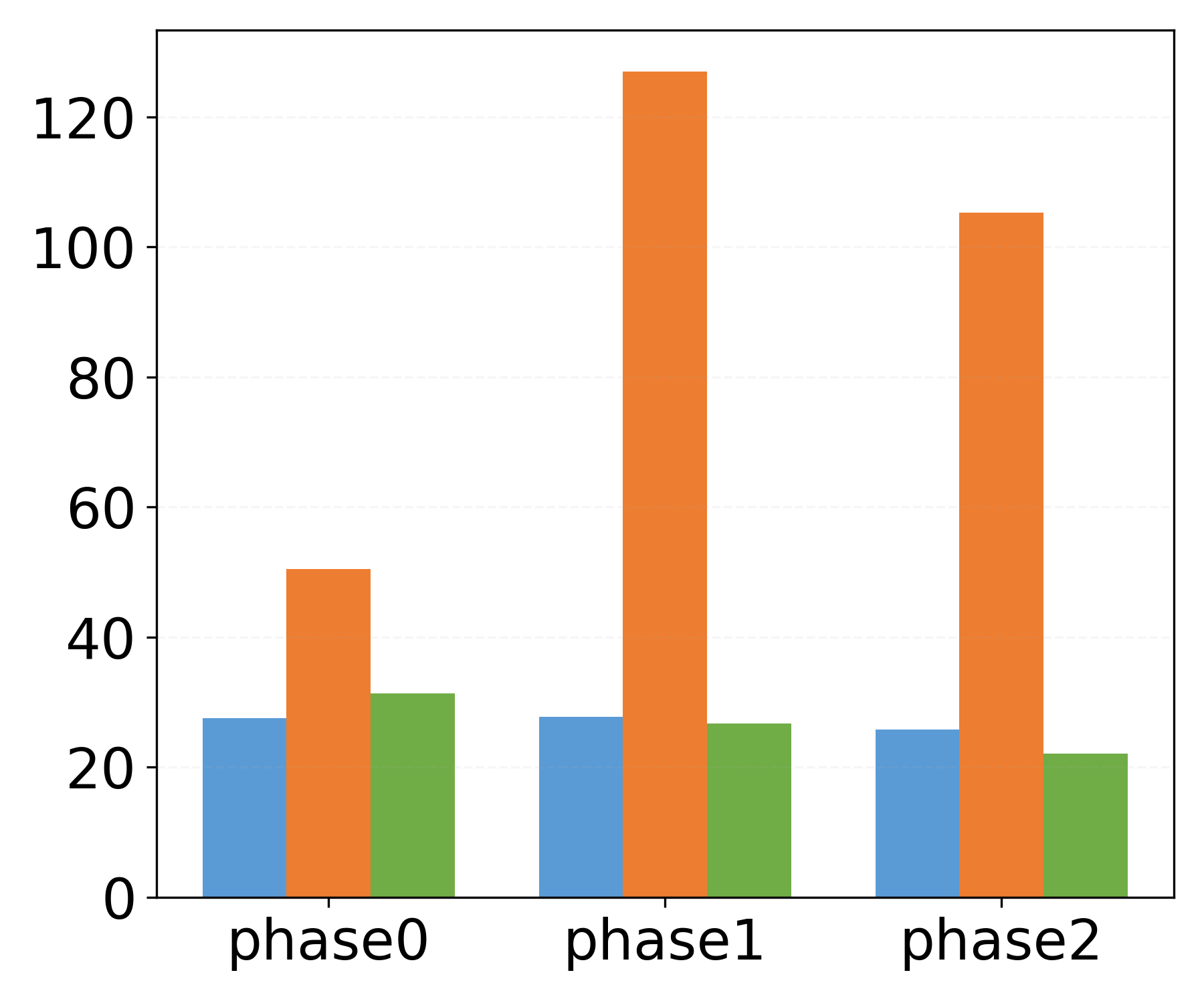}
    \caption{MUFAC}
    \end{subfigure}    
\caption{DBI scores for retain data across unlearning phases. A smaller DBI score reflects more compact clusters and better inter-class separation.}
\label{fig:retain_dbi}
\vspace{-1.em}
\end{figure}






\paragraph{Negative Unlearning Margin.}
To further evaluate how effectively SAFER prevents forgetting reversal and preserves the intended forgetting behavior across phases, we analyze the unlearning margin distributions of the forget and forgotten data. 
Recall that the proposed method explicitly enforces negative logit margins for all non-retain classes as defined in~\cref{eq:unlearning_margin}.

~\cref{fig:neg_logit_margin} presents the unlearning margin distributions for the forget set \(\mathcal{D}_{\text{forget}}^{(t)}\) and the forgotten set \(\mathcal{D}_{\text{forgot}}^{(t)}\) across phases. 
For the retrained model, \(UM(\mathbf{x})\) is strictly negative for CIFAR-100 and VGGFace2, whereas the margin distribution for MUFAC is skewed toward positive values. This indicates that, in the class-aligned continual unlearning setting, forget data never become more confident than the other classes, as they have never been part of the model’s training set. However, in the class-misaligned continual unlearning setting, the retain data influence all classes during training, causing the unlearning margin to form distributions with values exceeding zero.
Analogously to the retrained model, SAFER maintains negative unlearning margins across all phases in the class-aligned continual unlearning setting. 
This suggests that the explicit logit margin constraint is effective in preventing forgotten classes from re-entering the decision region. 
In contrast, under the class-misaligned continual unlearning setting, SAFER exhibits a tendency for its UM distribution to shift toward zero, unlike the retrained model. 
This indicates that applying the logit margin constraint to forget data influences how the unlearning margins are distributed in such settings.

In~\cref{sec:add_analysis}, we provide further analysis of our proposed method, including the performance contribution of each module in SAFER and the sensitivity of \(\lambda\) in~\cref{eq:retain_loss}.

\begin{figure}[!tbp]
\centering
    \begin{subfigure}[b]{0.49\columnwidth} 
        \centering 
        \includegraphics[width=\columnwidth]{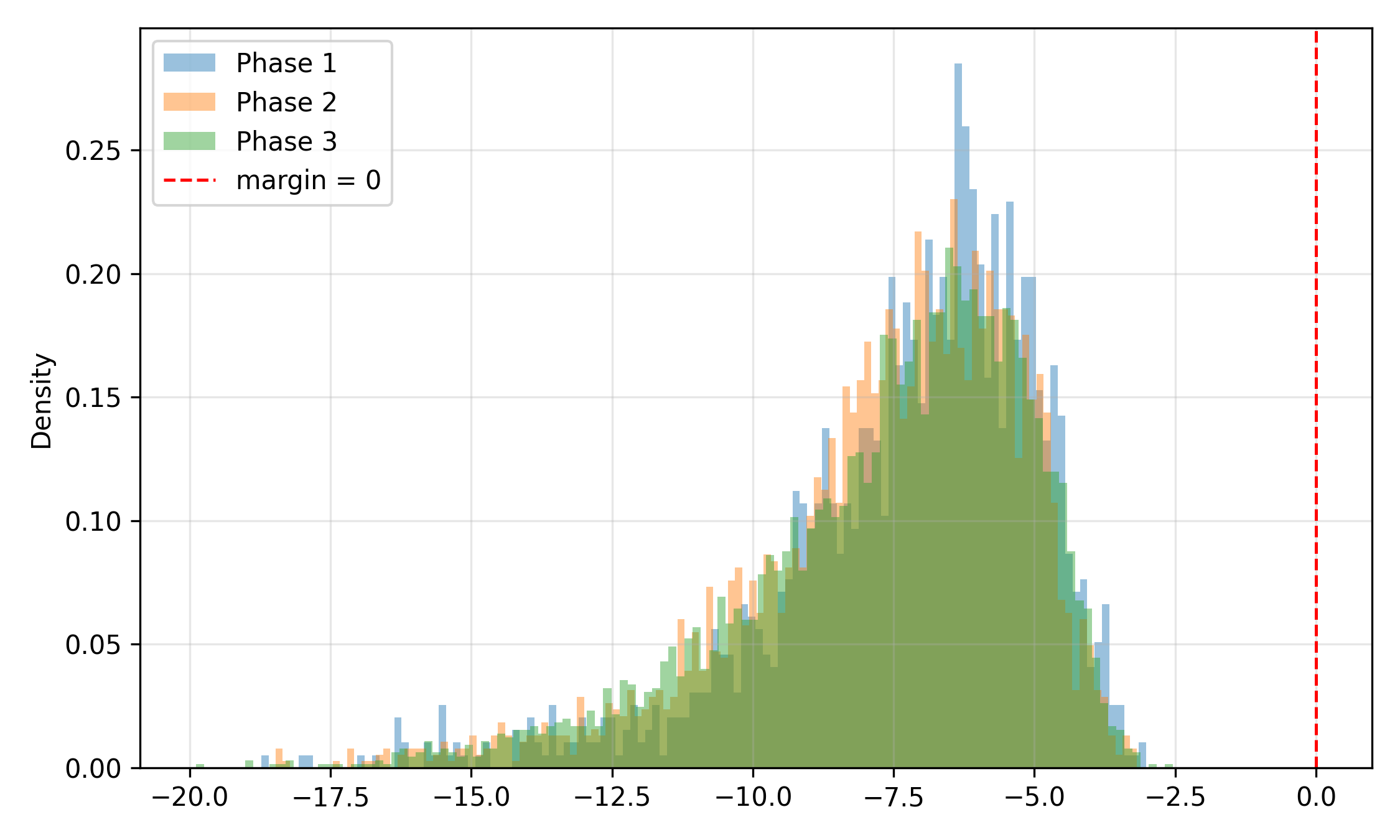}
        \vfill
        \includegraphics[width=\columnwidth]{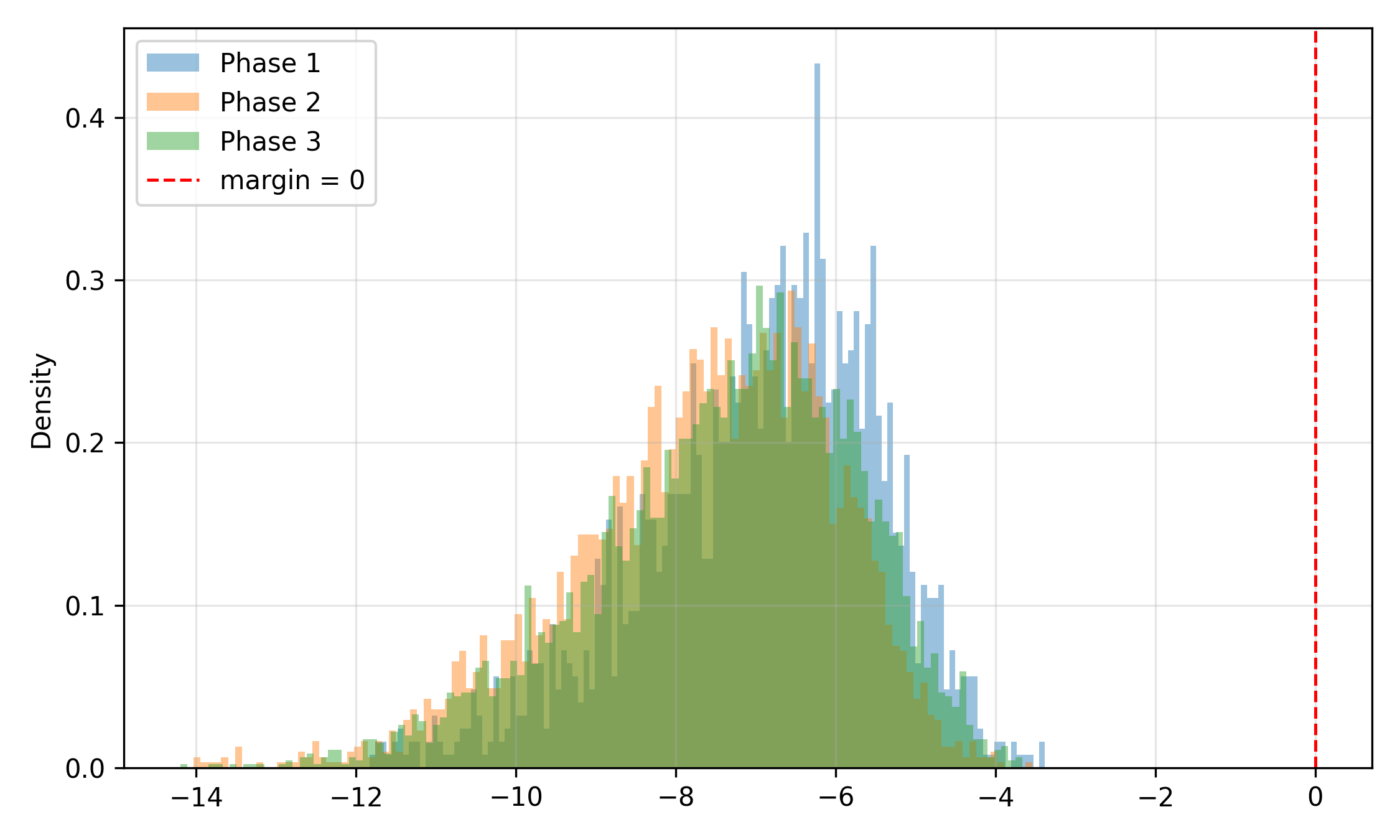}
        \vfill
        \includegraphics[width=\columnwidth]{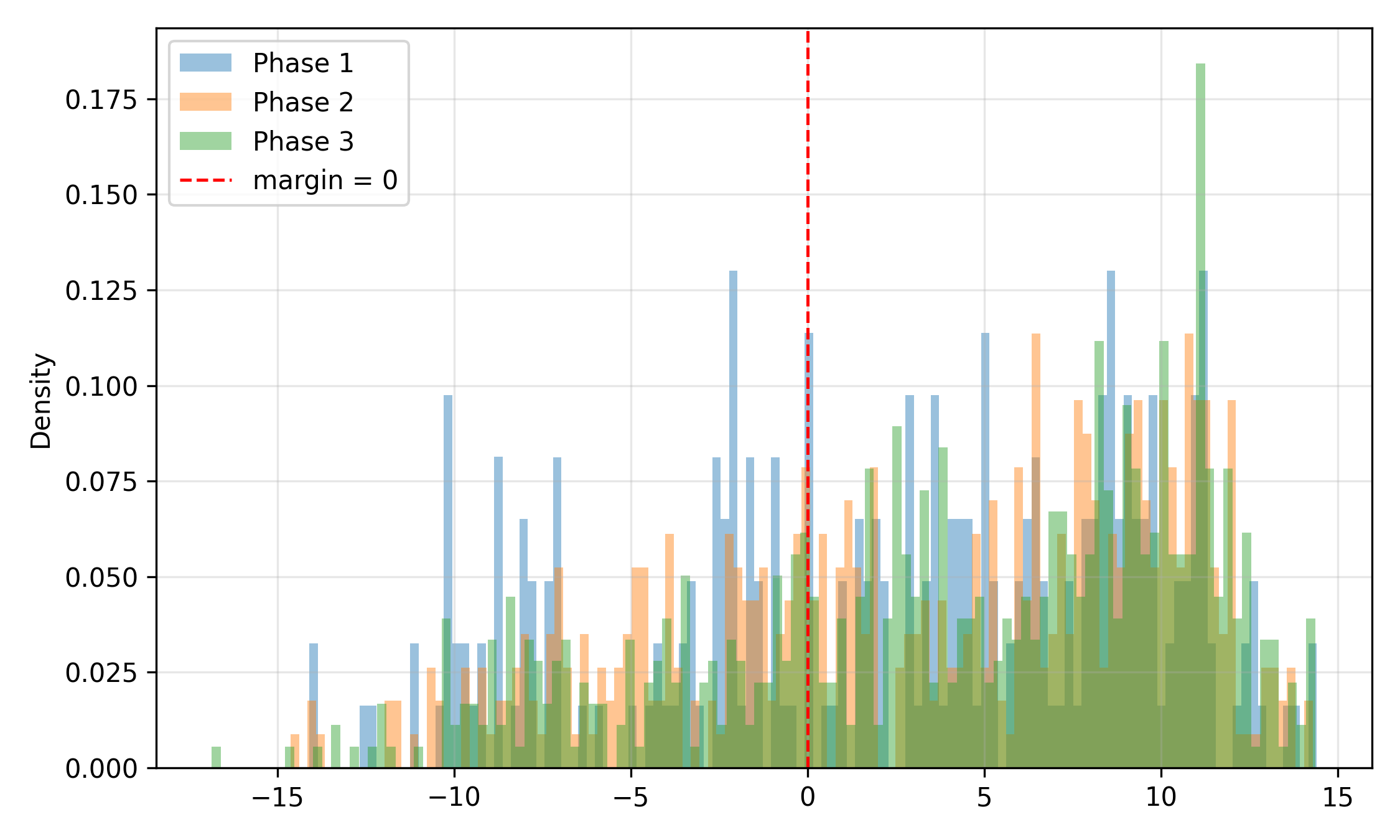}
        \caption{Retrain}
    \end{subfigure}
    \hfill
    \begin{subfigure}[b]{0.49\columnwidth} 
        \centering 
        \includegraphics[width=\columnwidth]{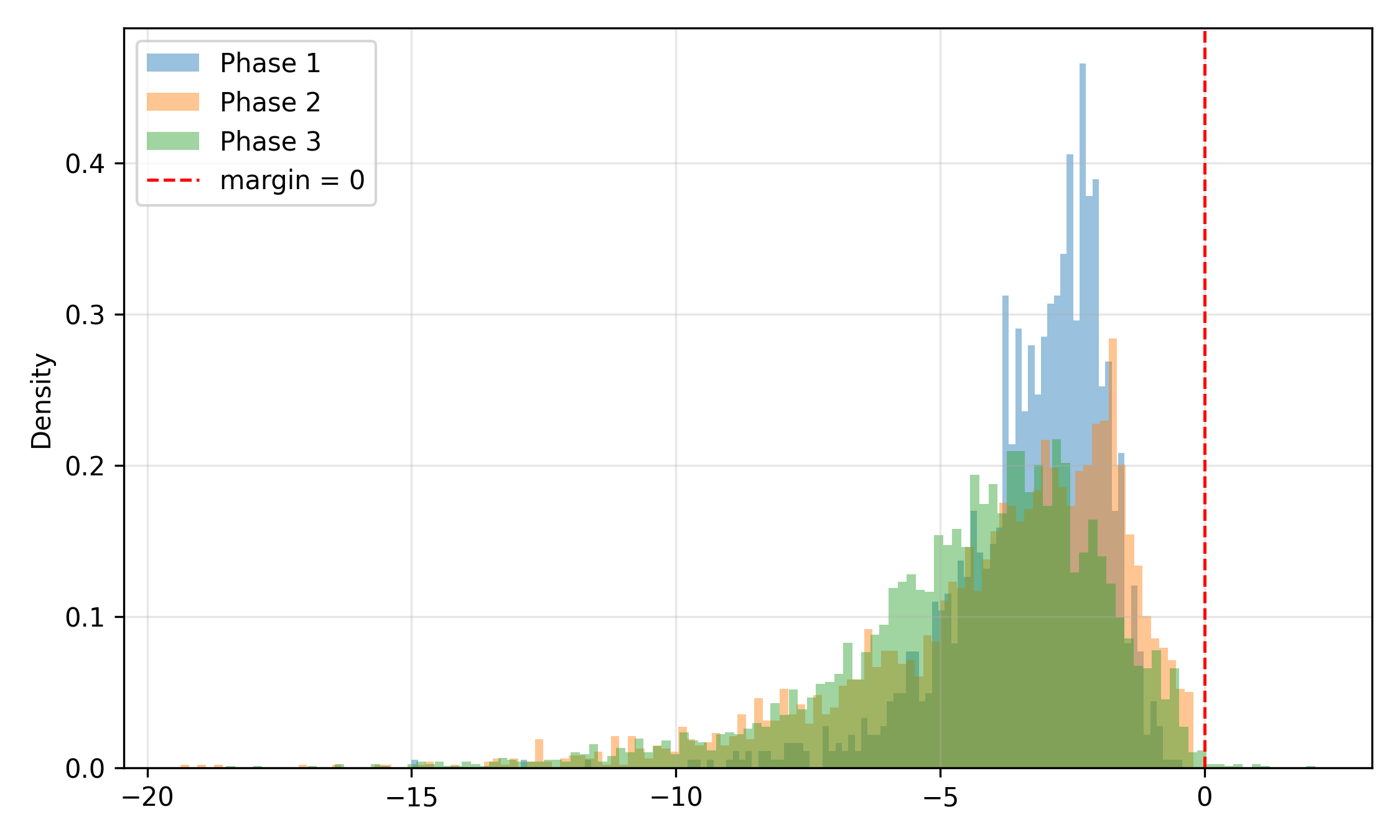}
        \vfill
        \includegraphics[width=\columnwidth]{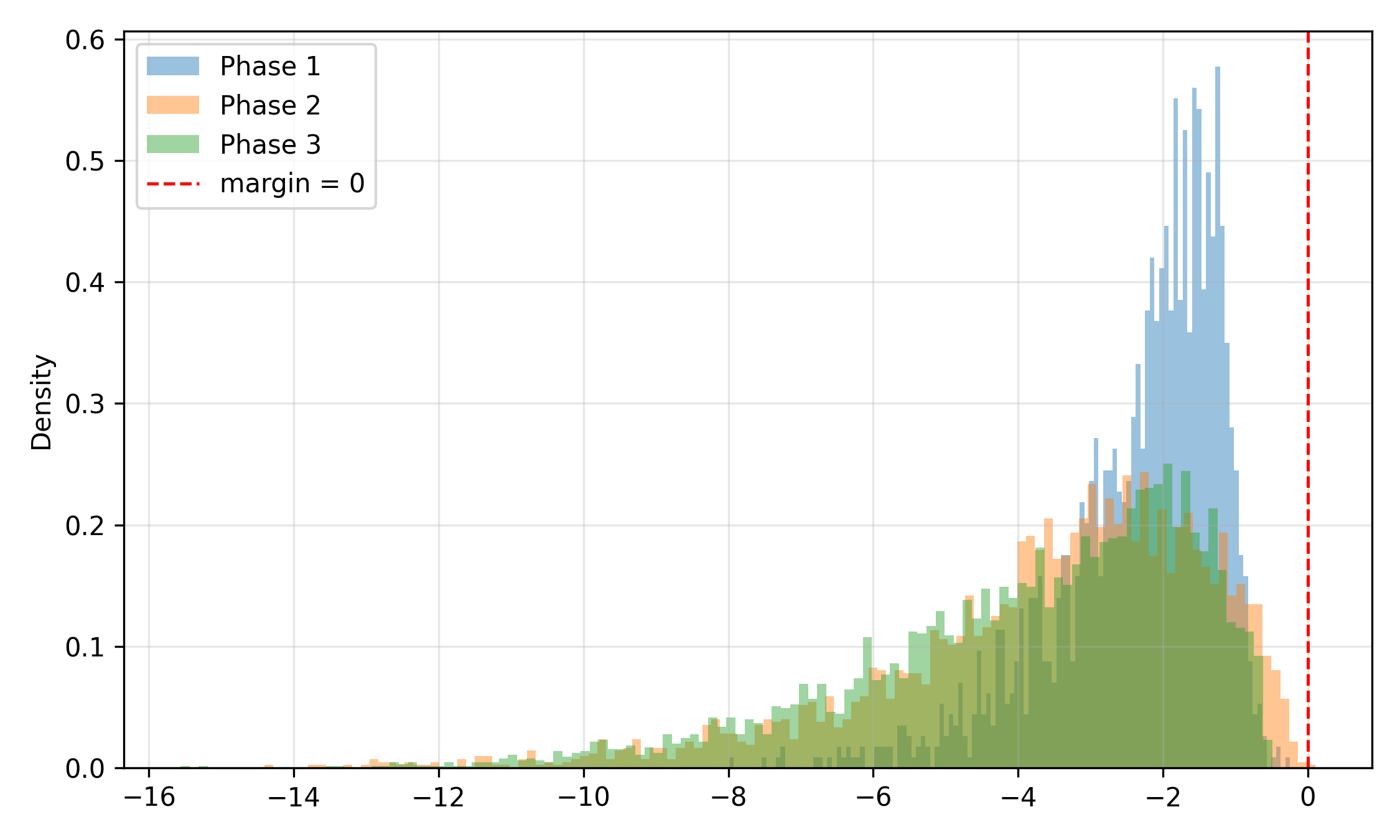}
        \vfill
        \includegraphics[width=\columnwidth]{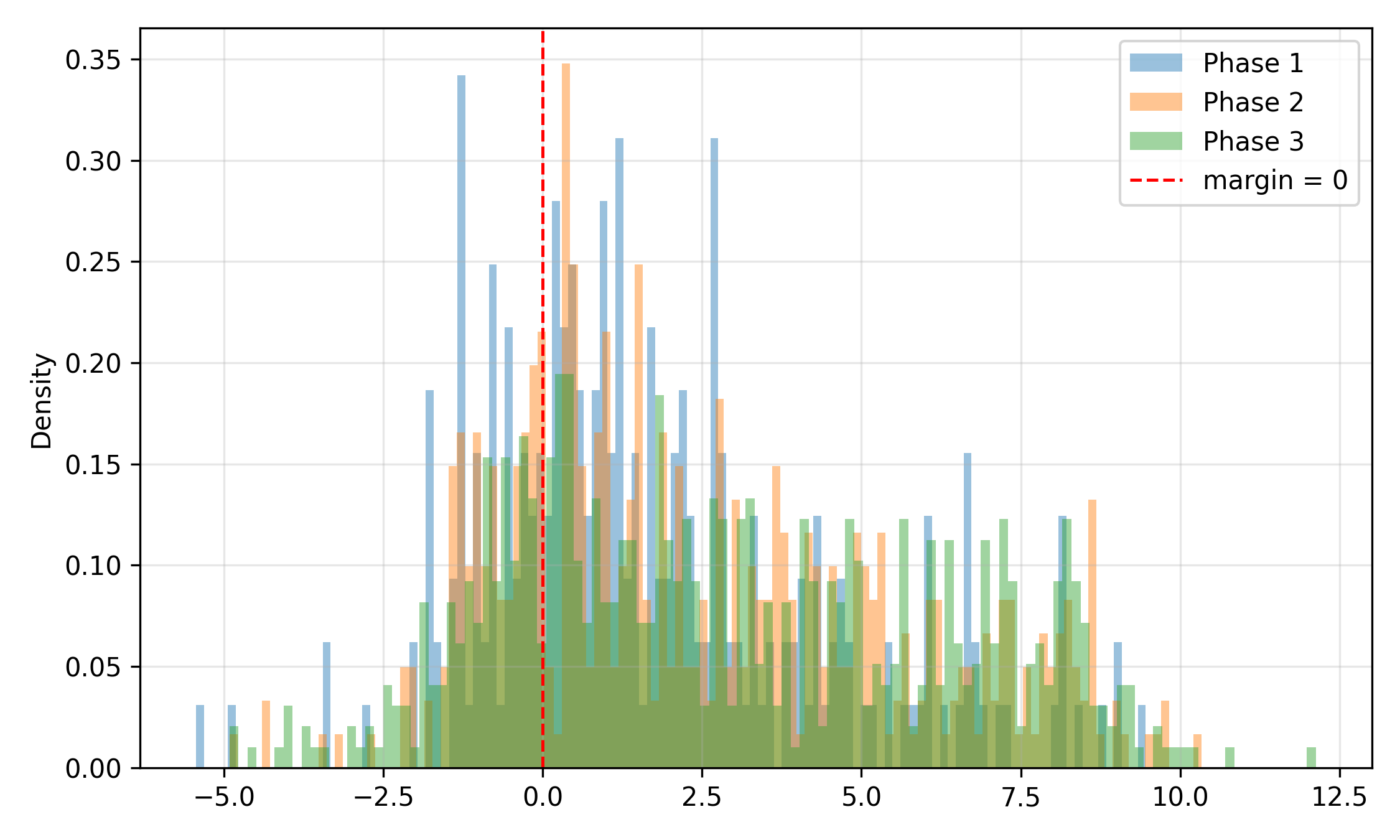}
        \caption{SAFER}
    \end{subfigure}  

\caption{Logit margin distribution. From top to bottom, the plots show the results for CIFAR-100, VGGFace2, and MUFAC.}
\label{fig:neg_logit_margin}
\vspace{-1.em}
\end{figure}

%% file: sec/7_conclusion.tex
\section{Conclusion}
\label{sec:conclusion}

In this work, we examined the challenges that emerge when machine unlearning is performed continually across multiple phases. Our analysis shows that existing unlearning approaches exhibit two critical issues in this setting: knowledge erosion, characterized by a gradual decline in performance on retain data, and forgetting reversal, where previously forgotten samples become recognizable again in later stages. These issues underscore the need for methods that not only ensure complete removal of designated information but also preserve model stability over time.
To address this challenge, we proposed SAFER, a continual unlearning framework that maintains well-clustered representations for retain samples while enforcing negative logit margins for forget samples. Experimental evaluations across multiple datasets and unlearning configurations demonstrate that SAFER consistently preserves model utility while preventing the reactivation of forgotten information.
Furthermore, we provide an in-depth analysis of how margin-based suppression contributes to stability in continual unlearning.
We hope our framework contributes to efforts toward robust and practically deployable continual unlearning solutions.

\section*{Acknowledgment}
This work was partly supported by Institute for Information \& communication Technology Planning \& evaluation (IITP) grants funded by the Korean government MSIT: (RS-2022-II220688, RS-2019-II190421, RS-2024-00437849, RS-2024-00337703). Also, this work was supported by the Cyber Investigation Support Technology Development Program (No.RS-2025-02304983) of the Korea Institute of Police Technology (KIPoT), funded by the Korean National Police Agency. Lastly, this work was supported by the National Research Foundation of Korea (NRF) grant funded by the Korea government (MSIT) (No. RS-2024-00356293).

%% file: sec/8_suppl.tex
\clearpage
\setcounter{page}{1}
\maketitlesupplementary

\section{Theoretical Analysis of Unlearning Margin Suppression}
\label{sec:lms_theory}


\begin{proposition}[Effect of Negative Unlearning Margin]
For a non-retain sample \(\boldsymbol{\mathrm{x}}\) with original label \(\mathrm{y}\), optimizing the KL divergence loss
with the random retain-class target distribution \(\mathbf{q}\) in Eq.~\ref{eq:random_logitdist}
drives its unlearning margin toward negative values:
\begin{equation}
    UM(\mathbf{x}) < 0.
\end{equation}
Thus, forgotten samples remain suboptimal within the decision space.
\end{proposition}

\begin{proof}[Proof Sketch]
Let \(\mathbf{p}(\mathbf{x}) = \mathrm{softmax}(\boldsymbol{\ell}(\mathbf{x}))\) denote the model’s predictive distribution over logits \(\boldsymbol{\ell}\).
Then the KL objective is written as:
\begin{equation}
    \mathcal{L}_{\mathrm{KL}}(\mathbf{x})
    = \mathrm{KL}(\mathbf{q} \,\|\, \mathbf{p}(\mathbf{x}))
    = \sum_{i} q_i \log \frac{q_i}{p_i(\mathbf{x})},
\end{equation}
where \(i\) indexes over all classes.

\noindent Since \(q_y = 0\) and \(q_i > 0\) only for \(i \in \mathcal{D}_{\text{retain}}^{(t)}\),
the gradient w.r.t. logits becomes:
\begin{equation}
    \frac{\partial \mathcal{L}_{\mathrm{KL}}}{\partial \ell_i}
    =
    \begin{cases}
    p_i(\mathbf{x}) - q_i, & i \in \mathcal{D}_{\text{retain}}^{(t)}, \\
    p_i(\mathbf{x}), & i \notin \mathcal{D}_{\text{retain}}^{(t)}.
    \end{cases}
\end{equation}

\noindent In particular, for the ground-truth class \(y\), if \(p_y(\boldsymbol{\mathrm{x}}) > 0\), gradient descent decreases \(\ell_y\).
Meanwhile, logits of retain classes \(p_i(\boldsymbol{\mathrm{x}})\) are pushed toward matching \(q_i\).
At a stationary point, 
\(p_y(\boldsymbol{\mathrm{x}})\) approaches 0, leading to 
\(\ell_y \ll \max_{j \in \mathcal{D}_{\text{retain}}^{(t)}} \ell_j\).
Because $\max_{k \neq y}\ell_k = \max_{j \in \mathcal{D}_{\text{retain}}^{(t)}}\ell_j$, we obtain
\(
UM(\mathbf{x}) = \ell_y - \max_{k \neq y}\ell_k < 0.
\)

\noindent Therefore, the forgotten class is consistently discouraged from re-entering its decision region.
Since this update is repeatedly applied across unlearning phases, negative unlearning margins consistently prevent forgetting reversal.
\end{proof}


\section{Experiment Details} 
\label{sec:exp_details}

\subsection{Implementation}
\label{sec:impl}

\textbf{Original Models.}
We trained ResNet-18 on CIFAR-100, ResNet-50 on VGGFace2, and ViT-B/16 on MUFAC from scratch. 
For optimization, we used SGD with a learning rate of 0.1 for CIFAR-100 and VGGFace2, and 0.001 for MUFAC, along with weight decay of 0.0005 and momentum of 0.9.
We set the batch size to 128 for CIFAR-100 and 64 for VGGFace2 and MUFAC. 

\noindent\textbf{SAFER and Baselines.}
Retrain models were trained in the same manner as original models, using only the remaining data \(\mathcal{D}_{\text{retain}}^{(t)}\).
For the other baseline methods, we utilized publicly available source code. 

\noindent We perform 10 unlearning epochs for all experiments, and the optimal hyperparameters are selected using Optuna~\cite{akiba2019optuna}.
All methods are tuned only in Phase 1, and the selected hyperparameters are fixed and reused for all subsequent phases to ensure a fair comparison without phase-wise re-optimization. This optimization protocol is applied to all experimental results in this paper unless otherwise noted.

\noindent\textbf{Hardware Specifications.} 
All experiments are conducted on Ubuntu 18.04 using NVIDIA GeForce RTX 3090 GPUs, each equipped with 24,268 MB of memory.

\subsection{Unlearning Targets}
\label{subsec:unlearning_targets}

\cref{tab:unlearning_targets} describes the unlearning targets used in our experiments.
At each phase, we unlearn 3 classes for CIFAR-100 and VGGFace2, and 20 identities for MUFAC.
The three-phase unlearning process is considered as a single evaluation unit, and we repeat it three times with different targets. The reported results are averaged over the three runs.

\input{tables/supp_table_targets}

\section{Additional Analysis} 
\label{sec:add_analysis}

\noindent\textbf{Ablation Study.}
As presented in~\cref{tab:ablation_results_tow}, we conduct ablation studies to evaluate the impact of each component in SAFER.
The ablation results are evaluated under the same conditions described in~\cref{sec:experiment}, and for each phase, the highest ToW score is highlighted in bold.

\noindent First, when optimizing intra-class compactness (IC), the ToW score gradually decreases after the first phase in both class-aligned and class-misaligned unlearning settings. 
This suggests that repeated unlearning with only IC optimization causes the model to diverge from the retrained model (the ideal unlearning reference) in both model utility and unlearning efficacy.
Second, adding centroid drift suppression to intra-class compactness (IC+CD) yields consistent performance improvements in the class-aligned setting compared to IC alone. However, its effect remains minimal in the class-misaligned setting, suggesting that controlling centroid drift has limited influence under this scenario.
Lastly, when all components are incorporated (UM+IC+CD), the ToW scores are significantly higher than those of the other settings across all datasets and phases. This result implies that jointly optimizing unlearning margin and retain-representation clusterability is crucial to preserving both model utility and effective forgetting throughout multiple unlearning phases.

\input{tables/supp_table_ablation_tow}

\noindent\textbf{Results in Extended Phases.}
A downward trend of SAFER is observed for MUFAC in~\cref{fig:tow_all}. Therefore, we conduct additional experiments up to five phases under the same experimental setting.
As described above, the phase 5 results are obtained using the same hyperparameters as in the main paper, without any additional tuning. 
In this experiment, we include GS-Lora as an additional baseline, using the authors' default hyperparameters.
As shown in~\cref{fig:baseline_comparison}, SAFER remains competitive with strong baselines and does not exhibit a persistent downward trend beyond phase 3.
It consistently maintains a ToW above 0.8, indicating that the discrepancy from the Retrain model does not accumulate over extended phases, even in the class-misaligned setting. 
These results indicate that SAFER maintains competitive and stable performance beyond three phases.


\noindent\textbf{Hyperparameter Sensitivity.}
To further examine the robustness of our proposed method, we provide a sensitivity analysis of the regularization weight $\lambda$ in Eq. (6), which controls cluster separation. As shown in~\cref{fig:lambda_sensitivity}, varying $\lambda$ results in highly consistent trends up to phase 5, without significant performance fluctuations. These results indicate that SAFER shows limited sensitivity and remains stable across a reasonable range of hyperparameter values.
\begin{figure}[!tbp]
\centering
    \begin{subfigure}[]{0.45\columnwidth}
    \centering
    \includegraphics[width=\columnwidth]{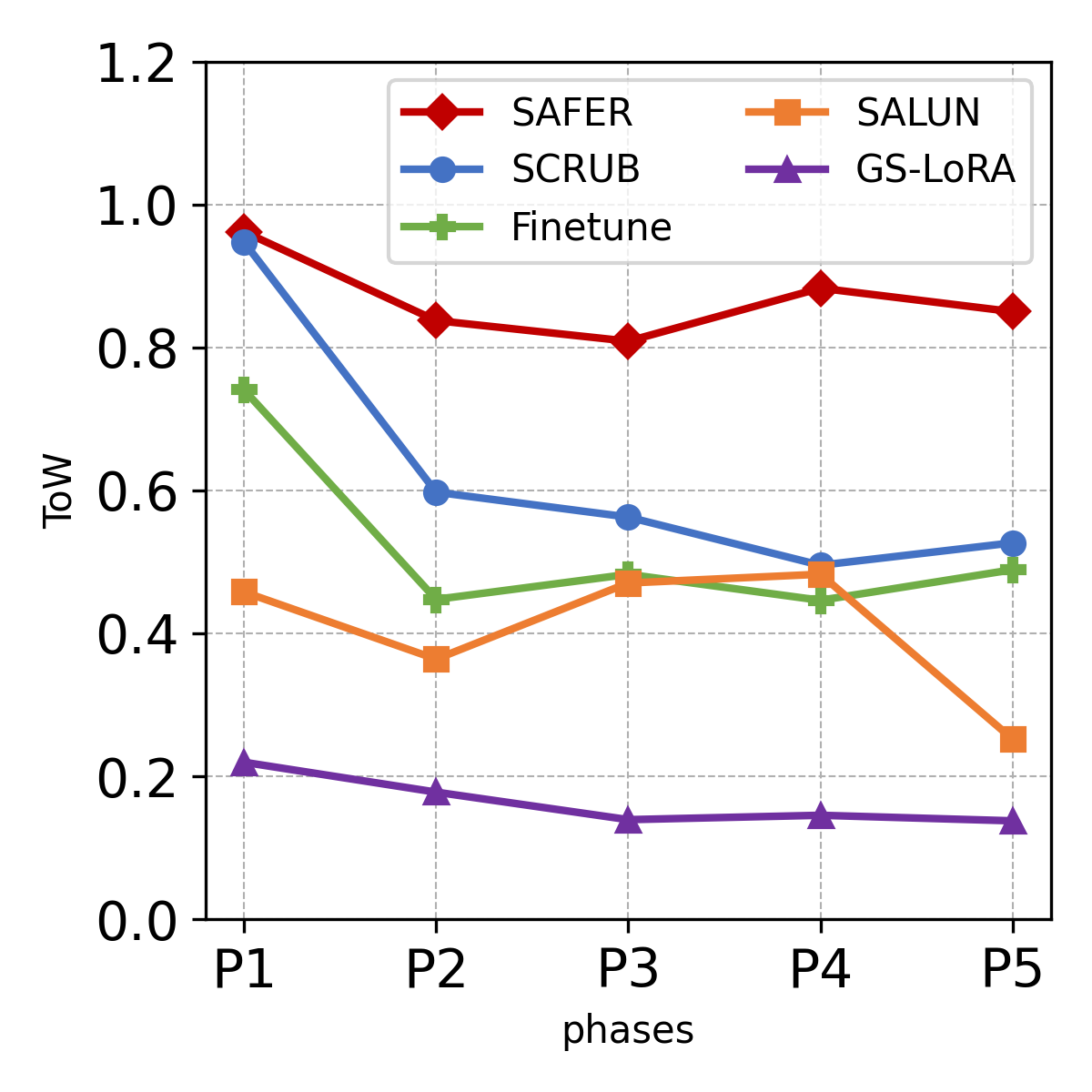}
    \caption{Baseline comparison}
    \label{fig:baseline_comparison}
    \end{subfigure}
    \begin{subfigure}[]{0.45\columnwidth}
    \centering
    \includegraphics[width=\columnwidth]{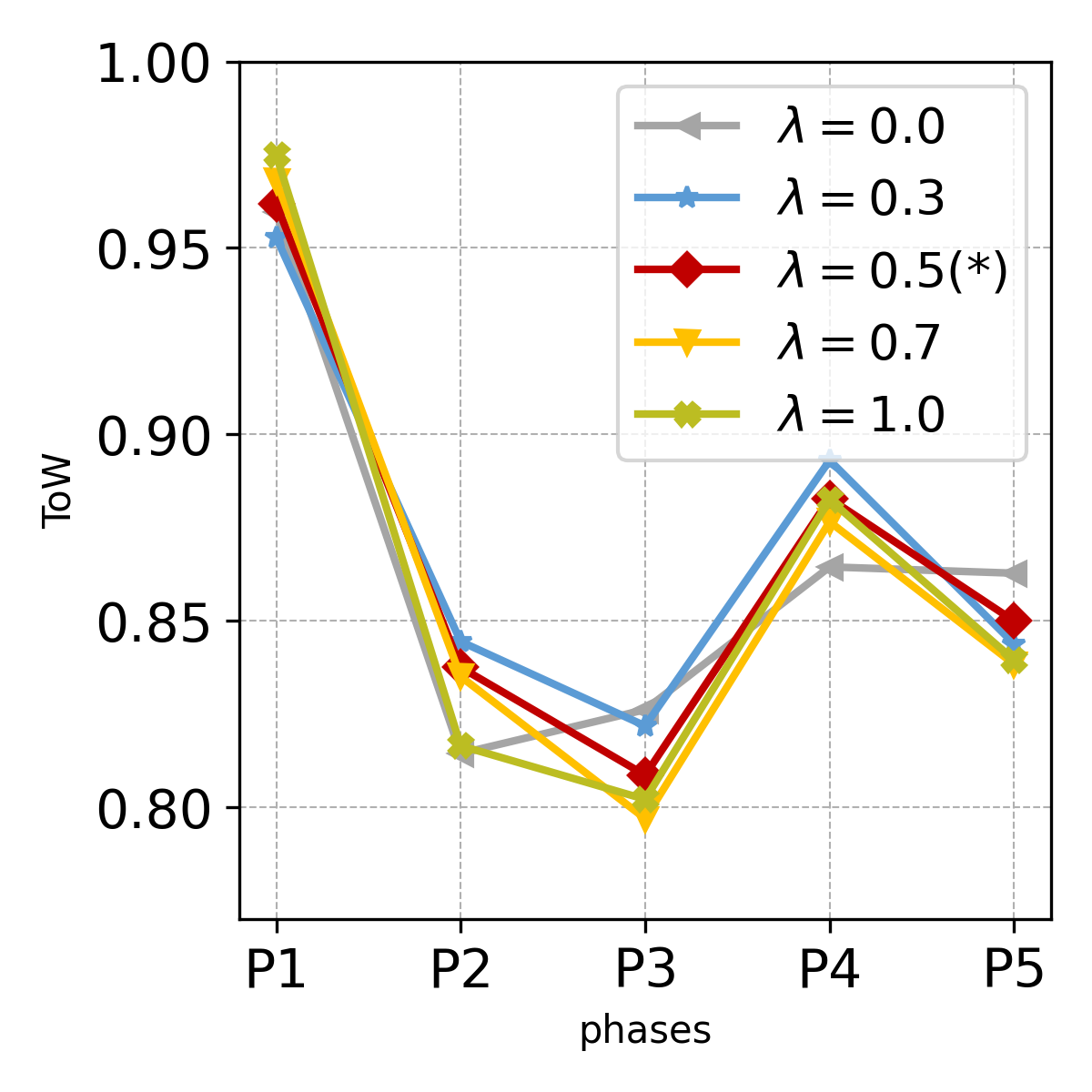}
    \caption{Hyperparameter sensitivity}
    \label{fig:lambda_sensitivity}
    \end{subfigure}
\caption{ToW performance across phases 1-5.}
\label{fig:phase5}
\vspace{-1.5 em}
\end{figure}


\noindent\textbf{Representation Similarity.}
To further evaluate representation stability under continual unlearning settings, we measure the cosine similarity between the feature representations extracted by the model at each unlearning phase and those of the original model before any unlearning is applied.

\noindent For a given sample \(x_i\), let the original model and the model after \(t\)-th unlearning phase produce feature representations:
\begin{equation}
    \mathbf{f}^{\text{before}}_i = f_{\theta_0}(x_i),
    \qquad
    \mathbf{f}^{\text{after}}_i = f_{\theta_t}(x_i),
\end{equation}
where \(t \in \{1, 2, 3\}\).
Applying \(L2\)-normalization, we measure the representation similarity using cosine similarity:
\begin{equation}
    \text{Sim}(x_i)
    =
    \frac{
    \mathbf{f}^{\text{before}}_i \cdot \mathbf{f}^{\text{after}}_i
    }{
    \|\mathbf{f}^{\text{before}}_i\|_2 \,
    \|\mathbf{f}^{\text{after}}_i\|_2
    }
\end{equation}
, where \(\text{Sim}(x_i) \in [-1, 1]\).
We obtain the similarity distributions for retain data and forget data as:
\begin{equation}
    \text{Sim}_{\text{retain}}^{(t)} 
    = 
    \{\text{Sim}(x_i) \mid i \in \mathcal{D}_{\text{retain}}^{(t)}\}, 
    \qquad
    \text{Sim}_{\text{forget}}^{(t)}
    = 
    \{\text{Sim}(x_i) \mid i \notin \mathcal{D}_{\text{retain}}^{(t)}\}.
\end{equation}

\noindent The underlying rationale is that an ideal unlearning algorithm should remove information only related to the forgotten samples, while preserving knowledge from the retain samples. 
Thus, high similarity for retain samples across unlearning phases indicates that the model successfully maintains useful knowledge without feature distortion. 
Conversely, decreasing similarity for forget samples suggests that their embeddings diverge from their original feature space, reflecting effective forgetting.

\noindent ~\cref{fig:c100_similarity}, ~\cref{fig:vgg2_similarity}, and ~\cref{fig:mufac_similarity} present the representation similarity results between the original model and the unlearned models at each phase for CIFAR-100, VGGFace2, and MUFAC.
For the retrained model in the class-aligned unlearning setting, the retain and forget samples form distinguishable distributions in terms of their representation similarity to the original model. 
The retain-sample distribution is shifted toward the right with a narrower spread, indicating consistently high similarity. 
In contrast, the forget-sample distributions across all unlearning phases shift further left with a broader spread, indicating that their similarity to the original model remains uniformly low.
In the class-misaligned unlearning setting, the retrained model exhibits retain and forget similarity distributions with nearly identical shapes, although the forget distribution tends to be slightly more dispersed.
Among all approximate unlearning methods, SAFER most closely preserves the representation similarity characteristics of Retrain. 
This indicates that SAFER achieves desirable and consistent unlearning behavior, not only in accuracy-based evaluation but also in this representational analysis.

\input{figs/fig_similarity_c100}
\input{figs/fig_similarity_vgg2}
\input{figs/fig_similarity_mufac}

%% file: tables/supp_table_targets.tex
\begin{table}[h!]
\centering
    \setlength{\tabcolsep}{6pt}
    \resizebox{0.5\columnwidth}{!}{
    \begin{tabular}{c c c c c}
    \toprule
    Set & Phase & CIFAR-100 & VGGFace2 & MUFAC \\
    \midrule
    \multirow{3}{*}{1}  & 1 & 31,33,35 & 31,33,35 & 250 - 259, 260 - 269  \\
     & 2 & 37,39,51 & 37,39,51 & 270 - 279, 280 - 289 \\
     & 3 & 53,55,57 & 53,55,57 & 700 - 709, 710 - 719 \\
    \midrule
    \multirow{3}{*}{2}  & 1 & 71,63,20 & 6,7,8     & 230 - 239, 320 - 329 \\
     & 2 & 35,17,81 & 26,27,28  & 510 - 519, 600 - 609 \\
     & 3 & 9,18,15  & 51,52,53  & 610 - 619, 620 - 629 \\
    \midrule
    \multirow{3}{*}{3} & 1 & 6,7,8    & 36,37,38  & 240 - 249, 340 - 349 \\
     & 2 & 26,27,28 & 2,3,4     & 540 - 549, 850 - 859 \\
     & 3 & 51,52,53 & 81,82,83  & 860 - 869, 870 - 879 \\
    \bottomrule
    \end{tabular}
}
\caption{Unlearning target classes for CIFAR-100, VGGFace2, and MUFAC across three phases in our experiments.}
\label{tab:unlearning_targets}
\end{table}

%% file: tables/supp_table_ablation_tow.tex
\begin{table}[t]
\centering
\setlength{\tabcolsep}{11pt}
\resizebox{0.5\columnwidth}{!}{
\begin{tabular}{c|c|ccc}
\toprule
Dataset & Phase & IC & IC+CD & UM+IC+CD \\
\midrule

\multirow{3}{*}{CIFAR-100}
& 1 & 0.8687 & 0.9857 & \textbf{0.9910}\\
& 2 & 0.6052 & 0.8635 & \textbf{0.9889}\\
& 3 & 0.4621 & 0.3944 & \textbf{0.9981} \\
\midrule

\multirow{3}{*}{VGGFace2}
& 1 & 0.6890 & 0.9806 & \textbf{0.9824} \\
& 2 & 0.4655 & 0.7456 & \textbf{0.9844} \\
& 3 & 0.6697 & 0.8090 & \textbf{0.9887} \\
\midrule

\multirow{3}{*}{MUFAC}
& 1 & 0.7372 & 0.7371 & \textbf{0.9618} \\
& 2 & 0.4563 & 0.4445 & \textbf{0.8375} \\
& 3 & 0.4895 & 0.4783 & \textbf{0.8086} \\
\bottomrule
\end{tabular}
}
\caption{Ablation study for CIFAR-100, VGGFace2, and MUFAC. The results represent ToW scores, and higher values indicate better unlearning performance.
UM corresponds to the optimization in~\cref{eq:random_logitdist}, whereas IC and CD optimize the intra-class compactness and centroid drift suppression terms in~\cref{eq:retain_loss}, respectively. 
SAFER includes all three components (UM+IC+CD).
}
\vspace{-1.5 em}
\label{tab:ablation_results_tow}
\end{table}

%% file: figs/fig_similarity_c100.tex
\begin{figure*}[t]
\centering
\setlength{\tabcolsep}{1pt}
\begin{tabular}{c|ccc}
\toprule
Method & {Phase 1} & {Phase 2} & {Phase 3}\\
\midrule

\textbf{Retrain} &
\includegraphics[width=0.22\textwidth]{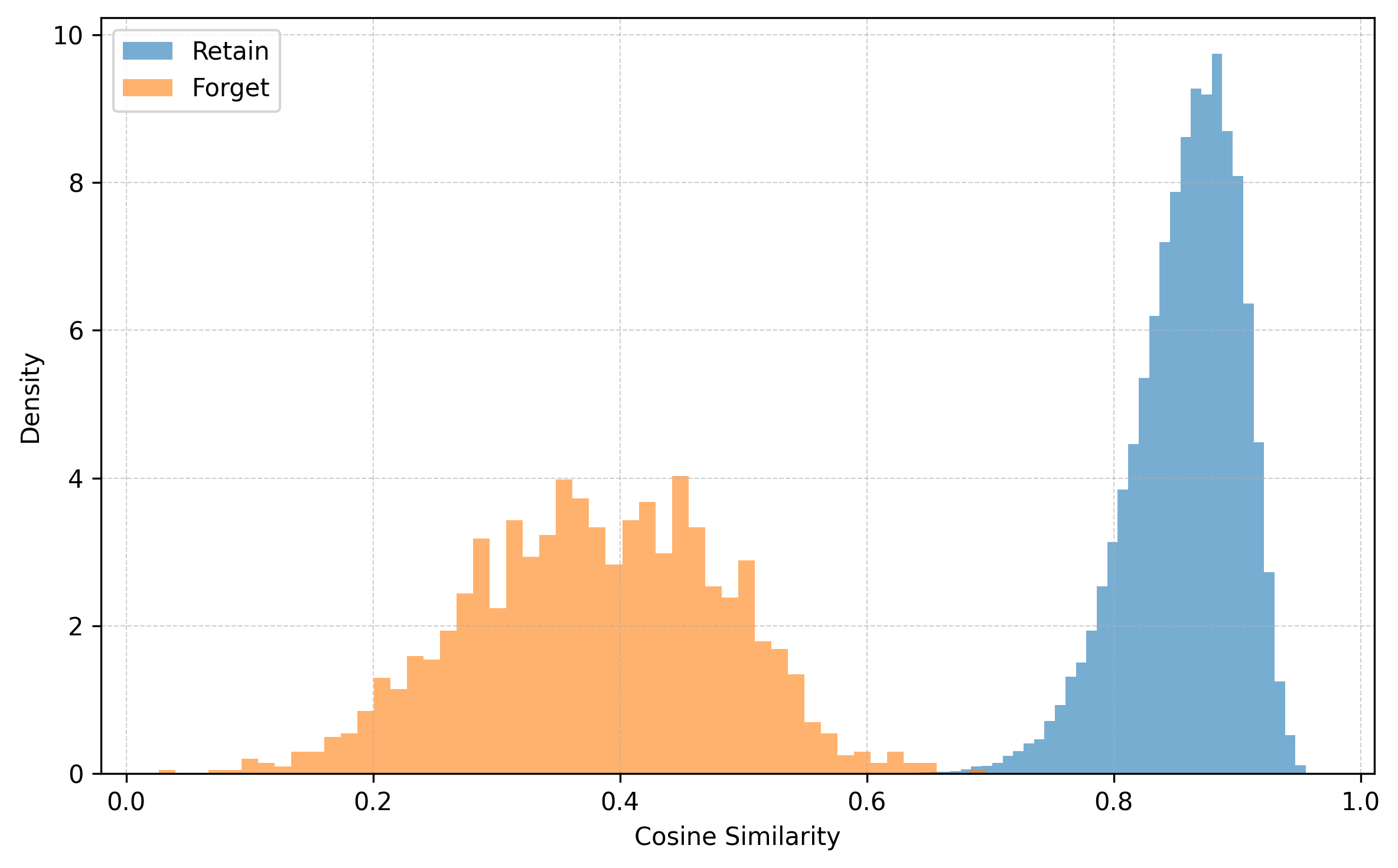} &
\includegraphics[width=0.22\textwidth]{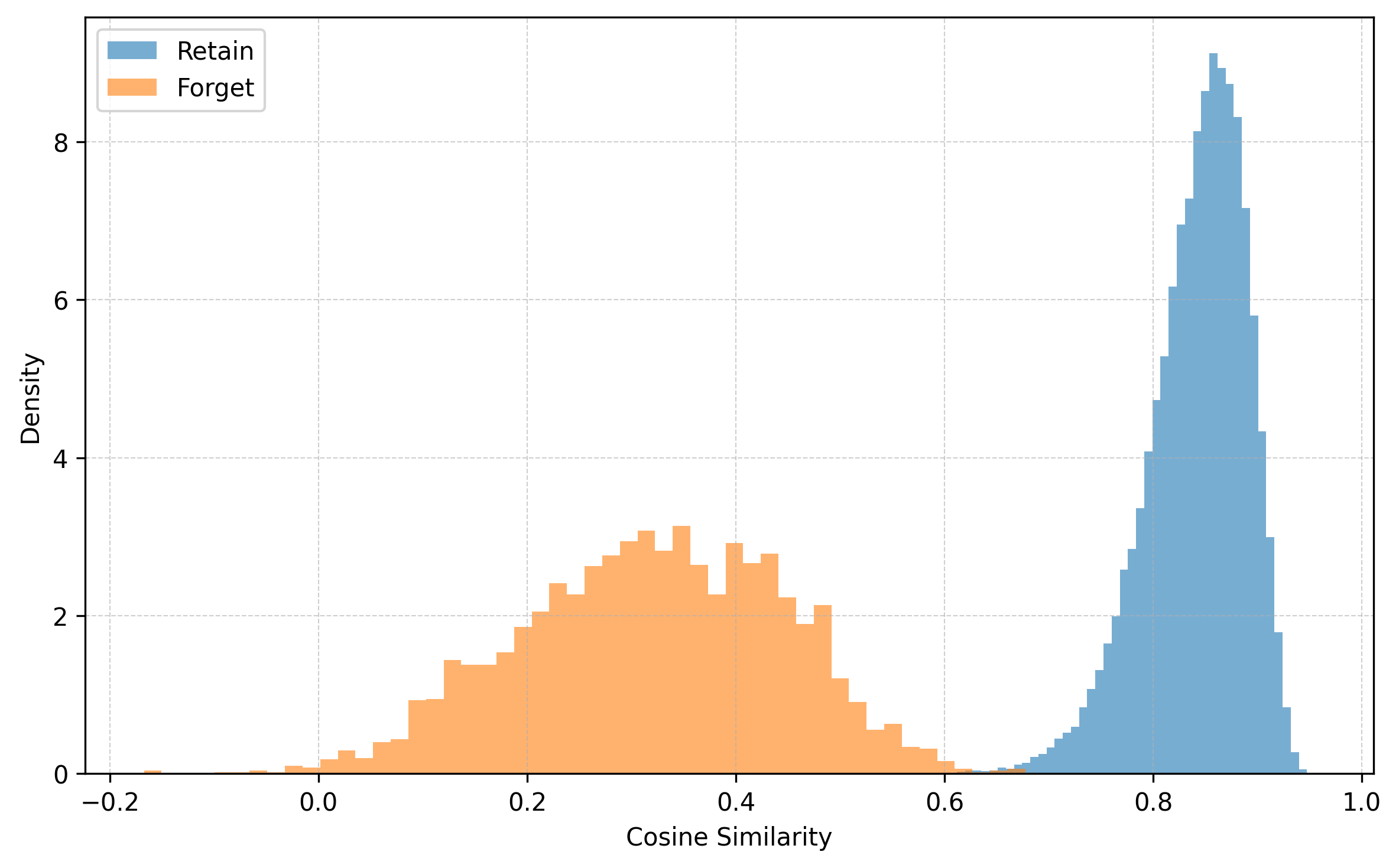} &
\includegraphics[width=0.22\textwidth]{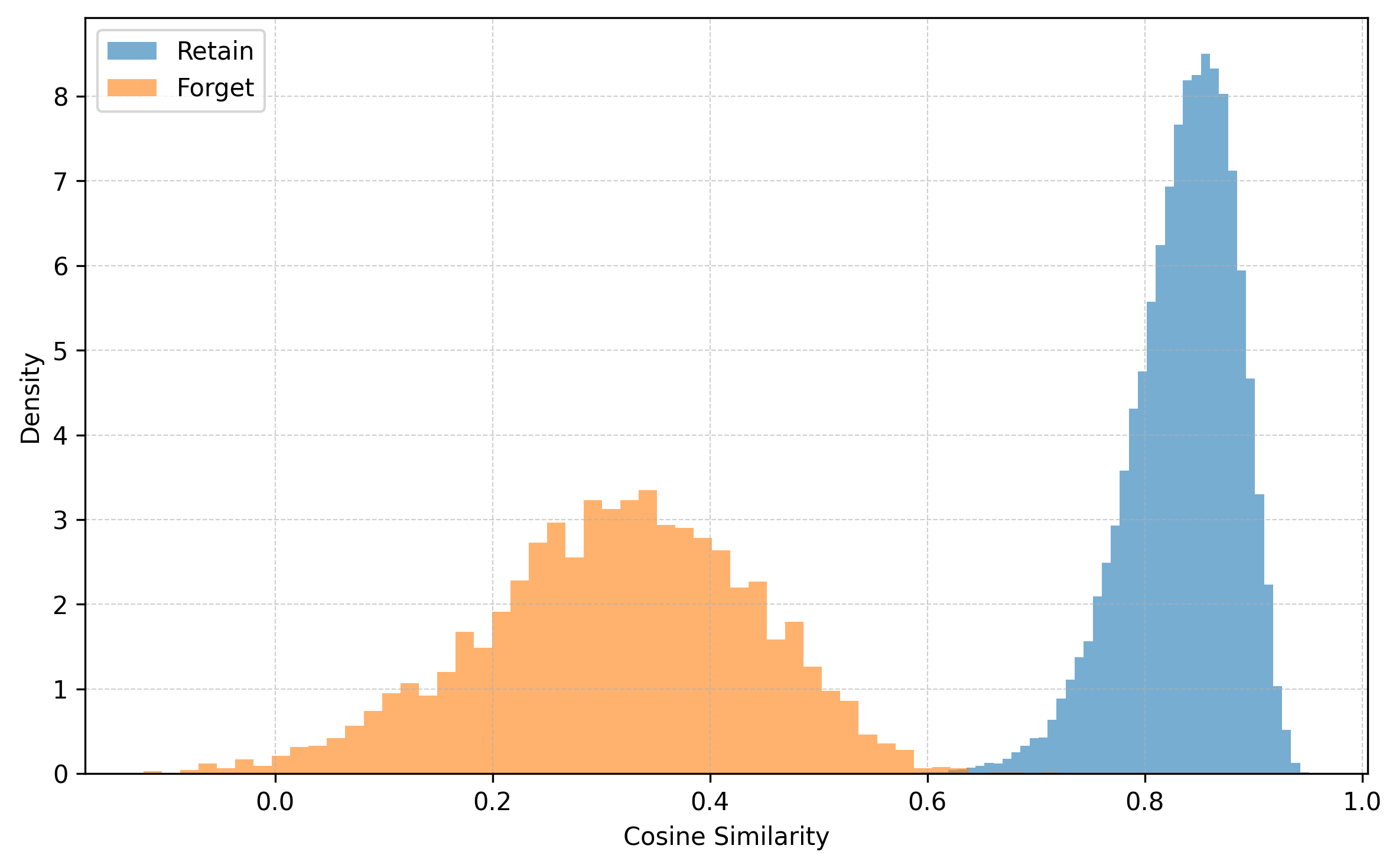} \\
\midrule

\textbf{SCRUB} &
\includegraphics[width=0.22\textwidth]{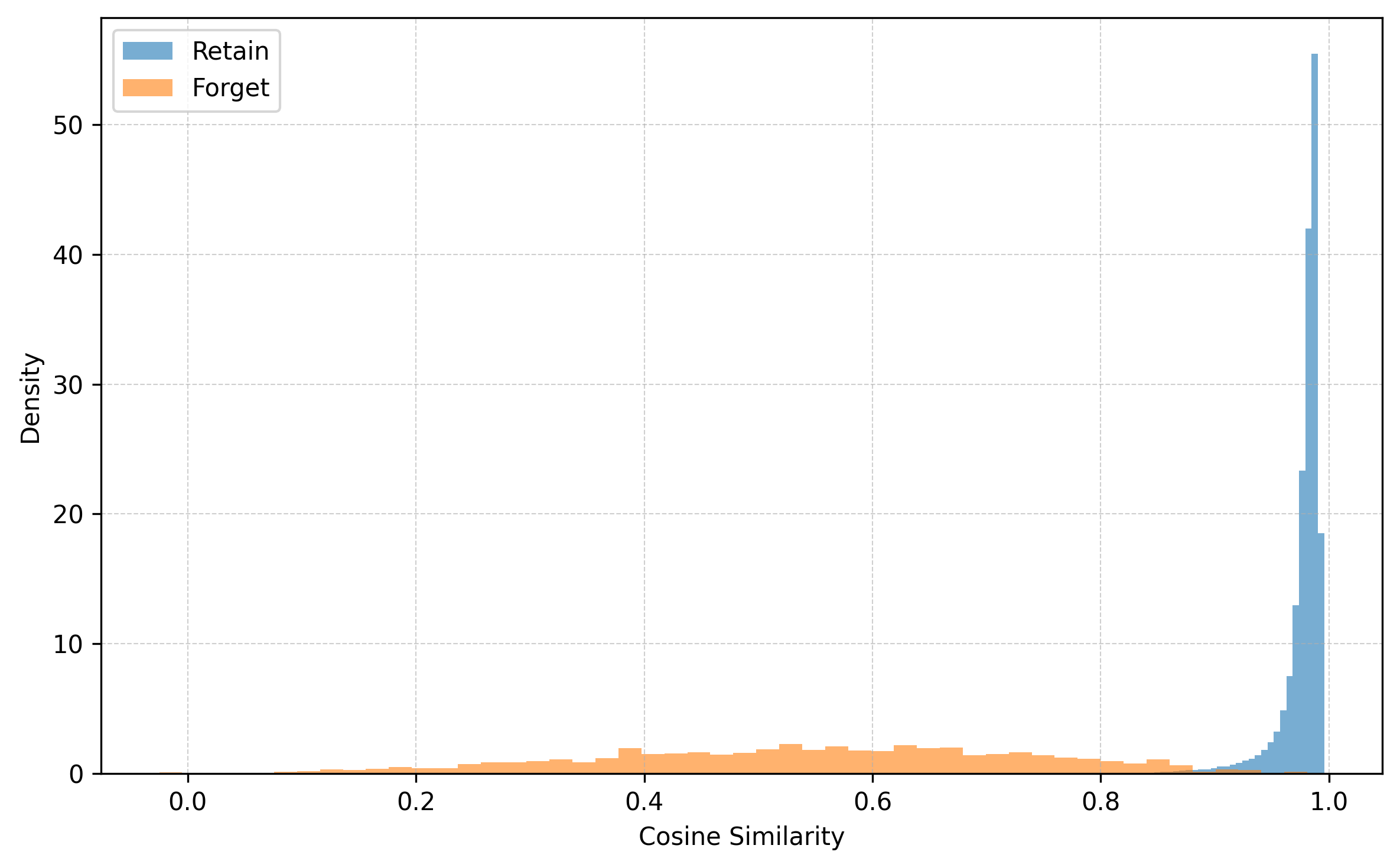} &
\includegraphics[width=0.22\textwidth]{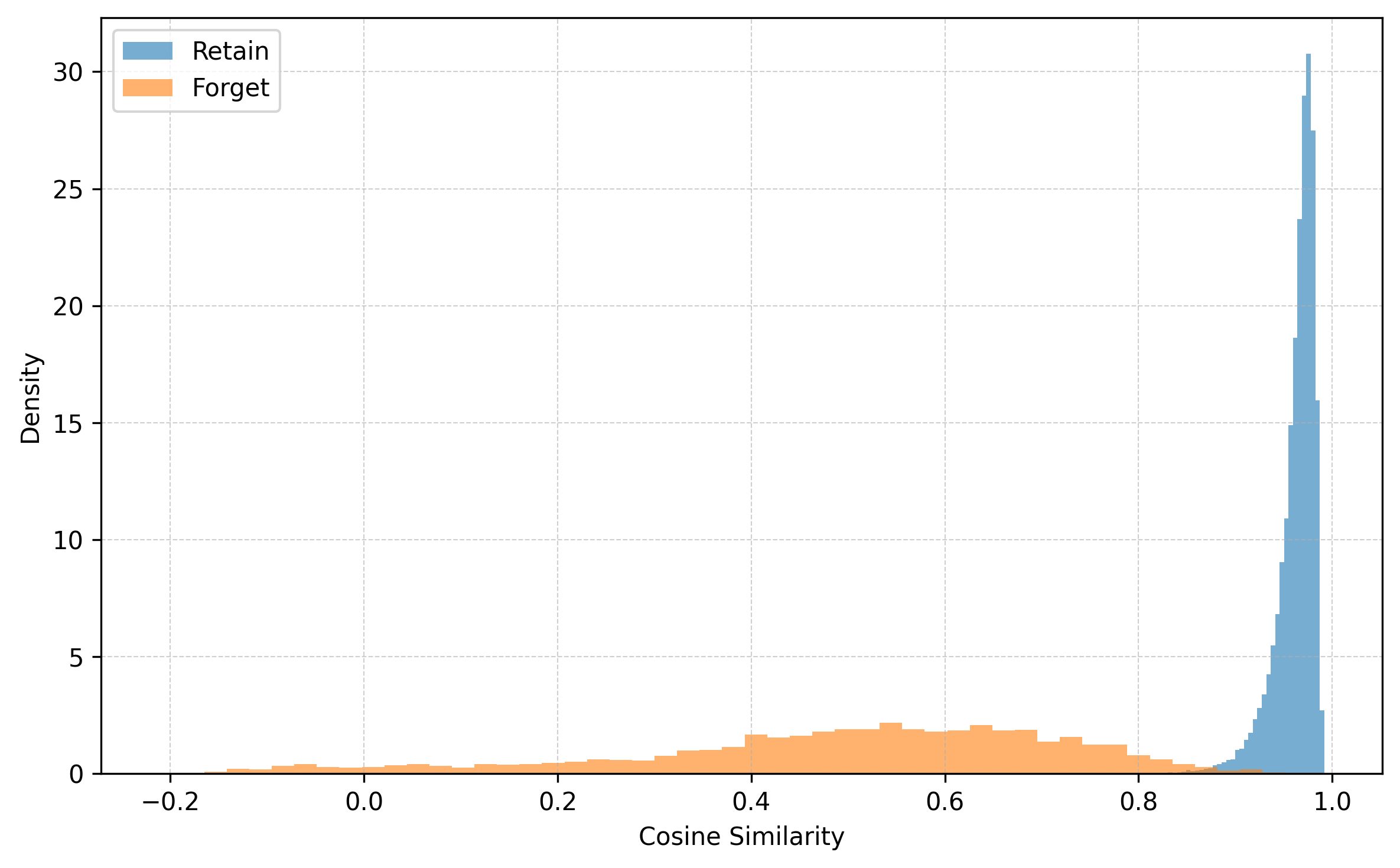} &
\includegraphics[width=0.22\textwidth]{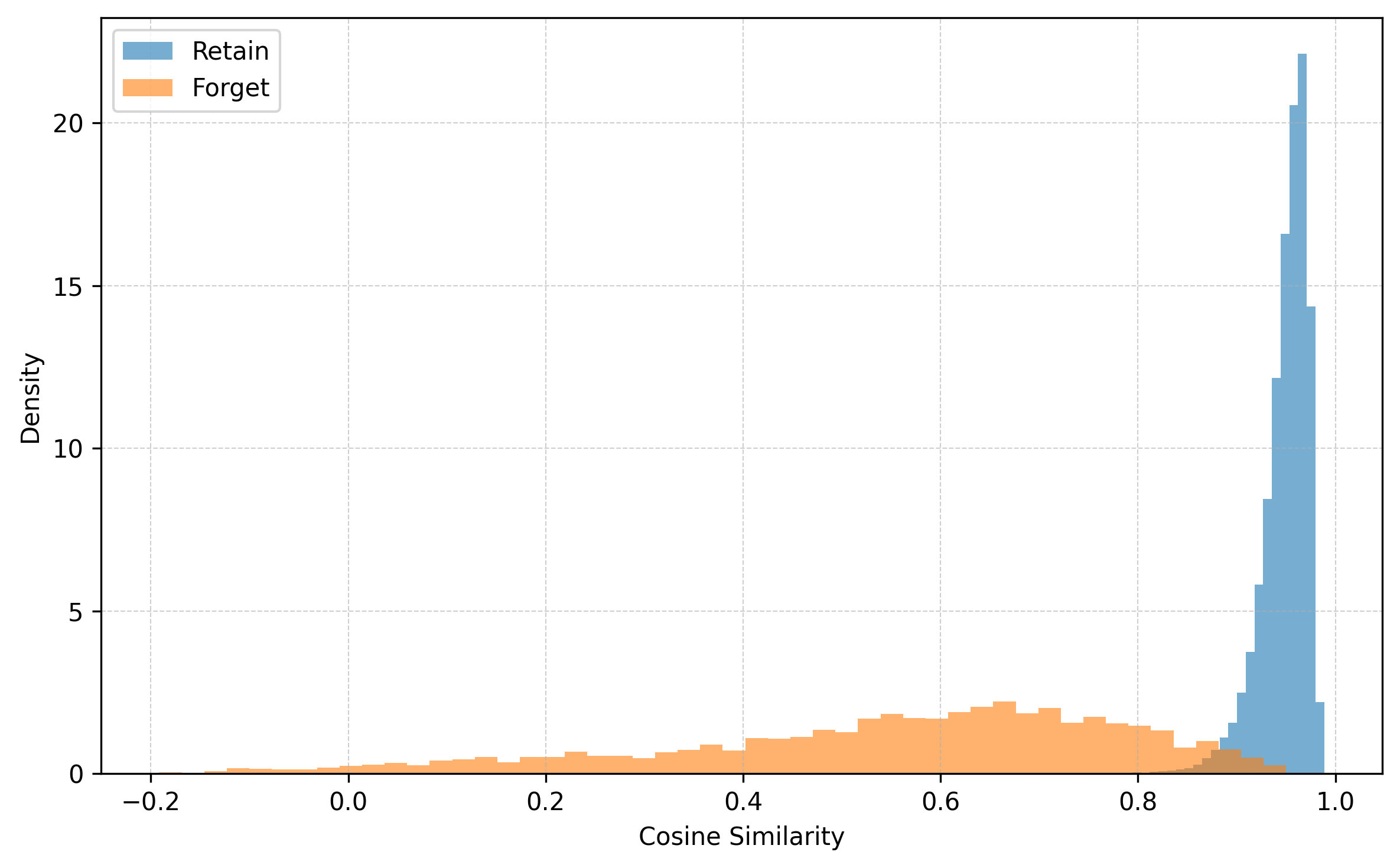} \\
\midrule

\textbf{SALUN} &
\includegraphics[width=0.22\textwidth]{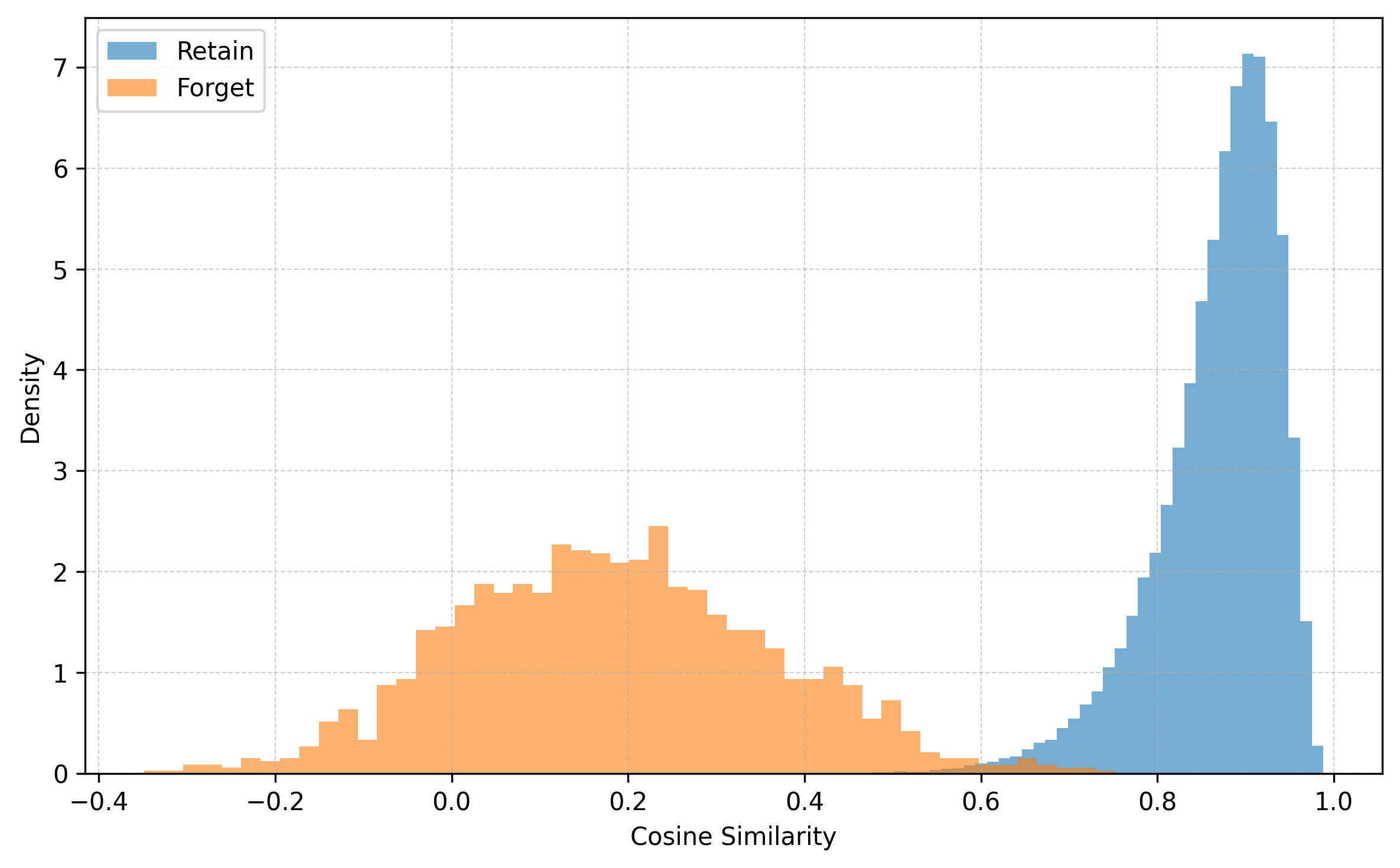} &
\includegraphics[width=0.22\textwidth]{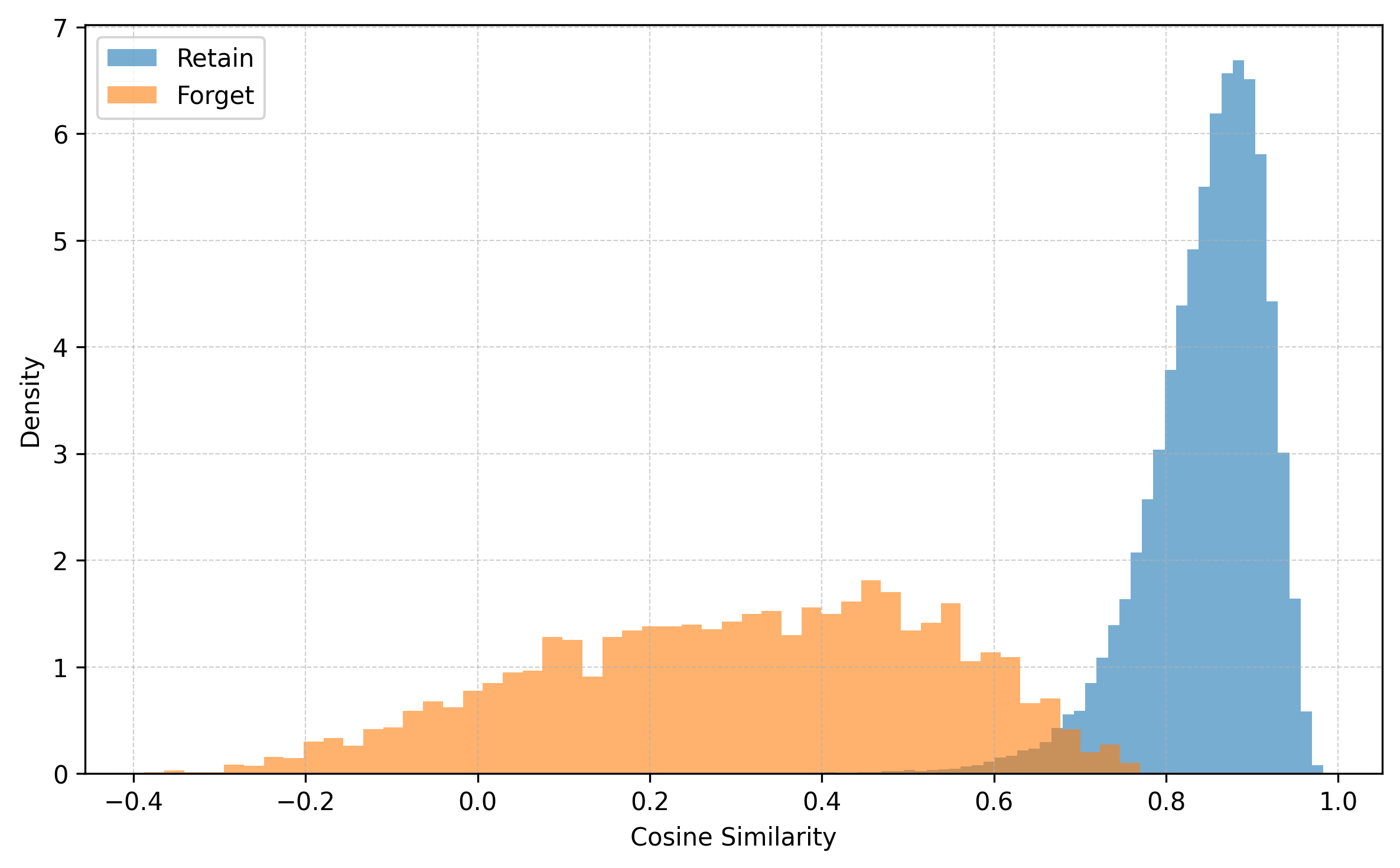} &
\includegraphics[width=0.22\textwidth]{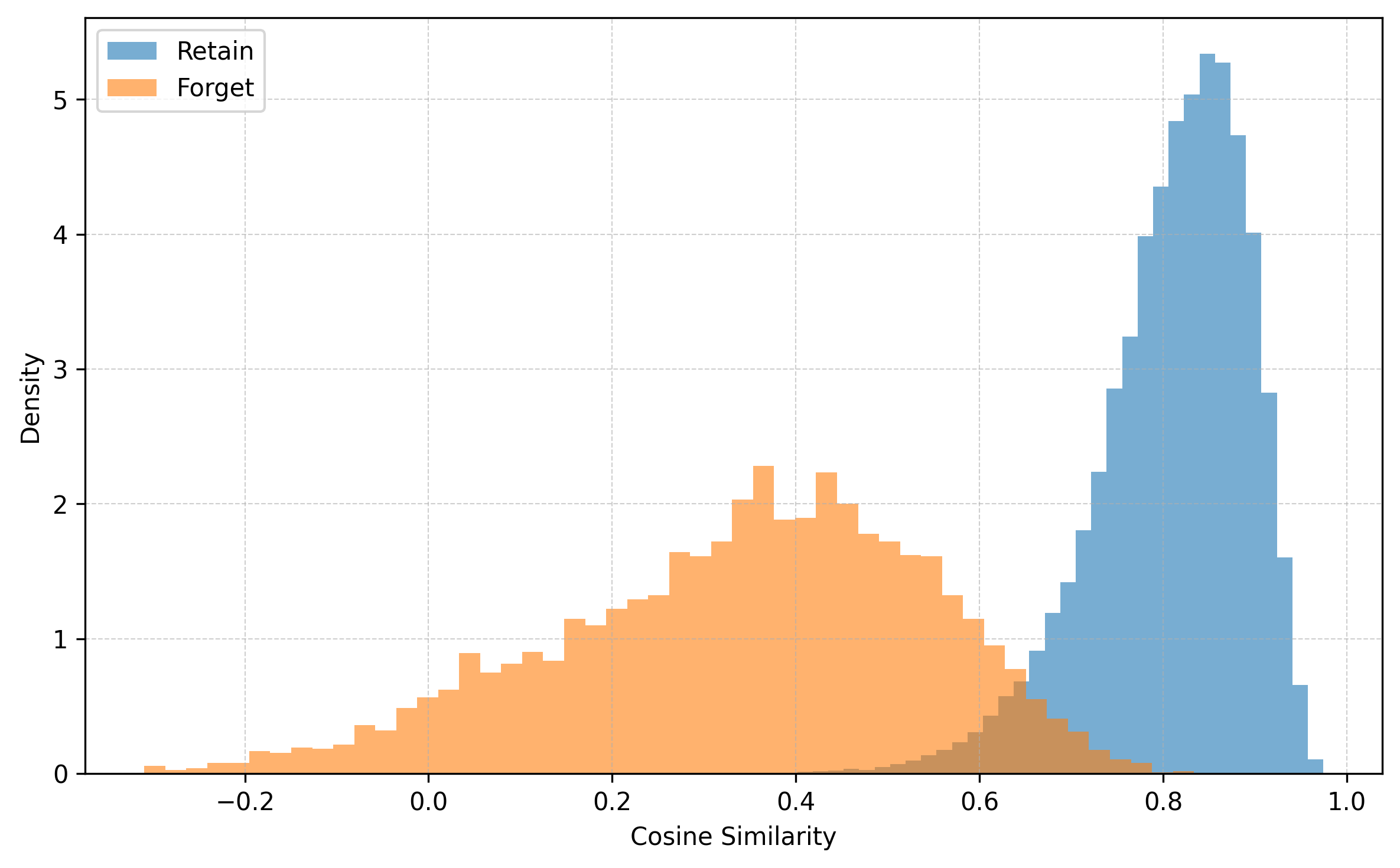} \\
\midrule

\textbf{SSD} &
\includegraphics[width=0.22\textwidth]{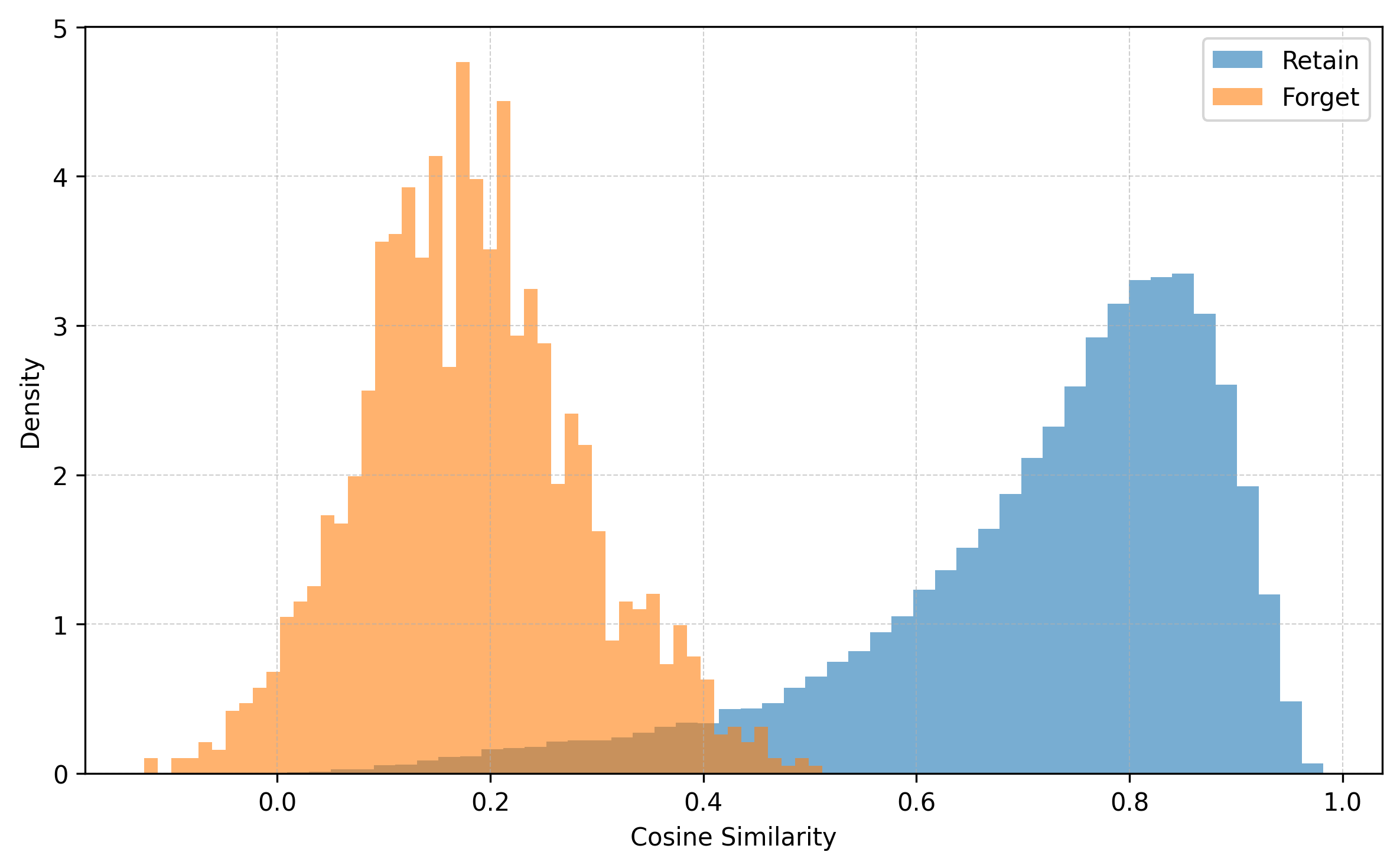} &
\includegraphics[width=0.22\textwidth]{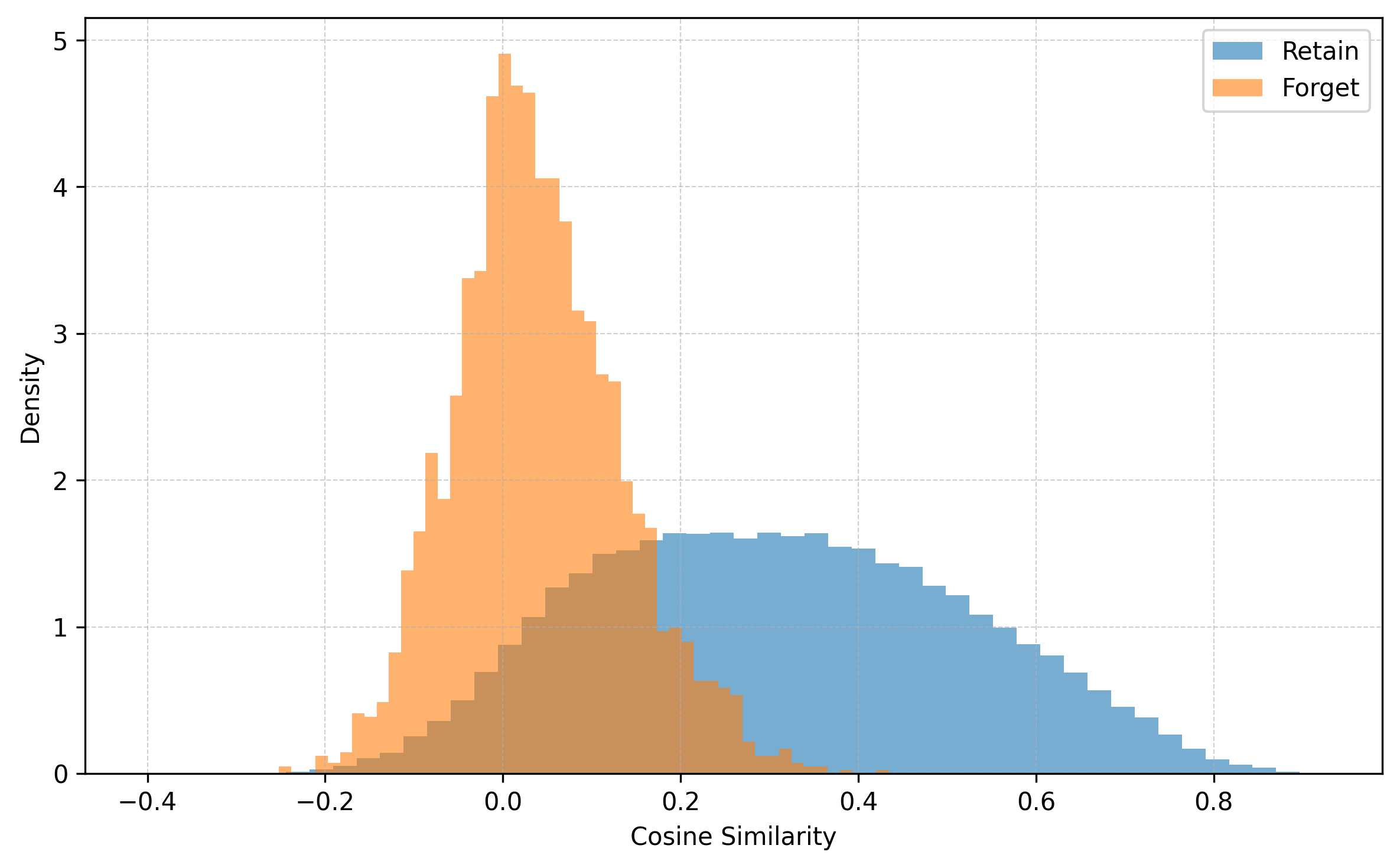} &
\includegraphics[width=0.22\textwidth]{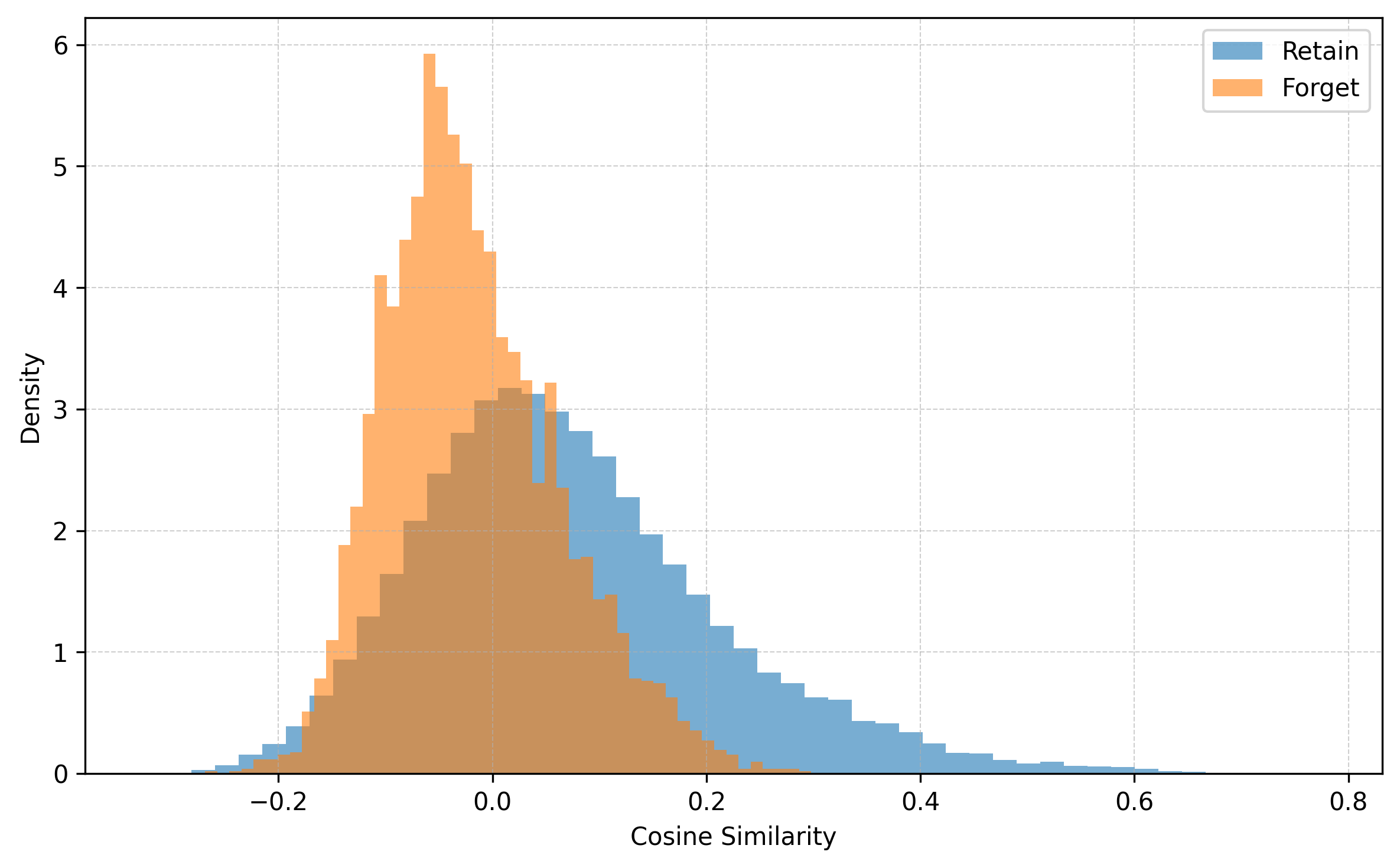} \\
\midrule

\textbf{BndShrink}  &
\includegraphics[width=0.22\textwidth]{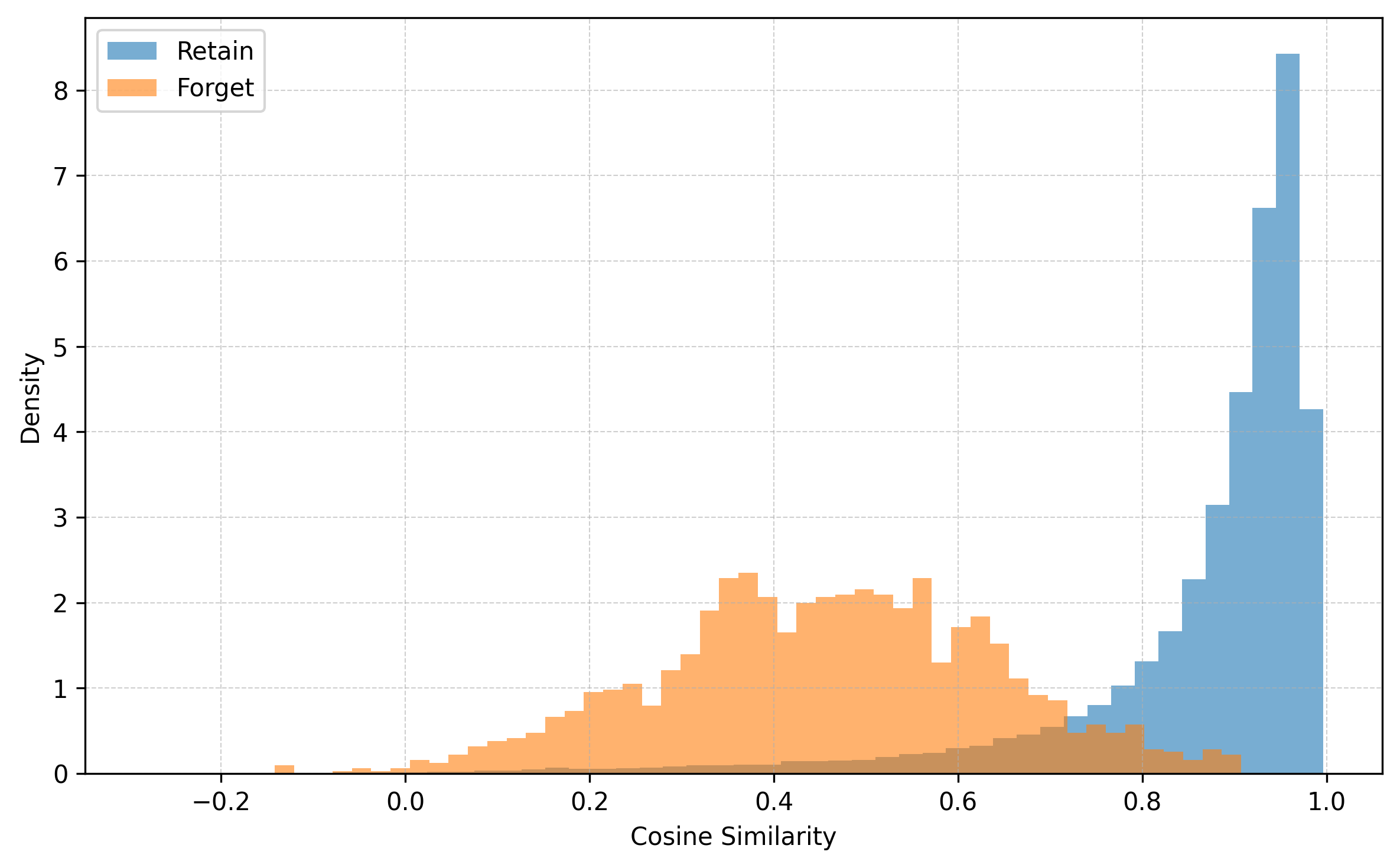} &
\includegraphics[width=0.22\textwidth]{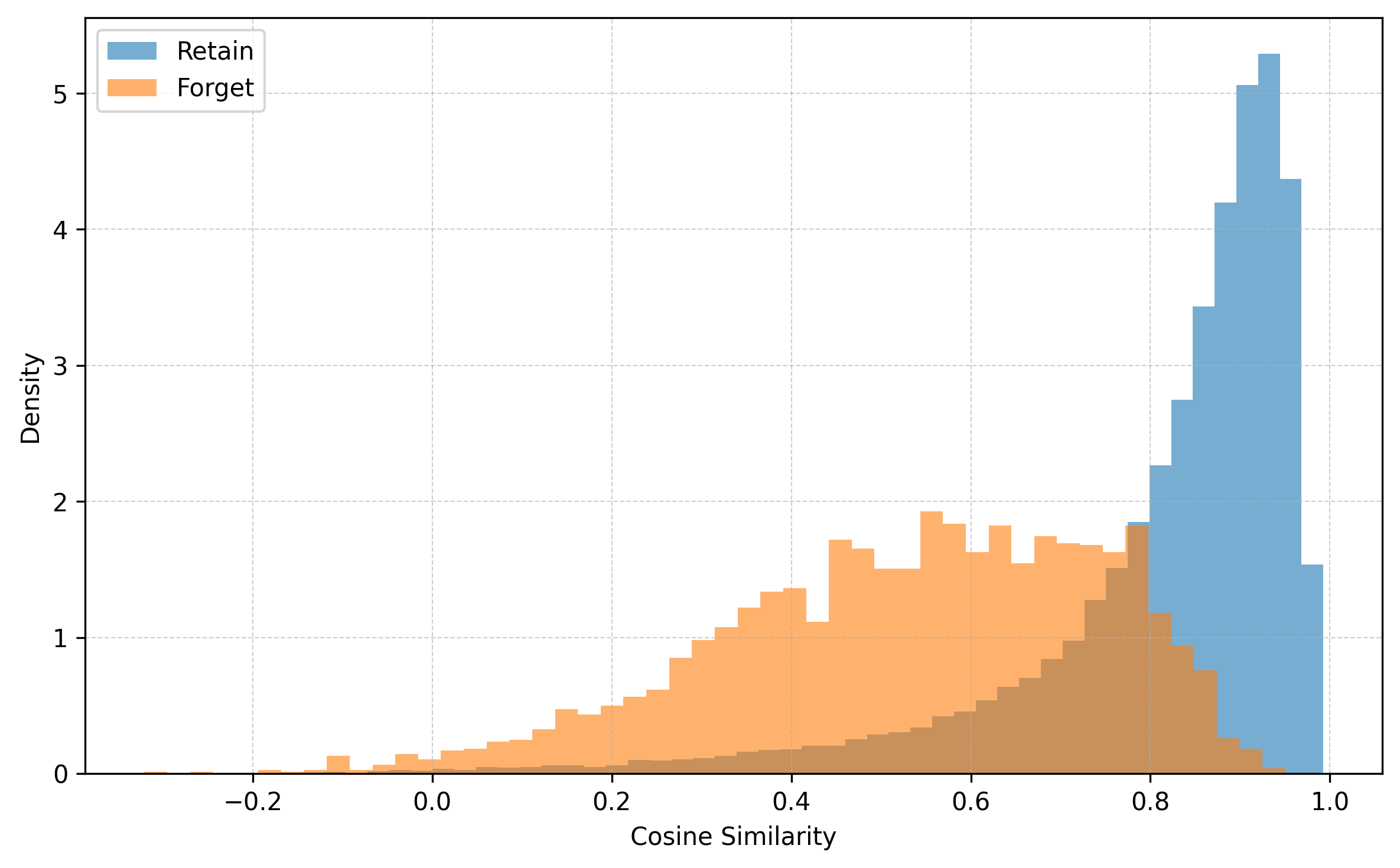} &
\includegraphics[width=0.22\textwidth]{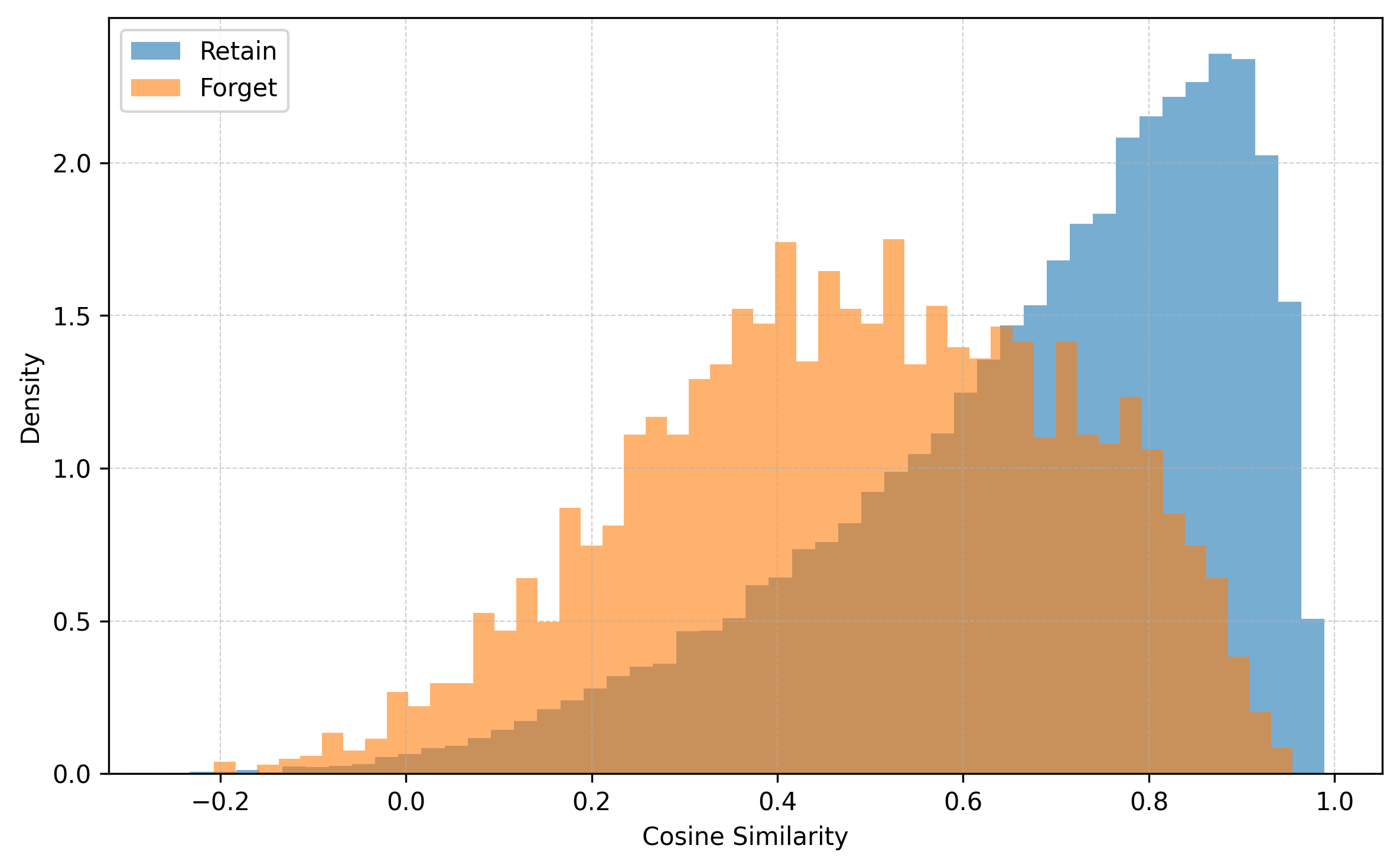} \\
\midrule

\textbf{NegGrad} &
\includegraphics[width=0.22\textwidth]{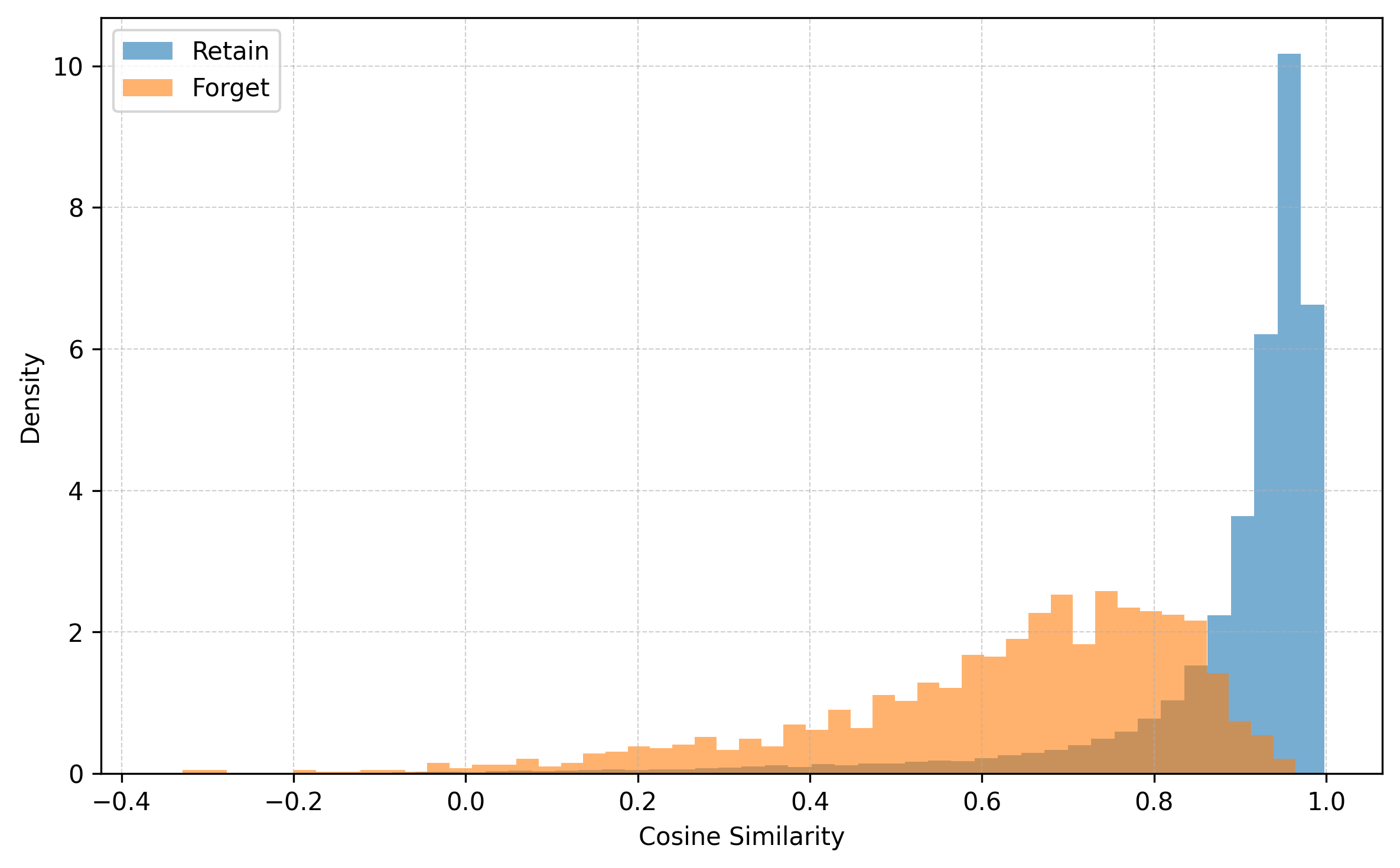} &
\includegraphics[width=0.22\textwidth]{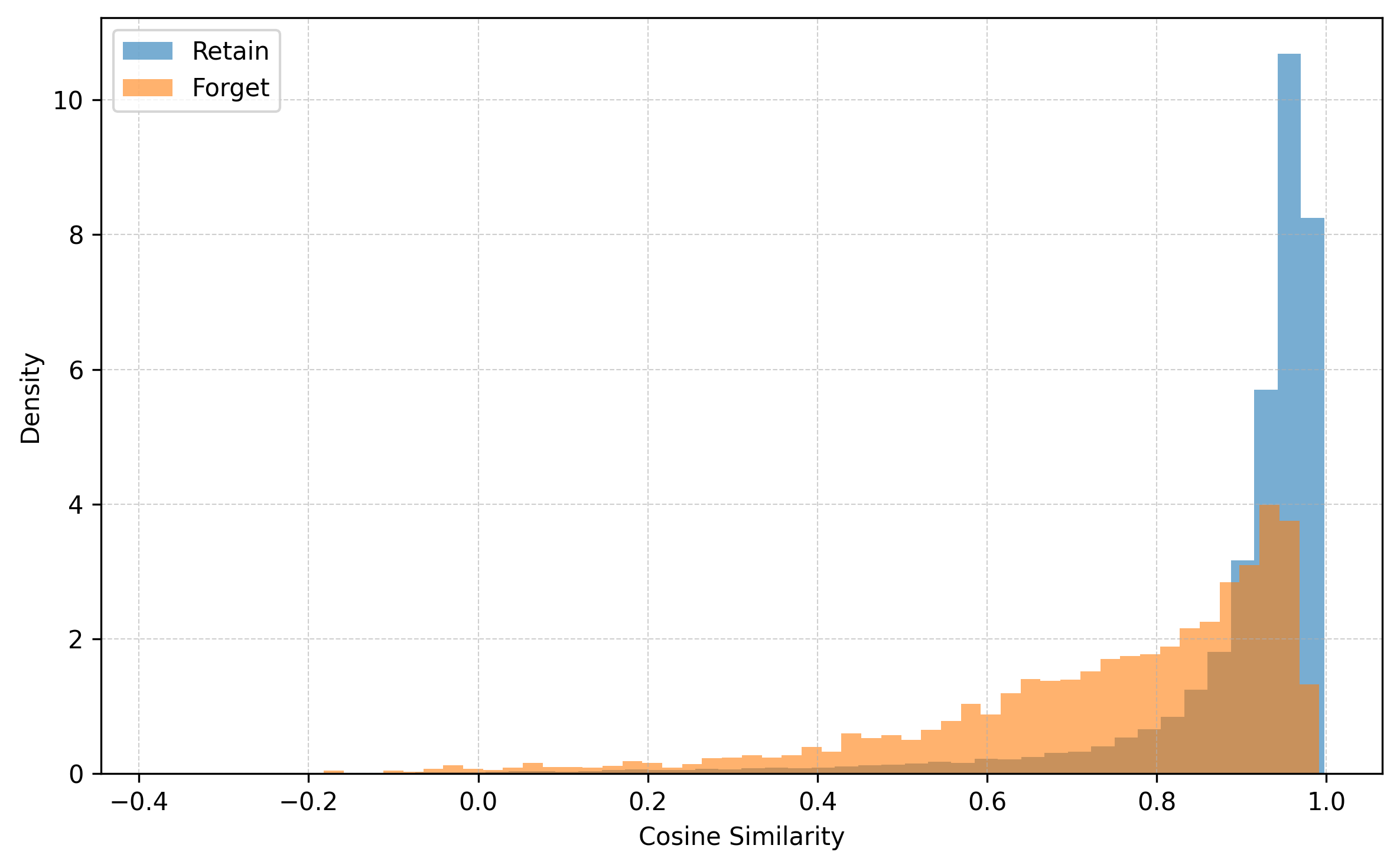} &
\includegraphics[width=0.22\textwidth]{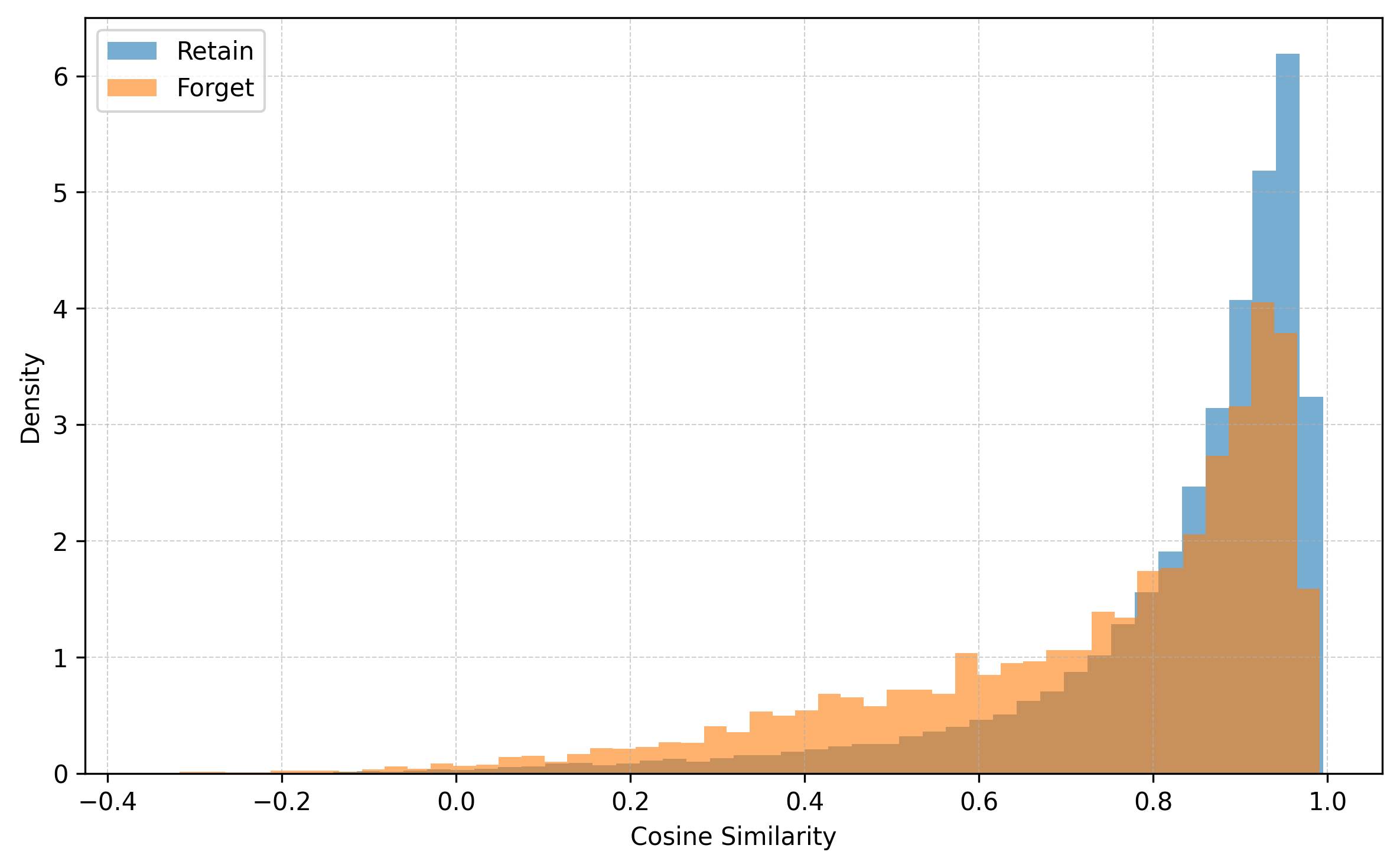} \\
\midrule

\textbf{Finetune} &
\includegraphics[width=0.22\textwidth]{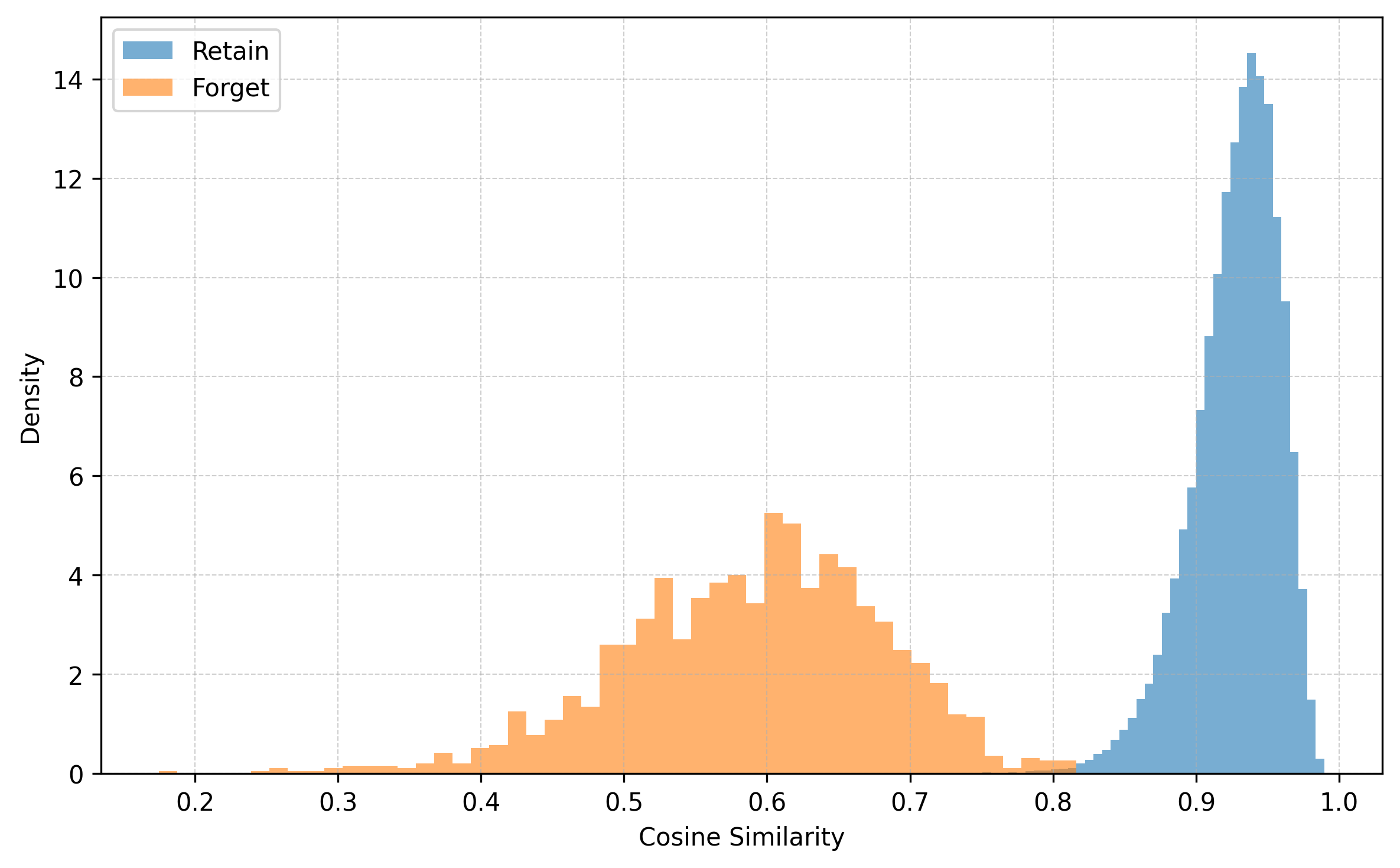} &
\includegraphics[width=0.22\textwidth]{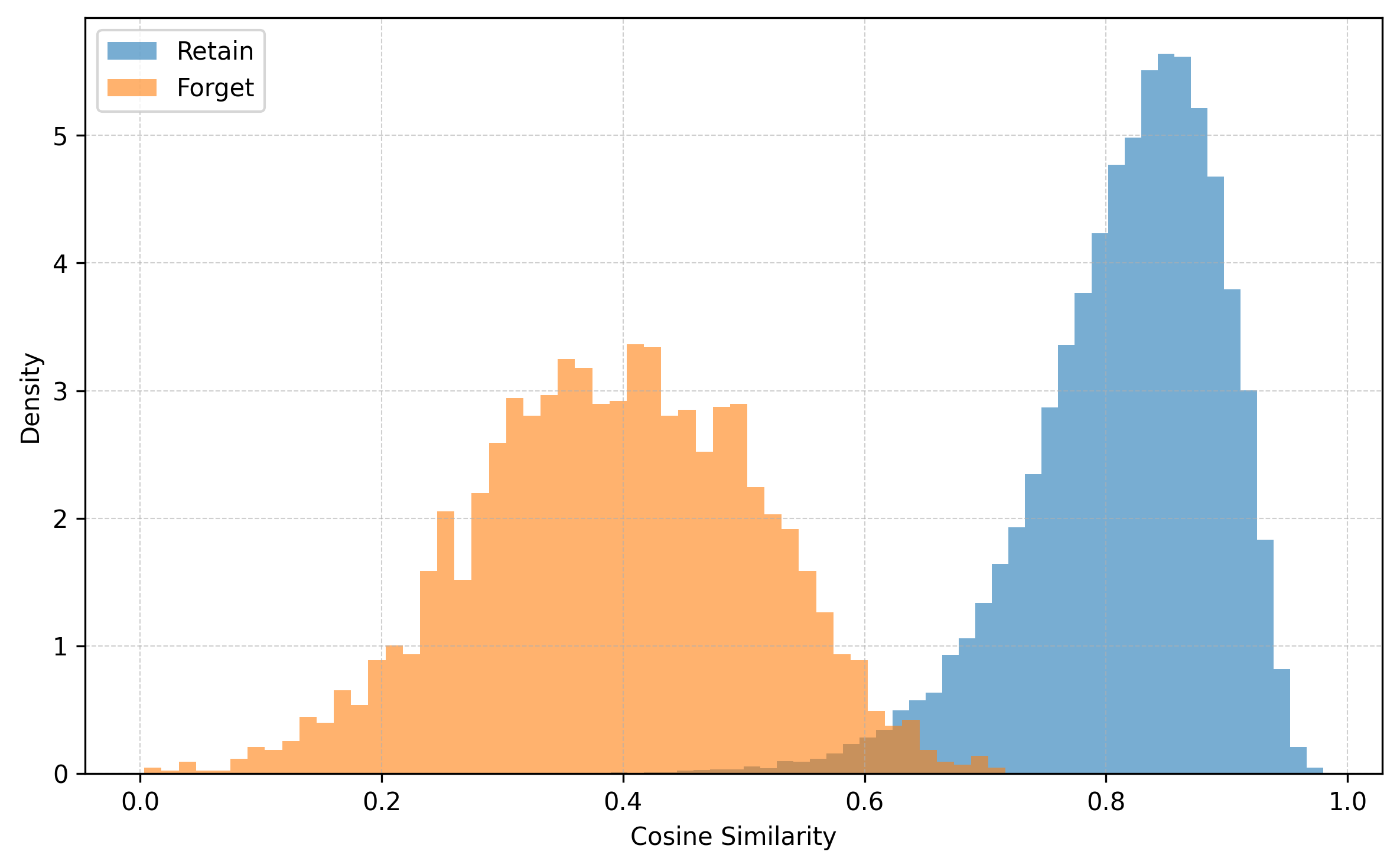} &
\includegraphics[width=0.22\textwidth]{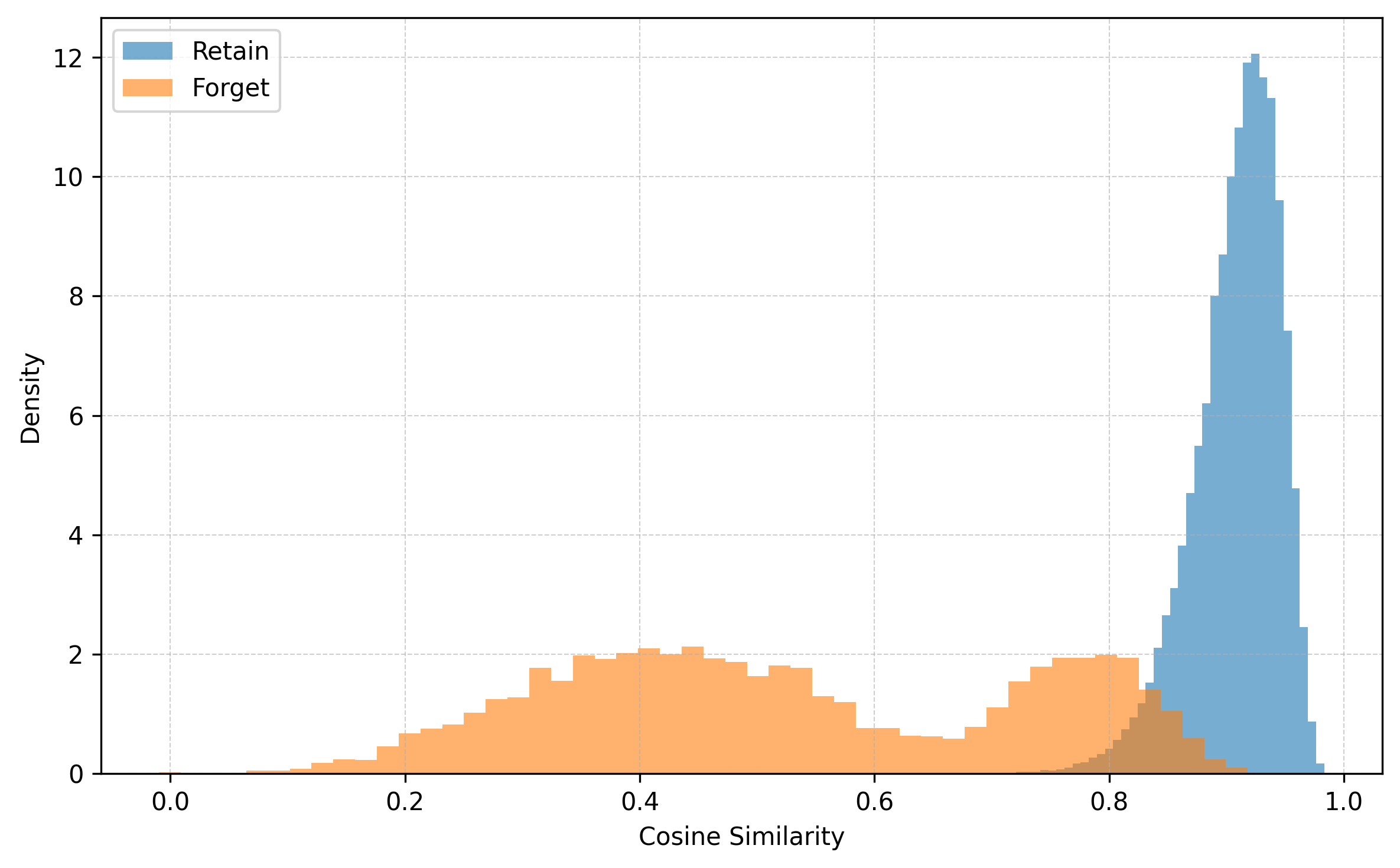} \\
\midrule

\textbf{SAFER} &
\includegraphics[width=0.22\textwidth]{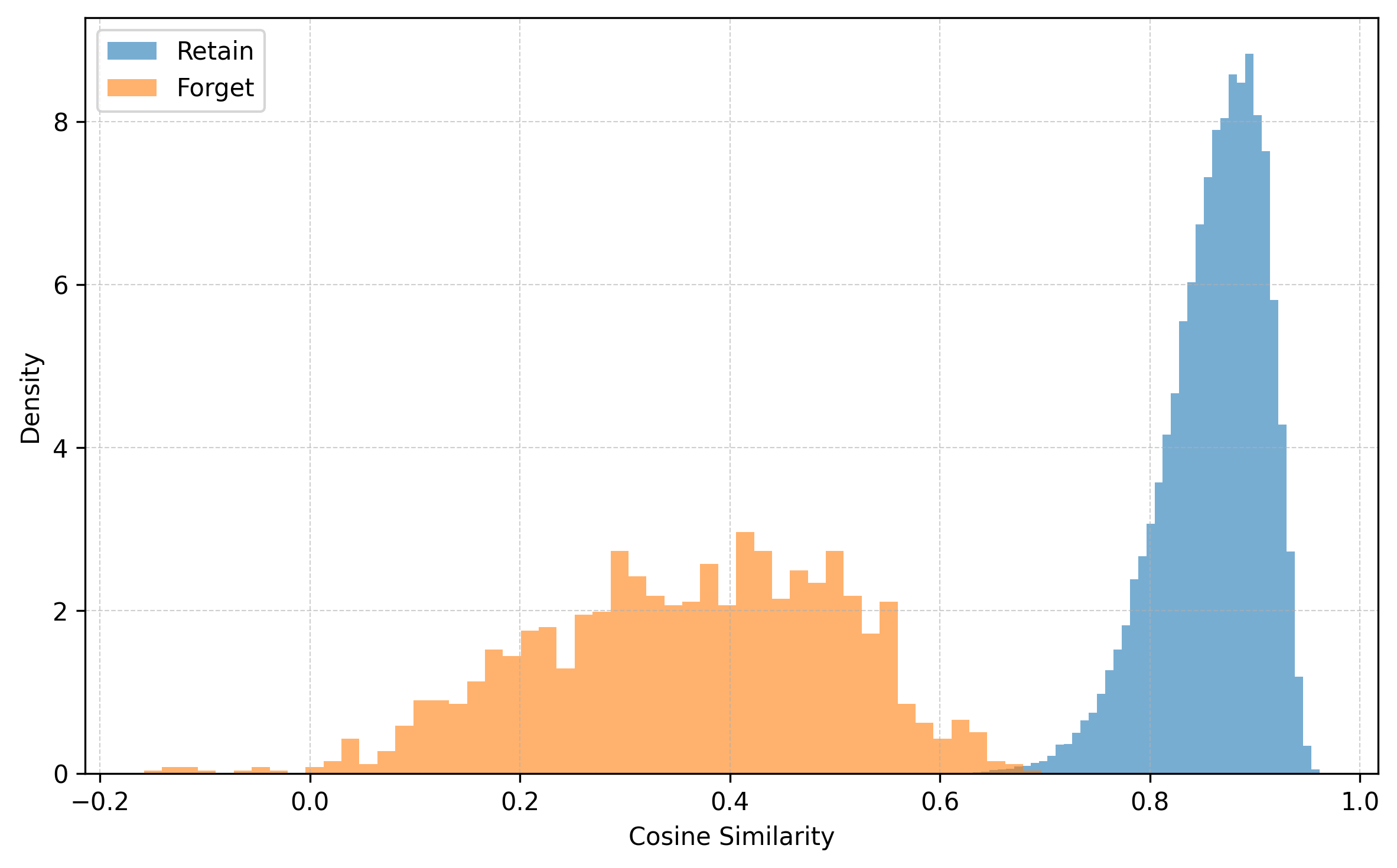} &
\includegraphics[width=0.22\textwidth]{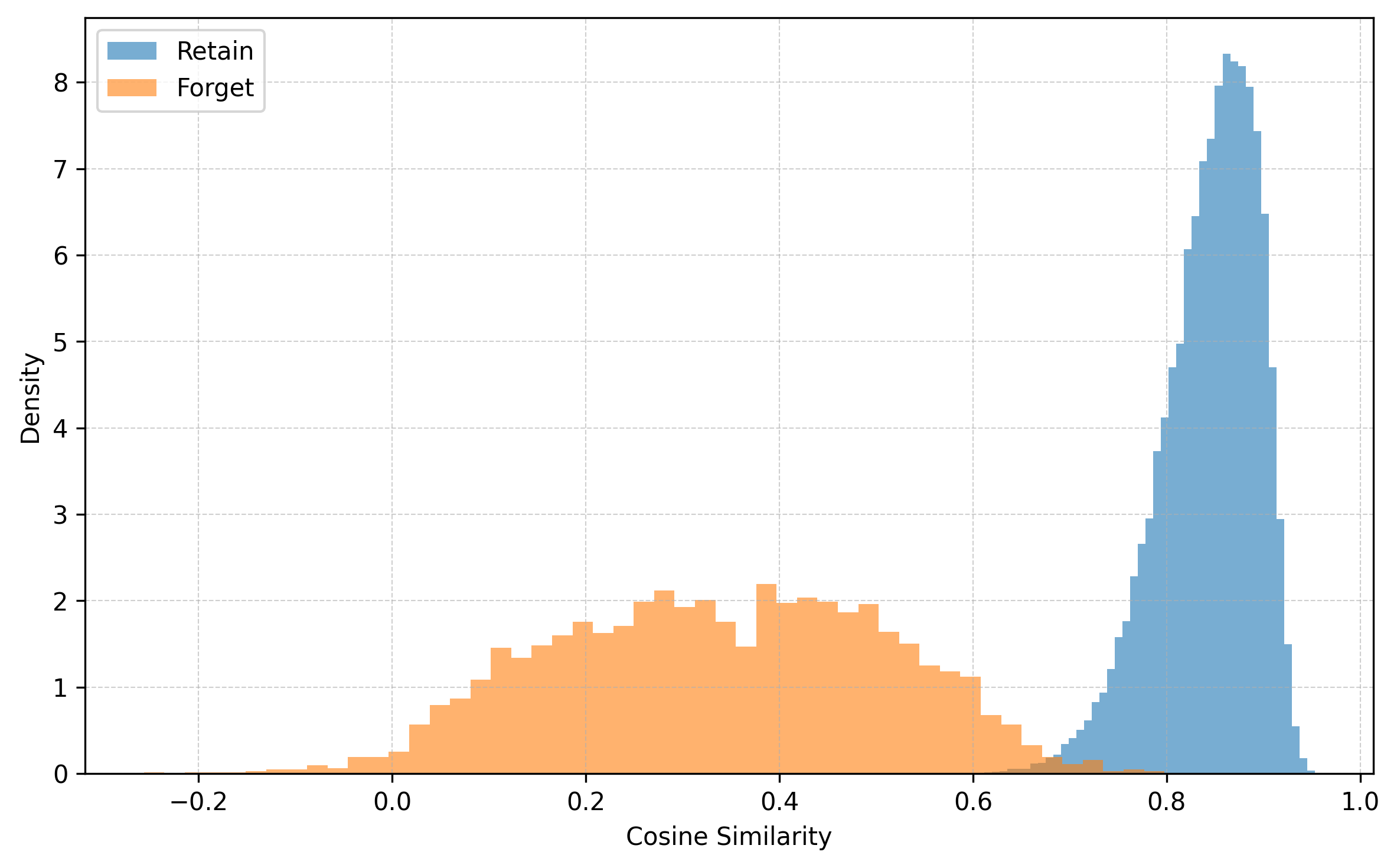} &
\includegraphics[width=0.22\textwidth]{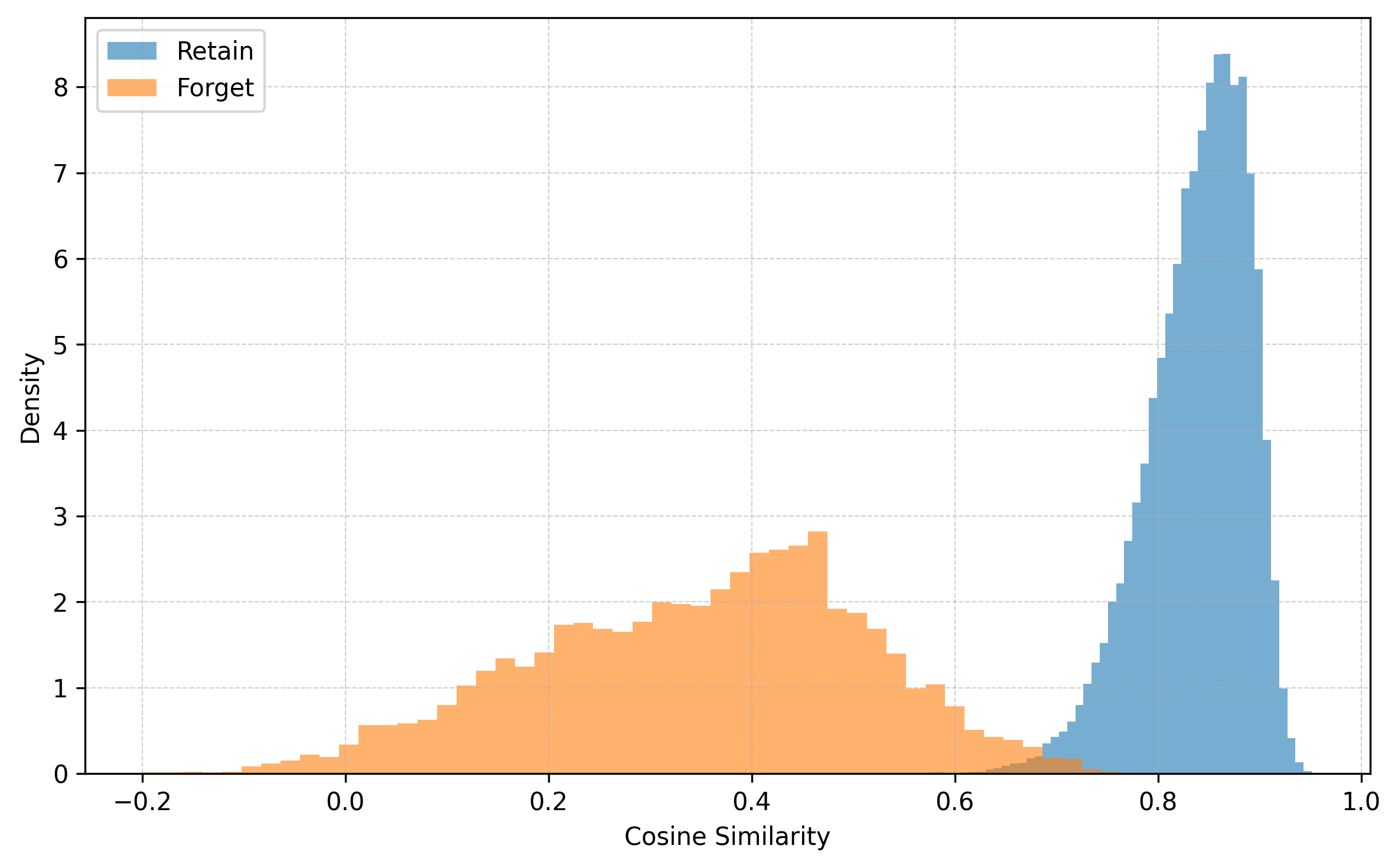} \\

\bottomrule

\end{tabular}
\caption{Representation similarity over the three-phase unlearning process on CIFAR-100. The orange shows the retain data distribution, while the blue shows the forget data distribution.}
\label{fig:c100_similarity}
\end{figure*}

%% file: figs/fig_similarity_vgg2.tex
\begin{figure*}[t]
\centering
\setlength{\tabcolsep}{1pt}
\begin{tabular}{c|ccc}
\toprule
Method & {Phase 1} & {Phase 2} & {Phase 3}\\
\midrule

\textbf{Retrain} &
\includegraphics[width=0.22\textwidth]{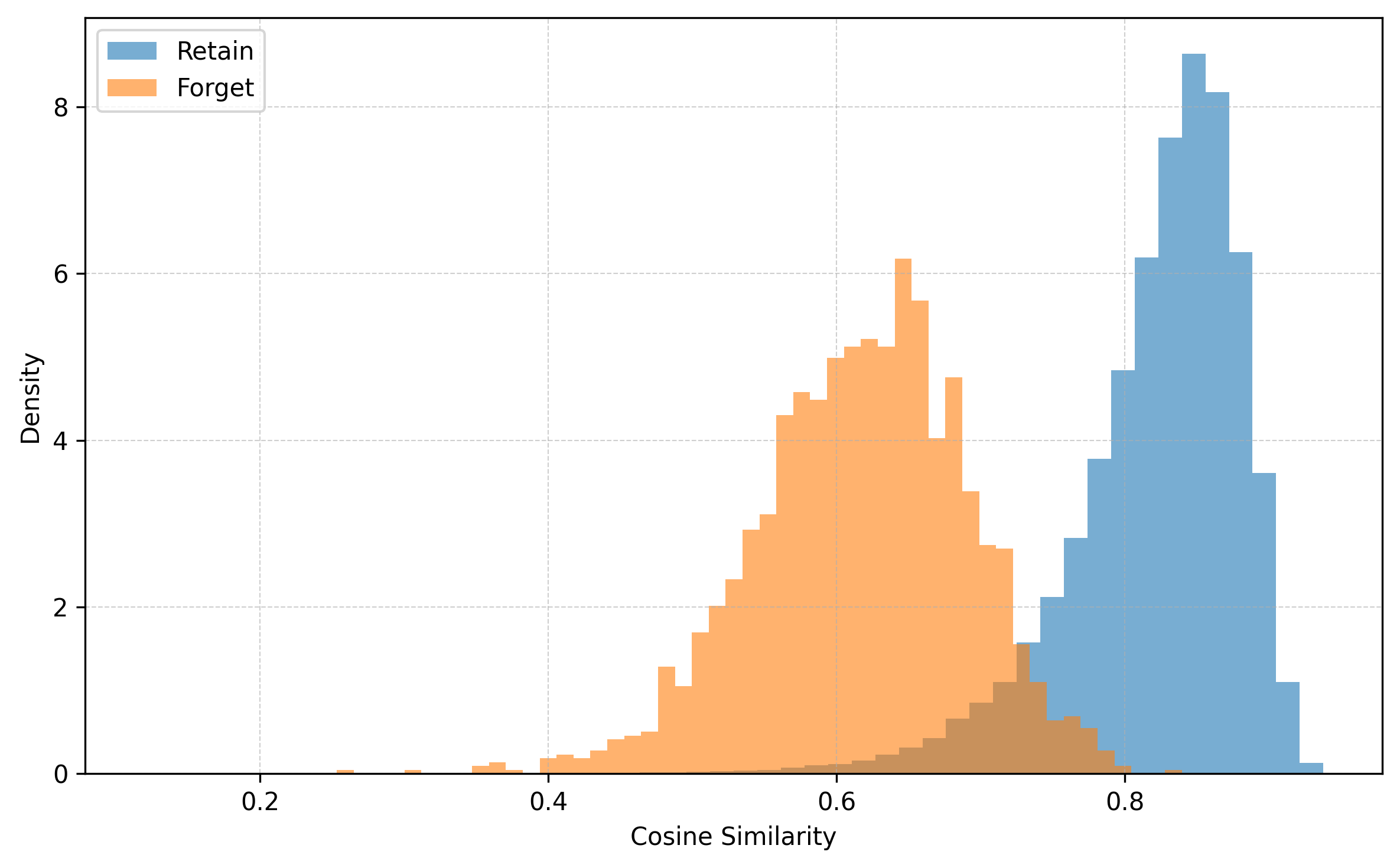} &
\includegraphics[width=0.22\textwidth]{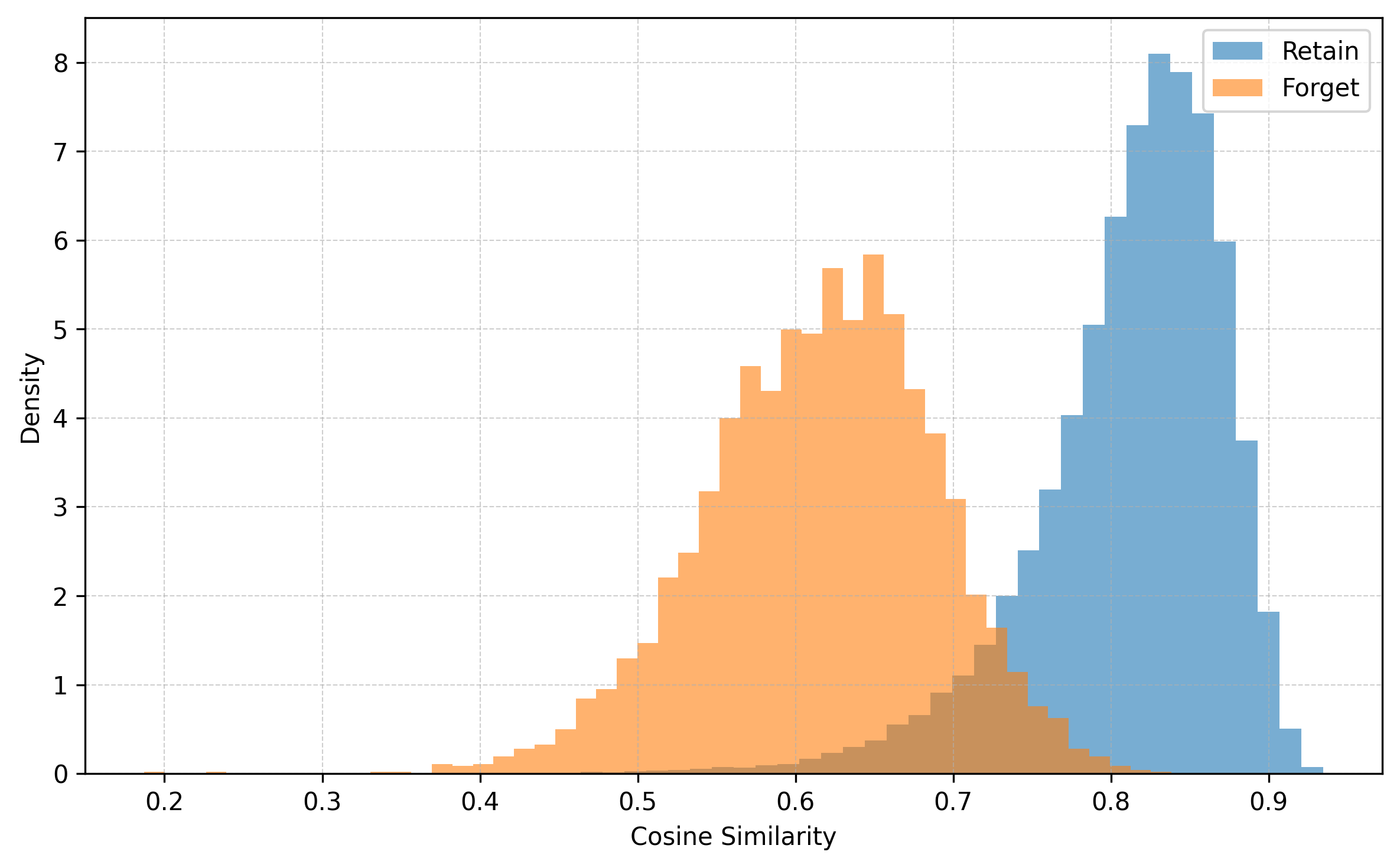} &
\includegraphics[width=0.22\textwidth]{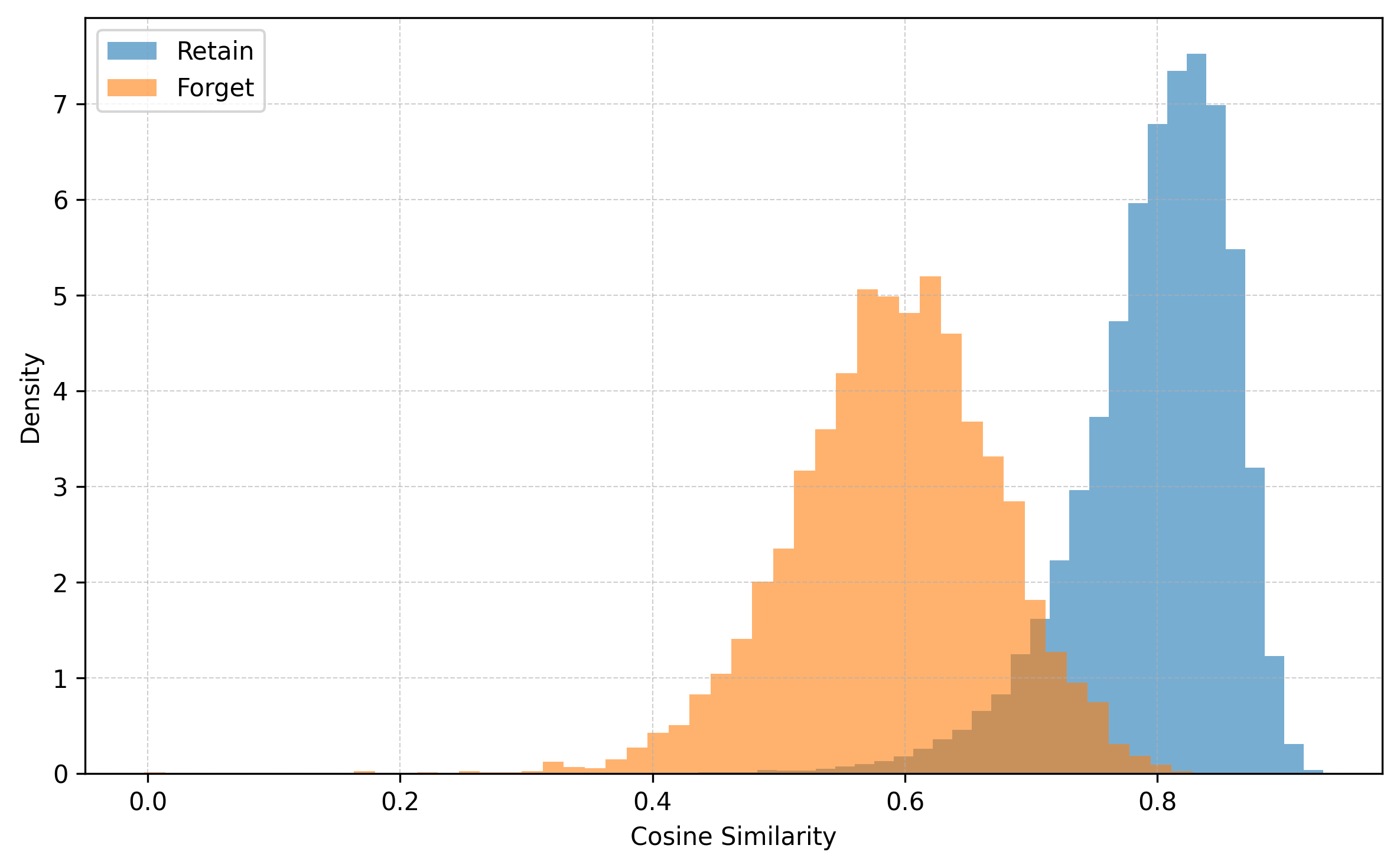} \\
\midrule

\textbf{SCRUB} &
\includegraphics[width=0.22\textwidth]{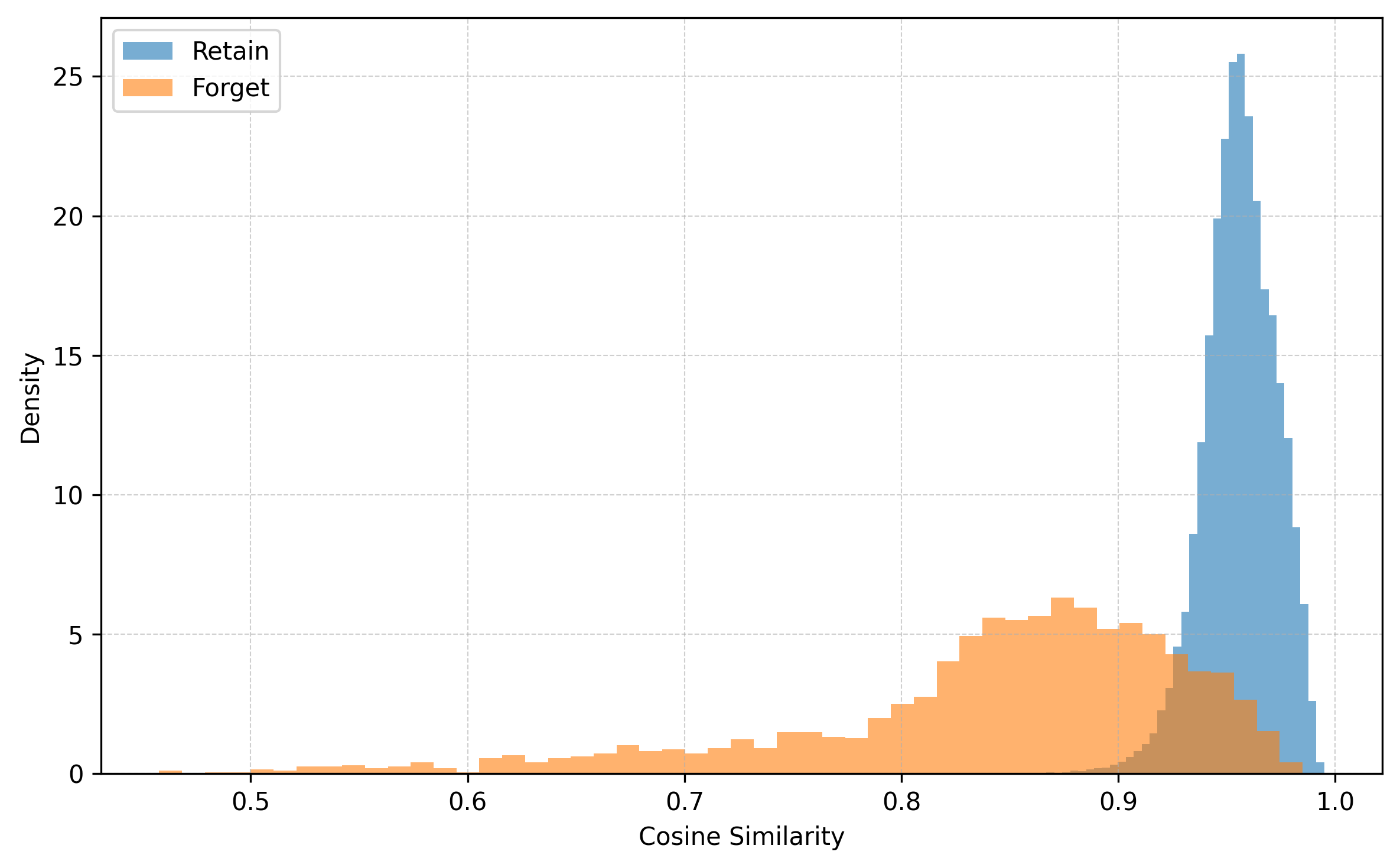} &
\includegraphics[width=0.22\textwidth]{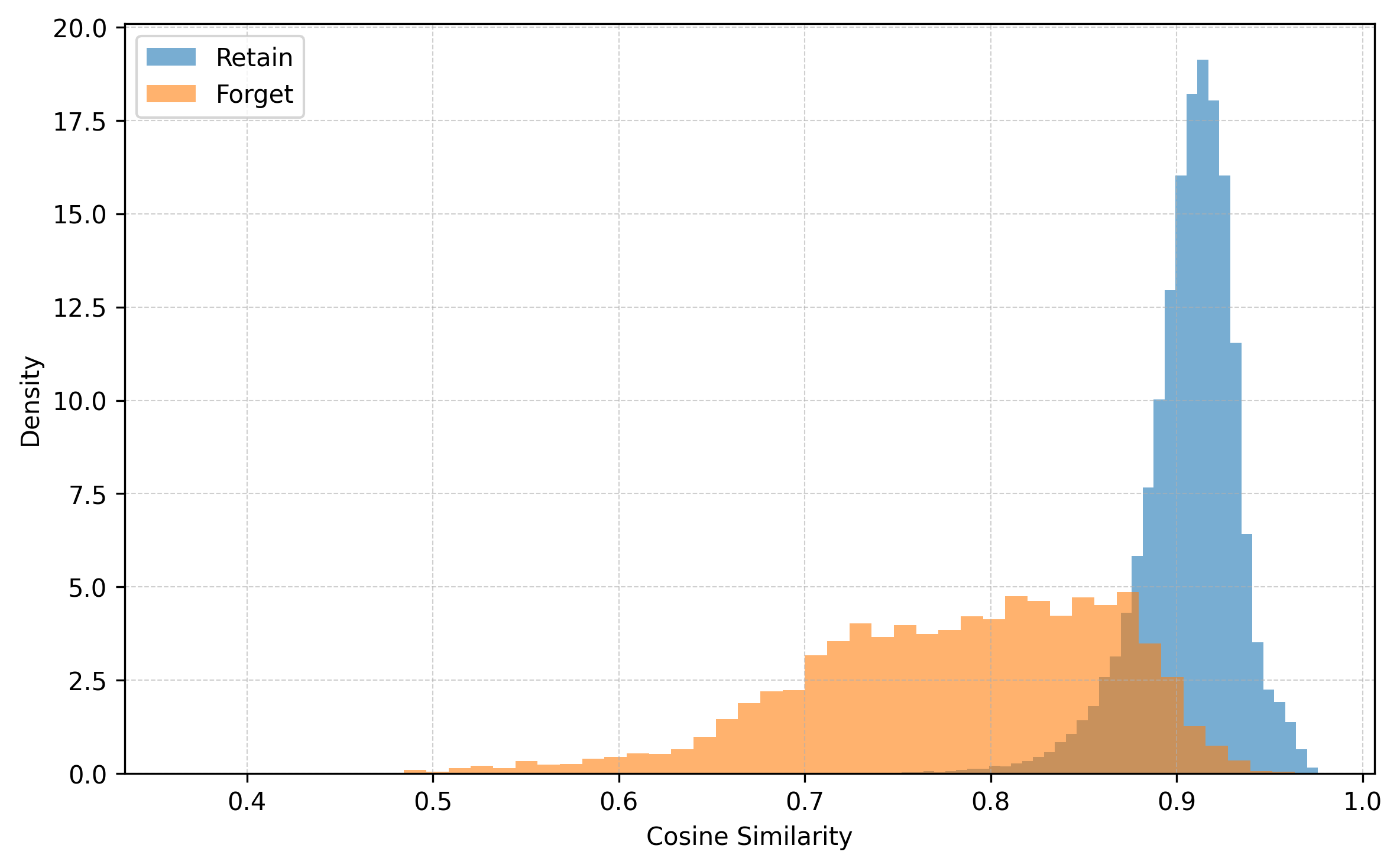} &
\includegraphics[width=0.22\textwidth]{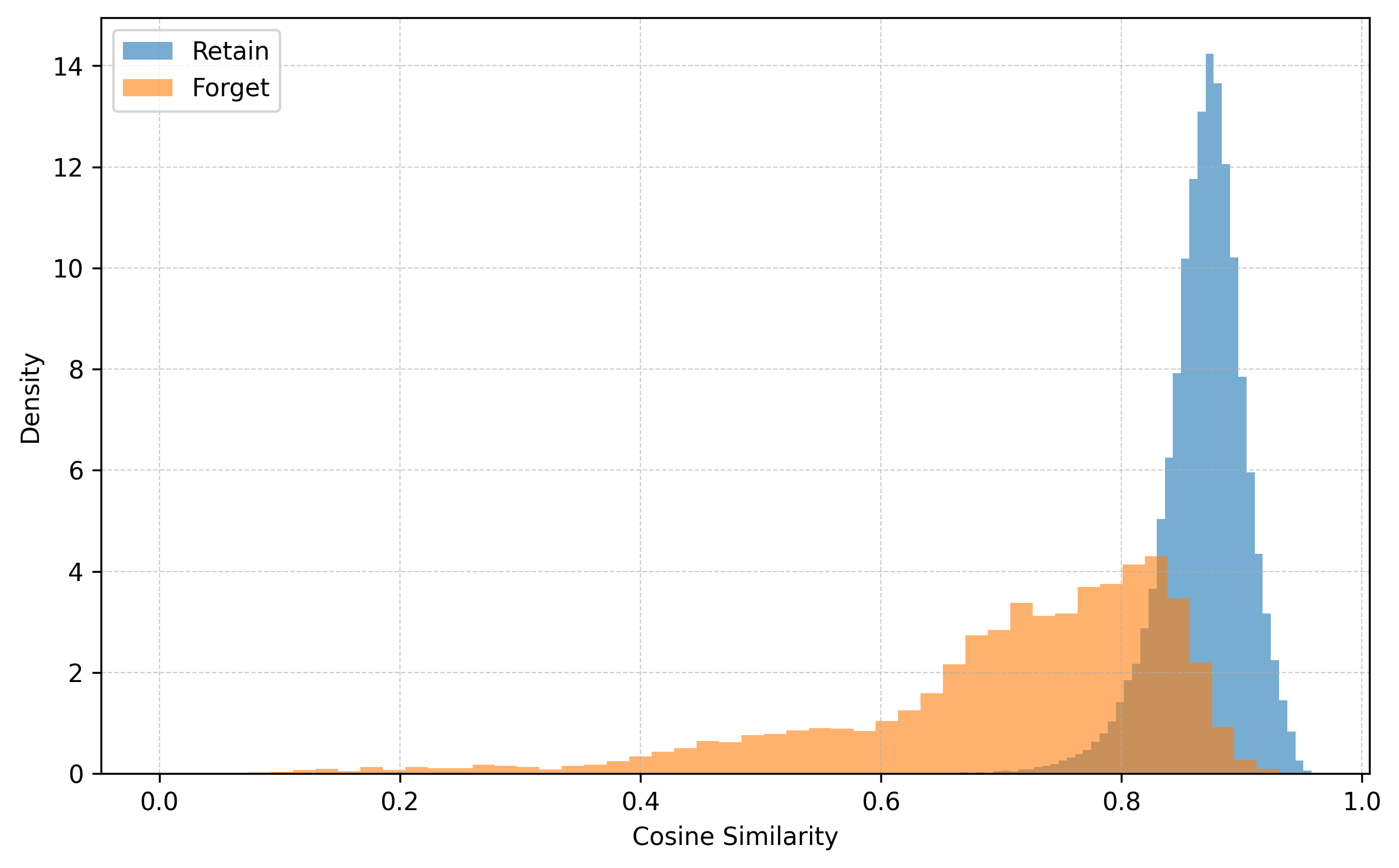} \\
\midrule

\textbf{SALUN} &
\includegraphics[width=0.22\textwidth]{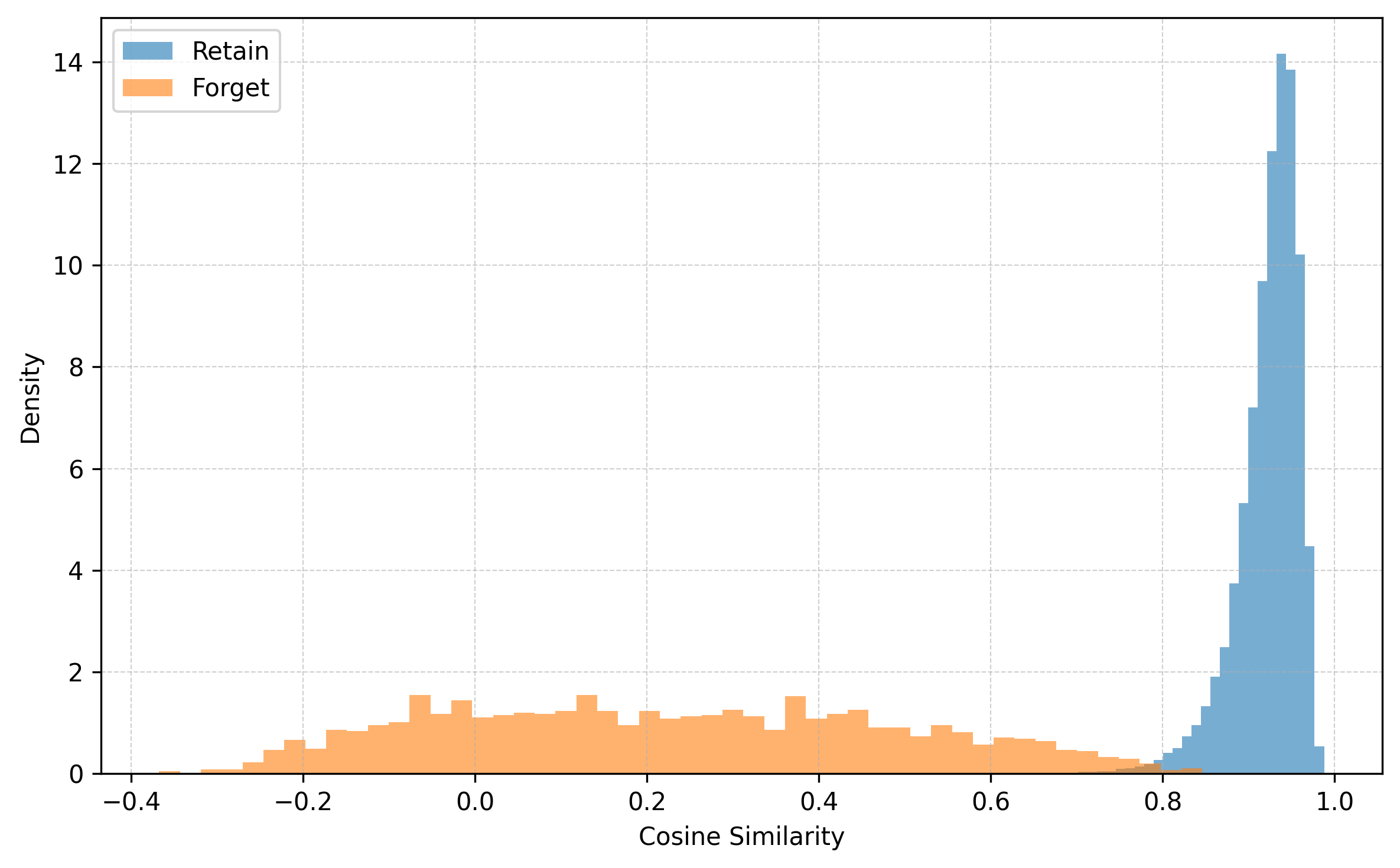} &
\includegraphics[width=0.22\textwidth]{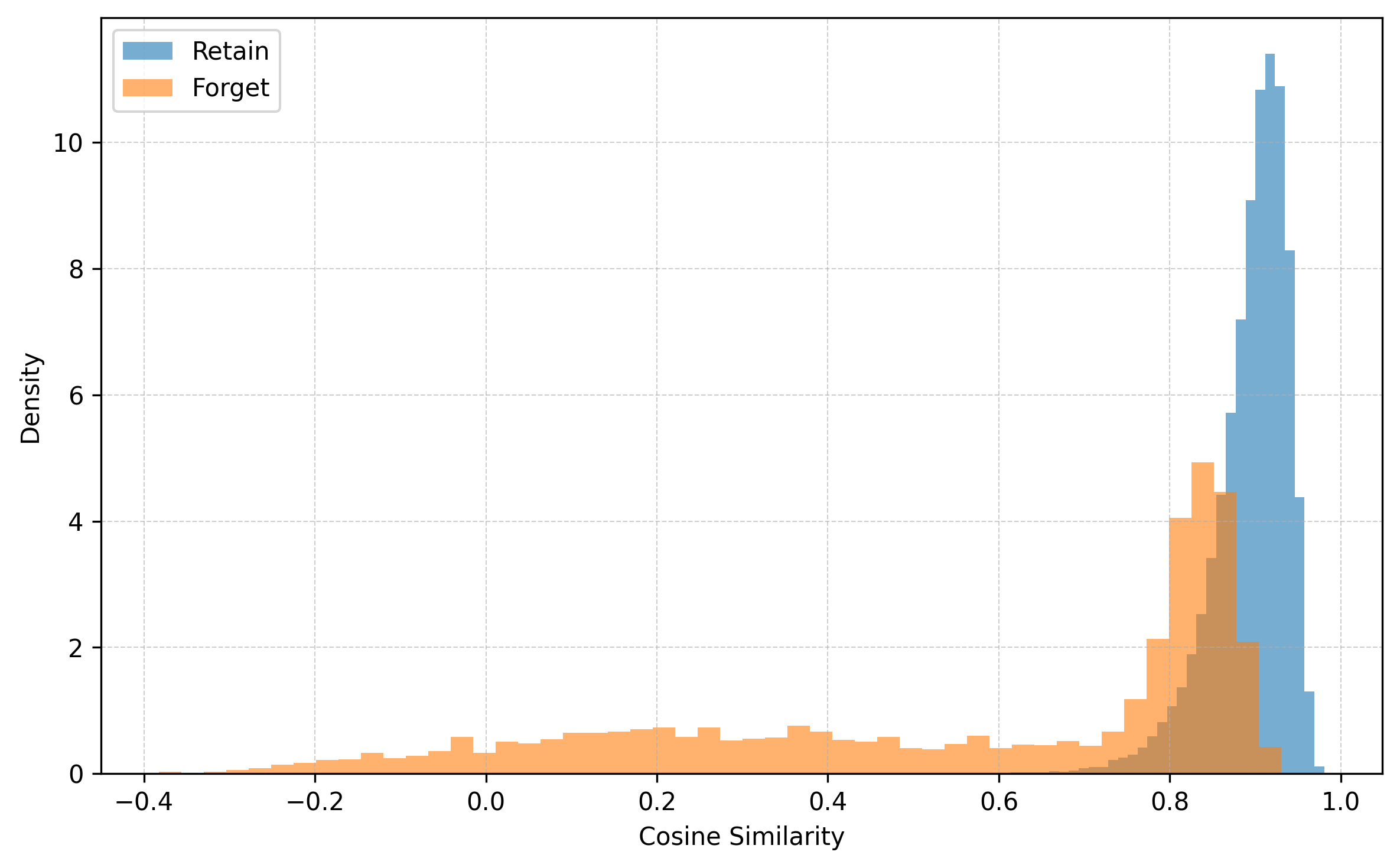} &
\includegraphics[width=0.22\textwidth]{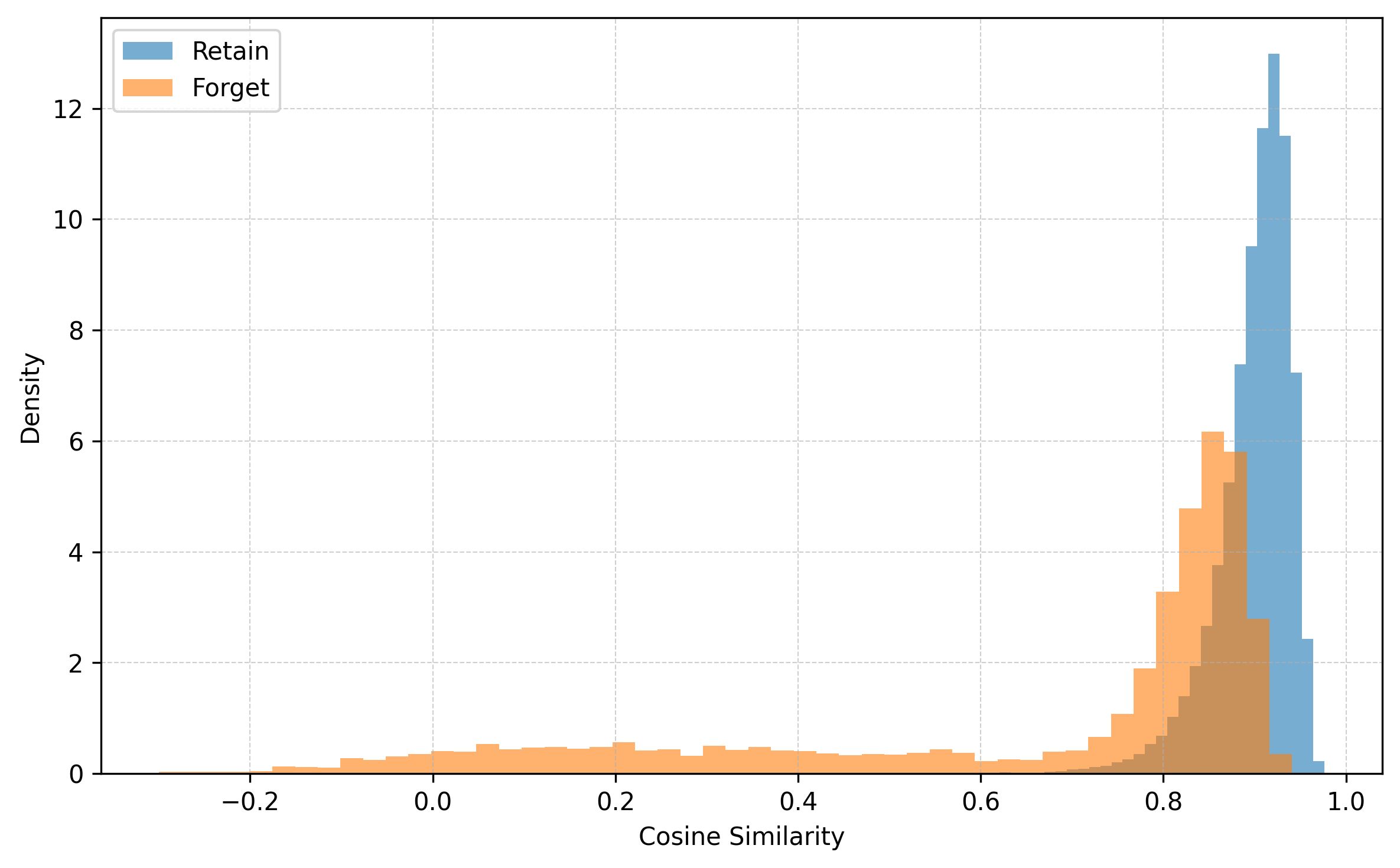} \\
\midrule

\textbf{SSD} &
\includegraphics[width=0.22\textwidth]{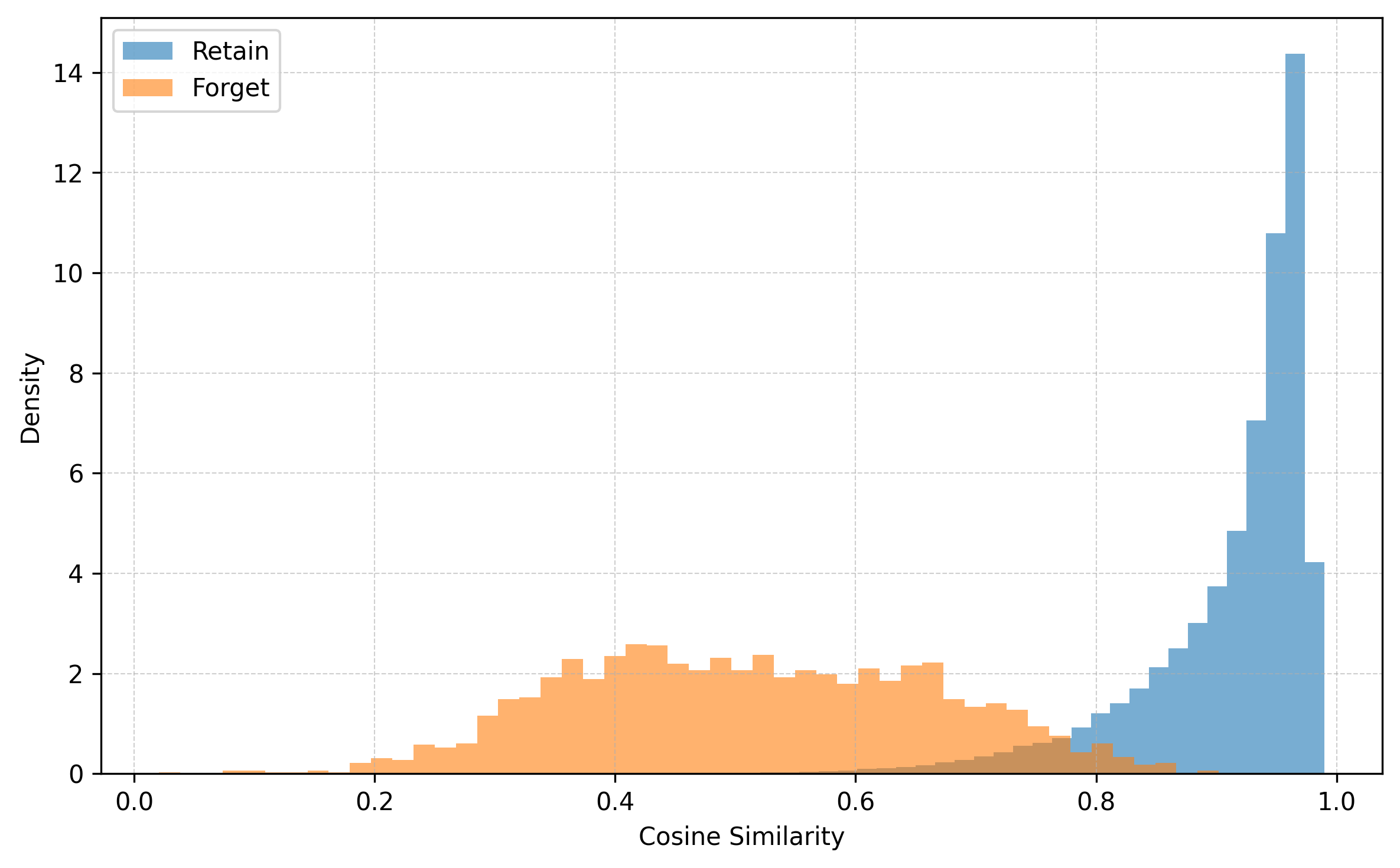} &
\includegraphics[width=0.22\textwidth]{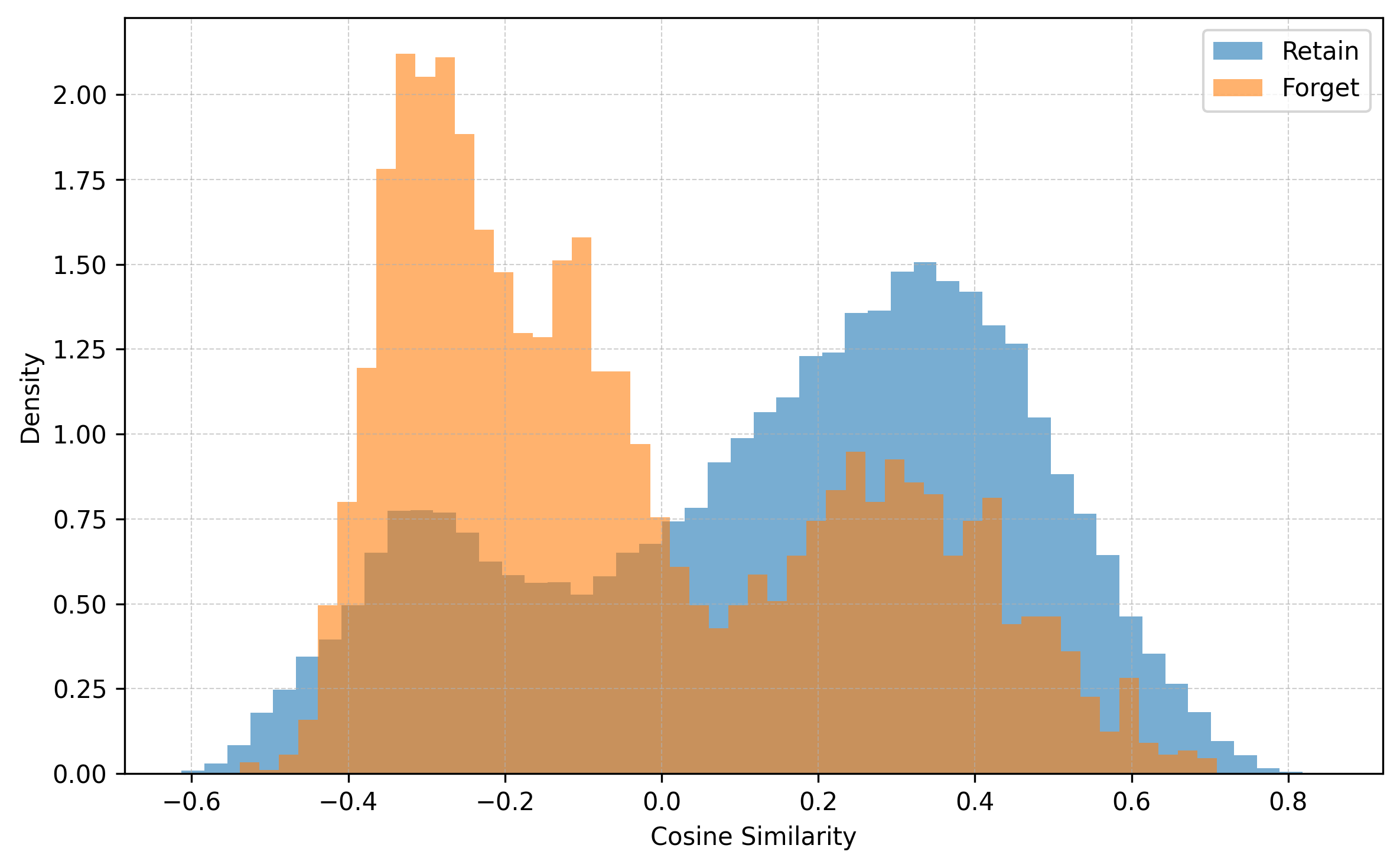} &
\includegraphics[width=0.22\textwidth]{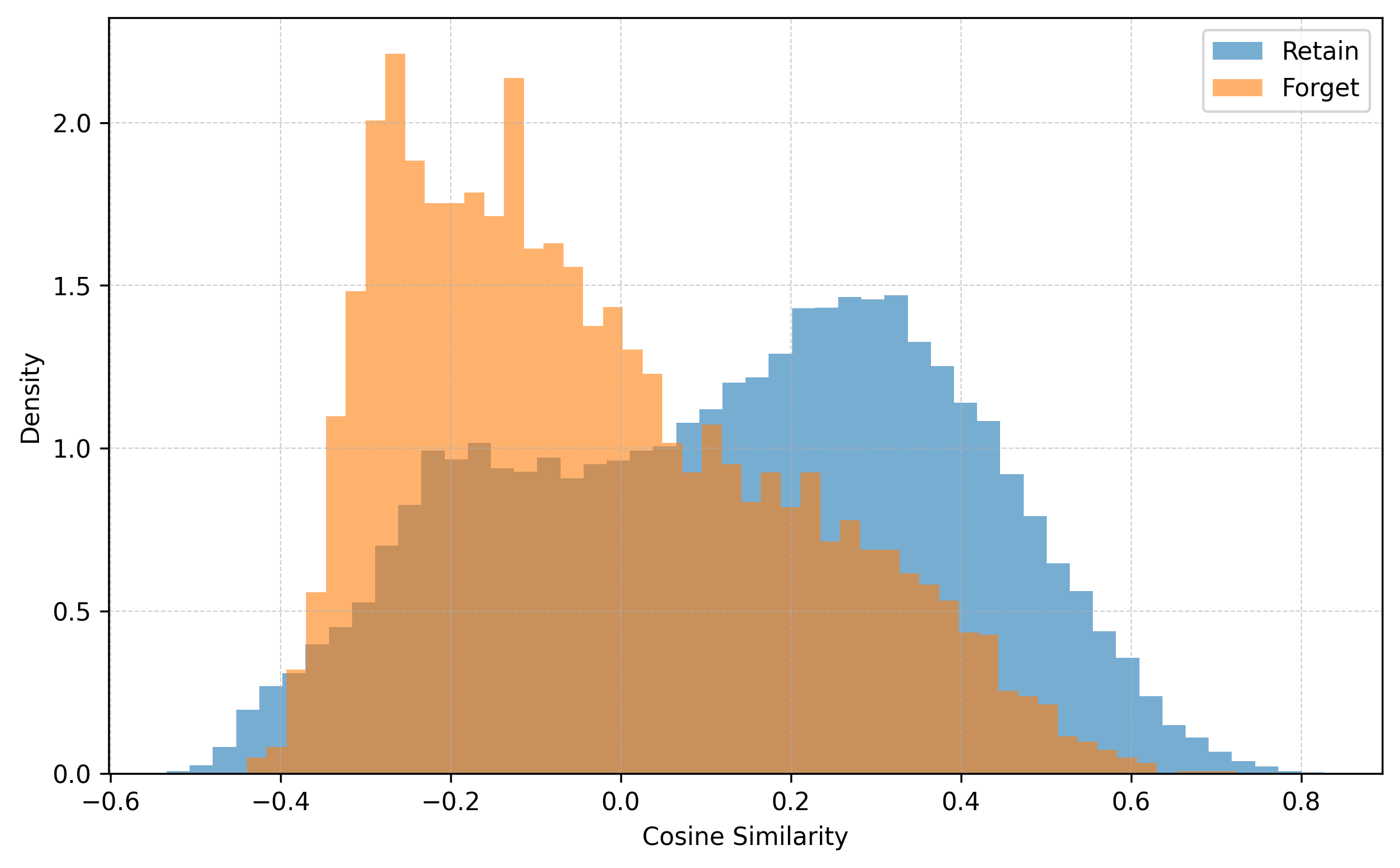} \\
\midrule

\textbf{BndShrink}  &
\includegraphics[width=0.22\textwidth]{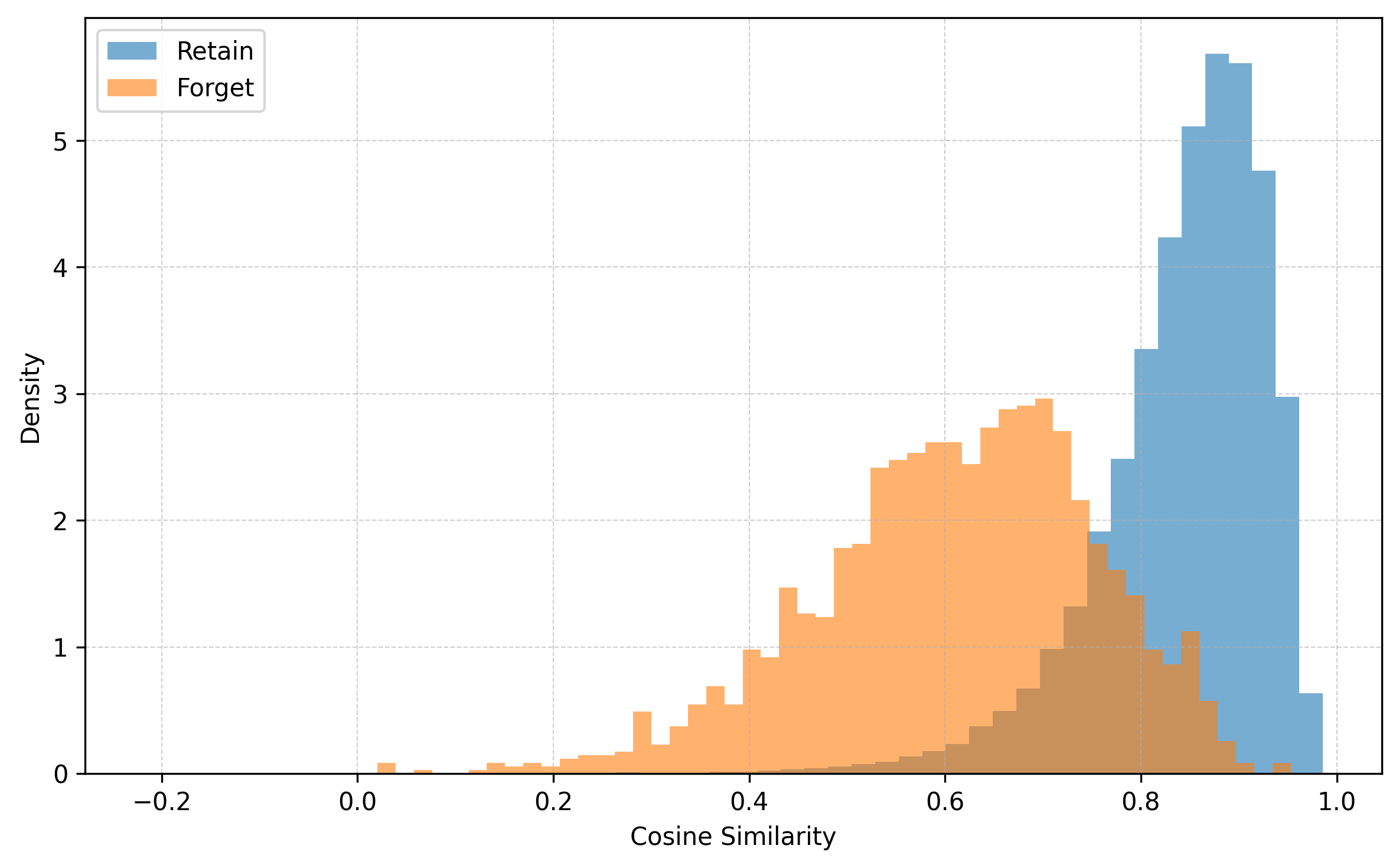} &
\includegraphics[width=0.22\textwidth]{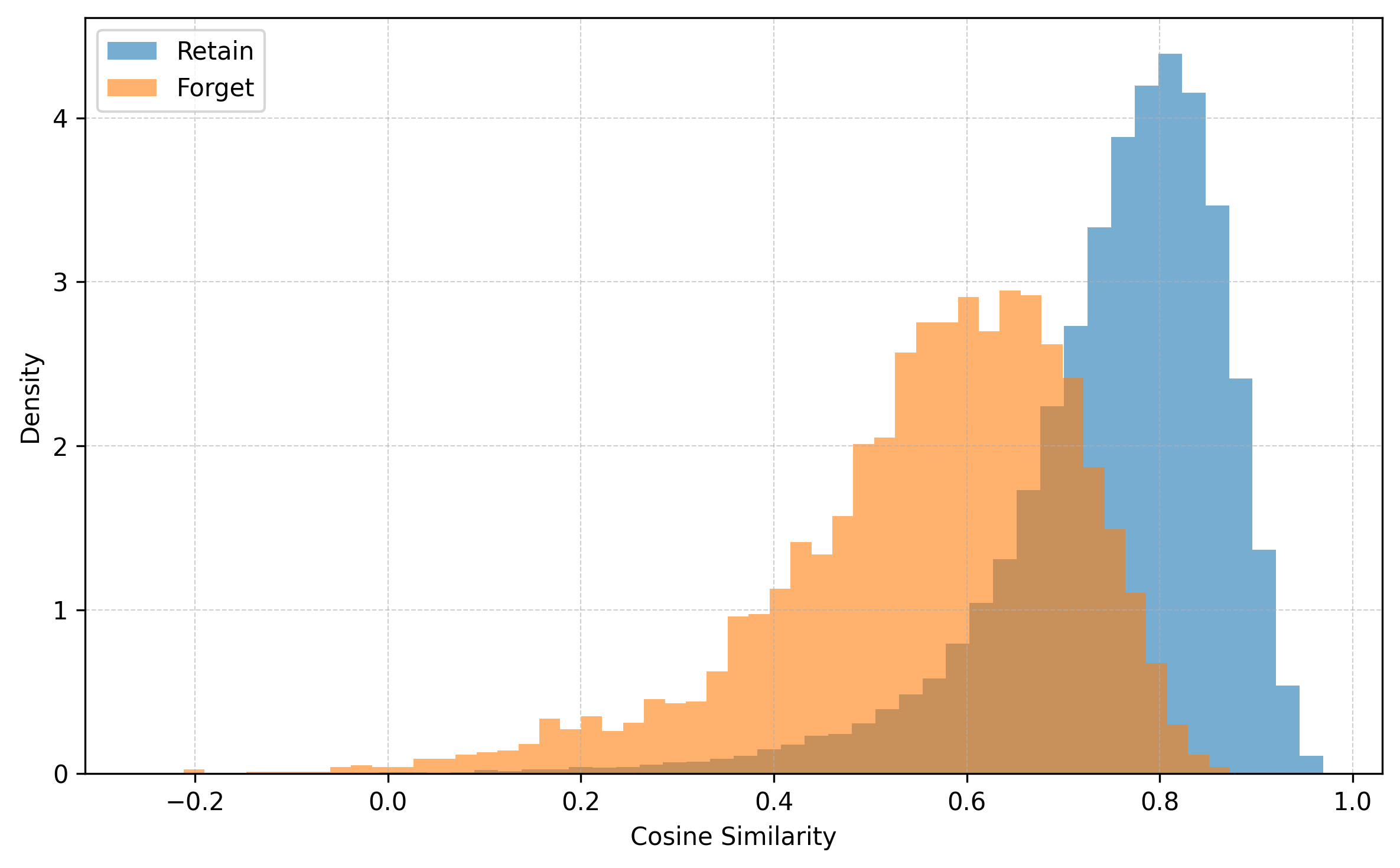} &
\includegraphics[width=0.22\textwidth]{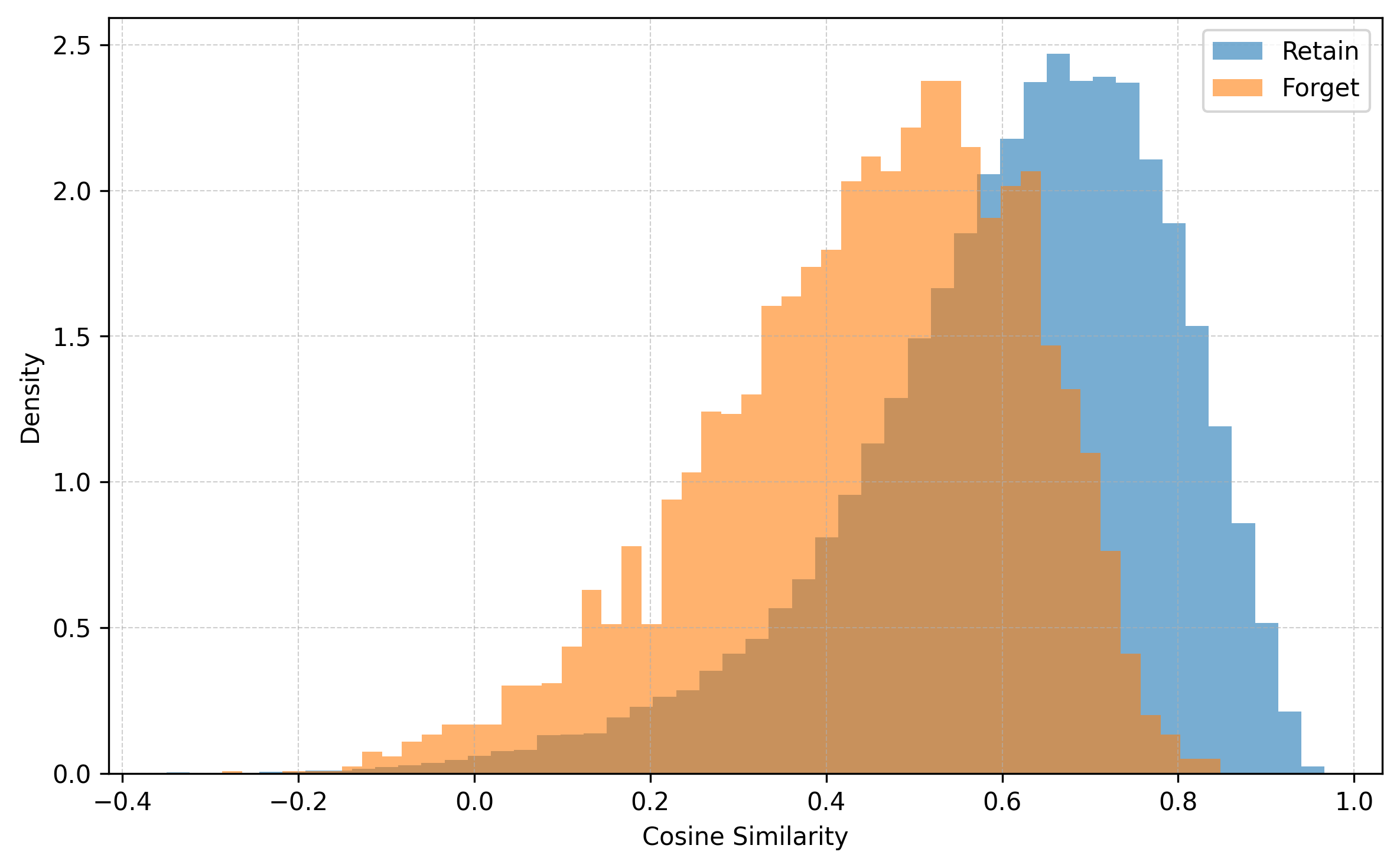} \\
\midrule

\textbf{NegGrad} &
\includegraphics[width=0.22\textwidth]{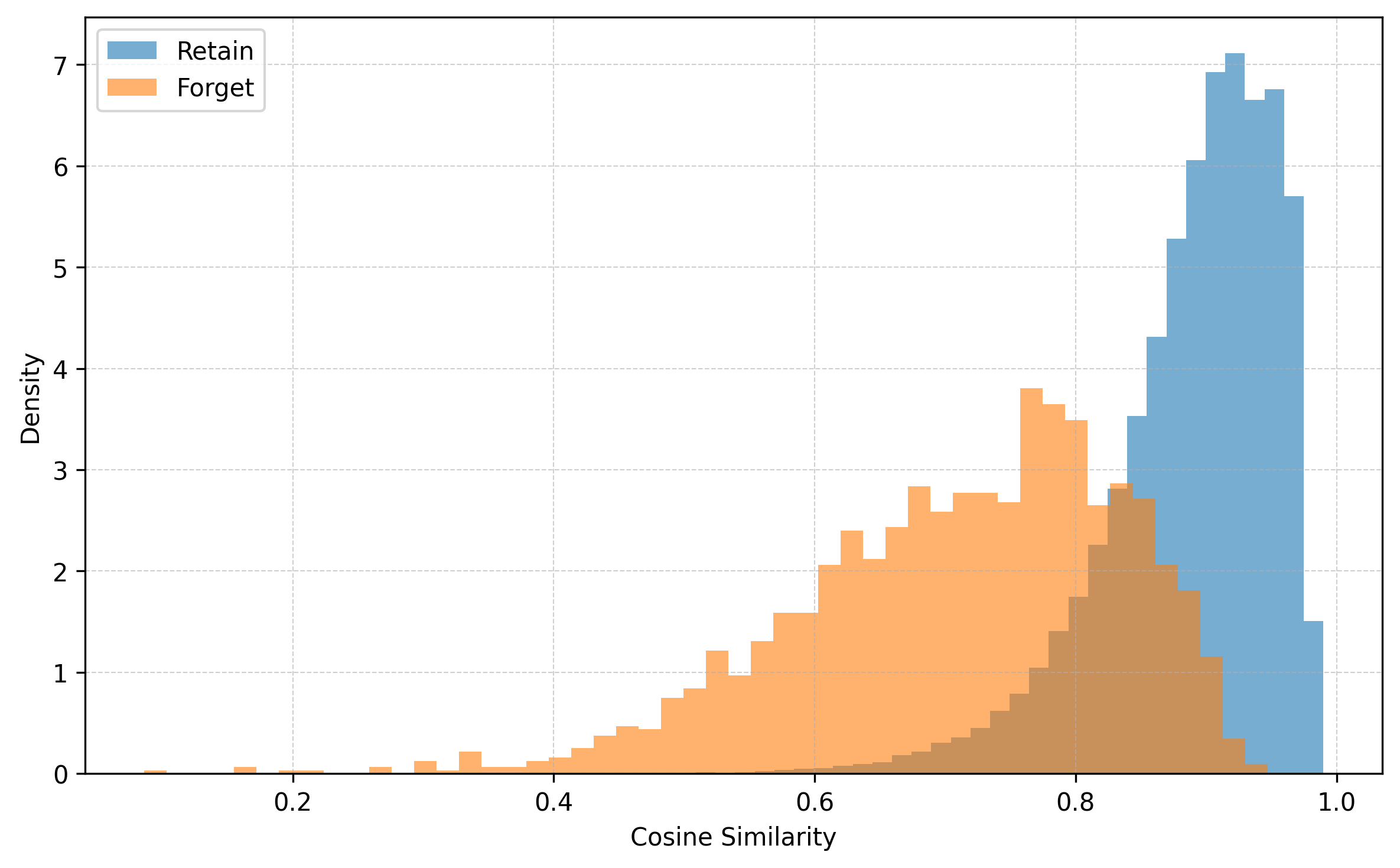} &
\includegraphics[width=0.22\textwidth]{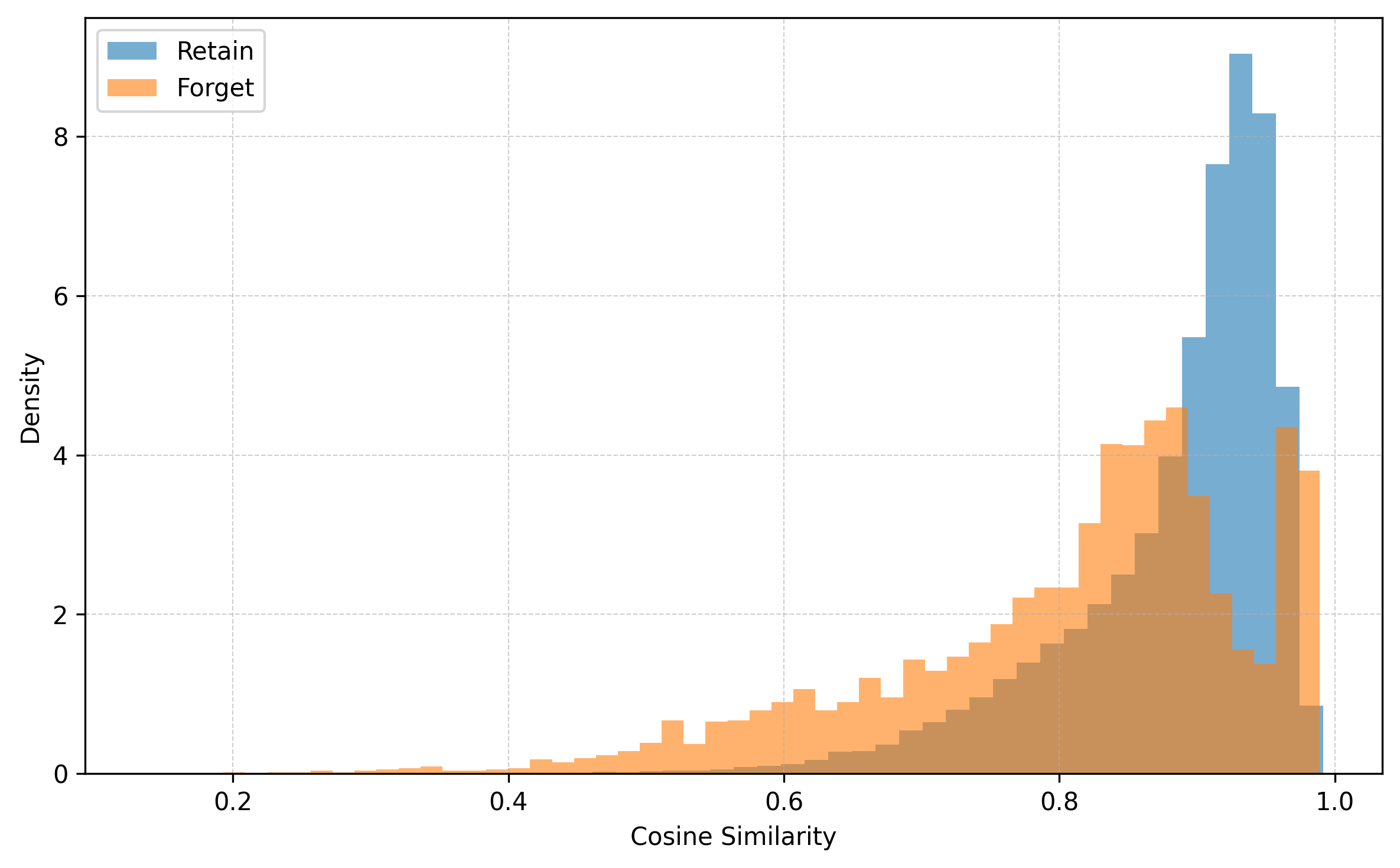} &
\includegraphics[width=0.22\textwidth]{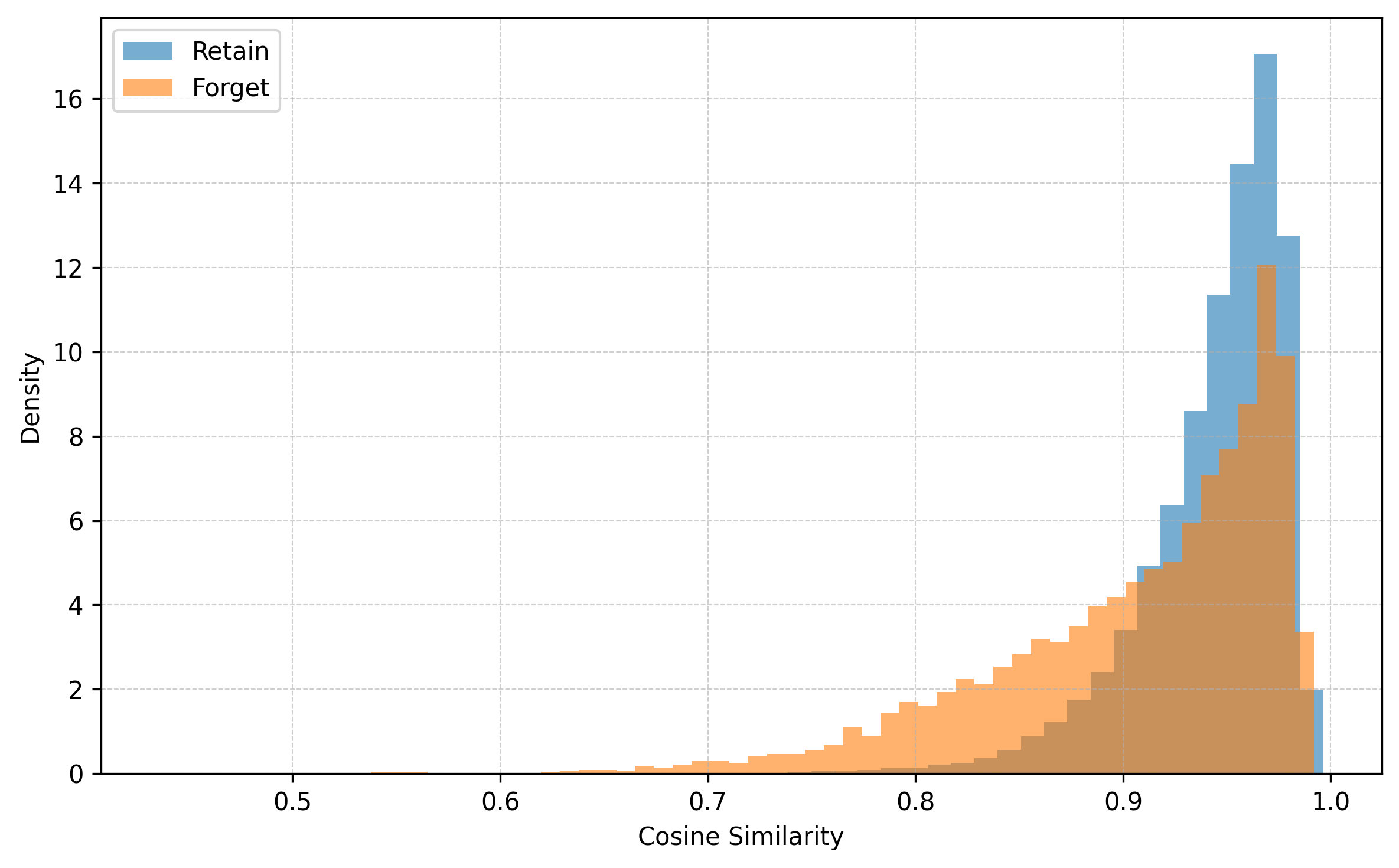} \\
\midrule

\textbf{Finetune} &
\includegraphics[width=0.22\textwidth]{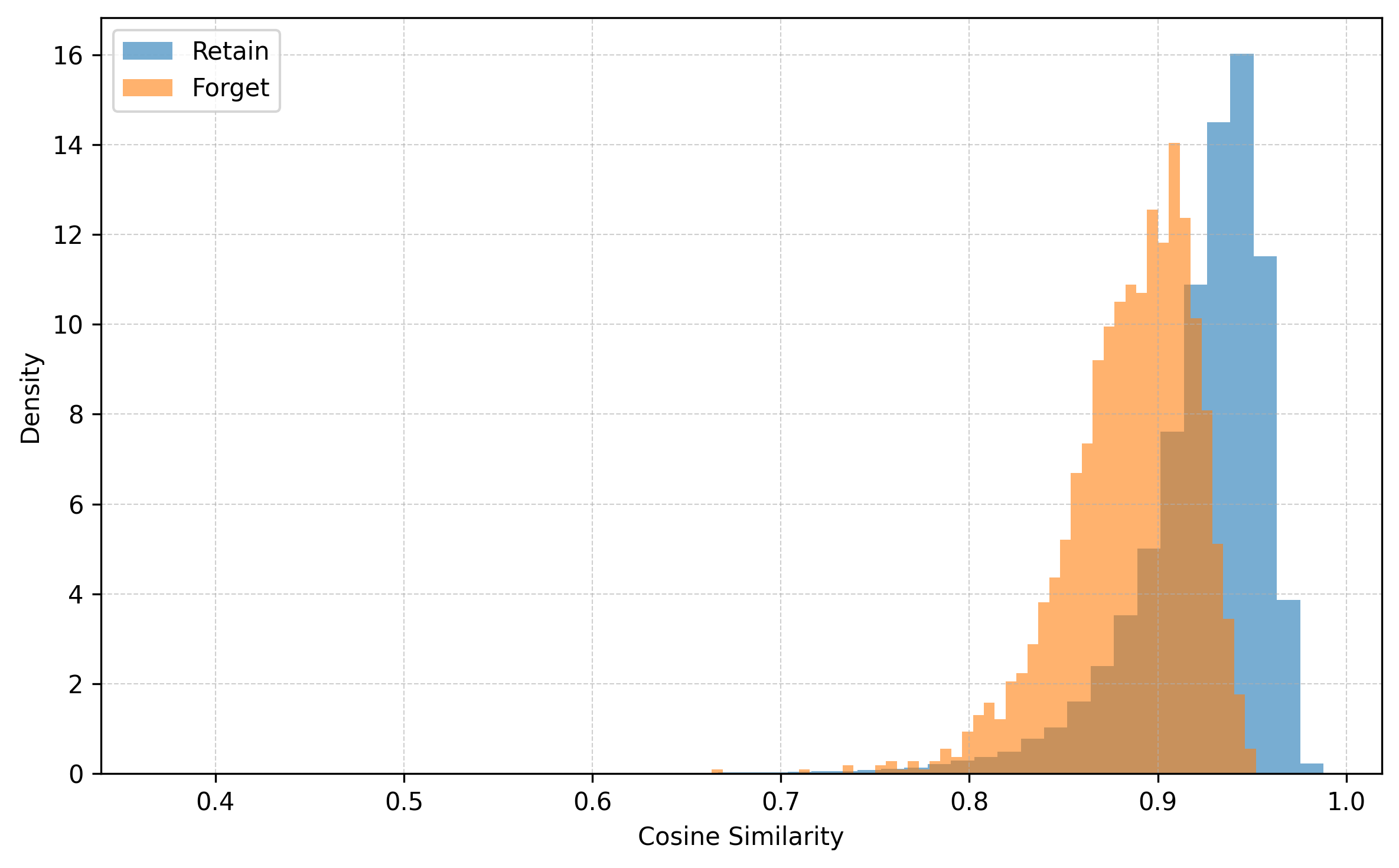} &
\includegraphics[width=0.22\textwidth]{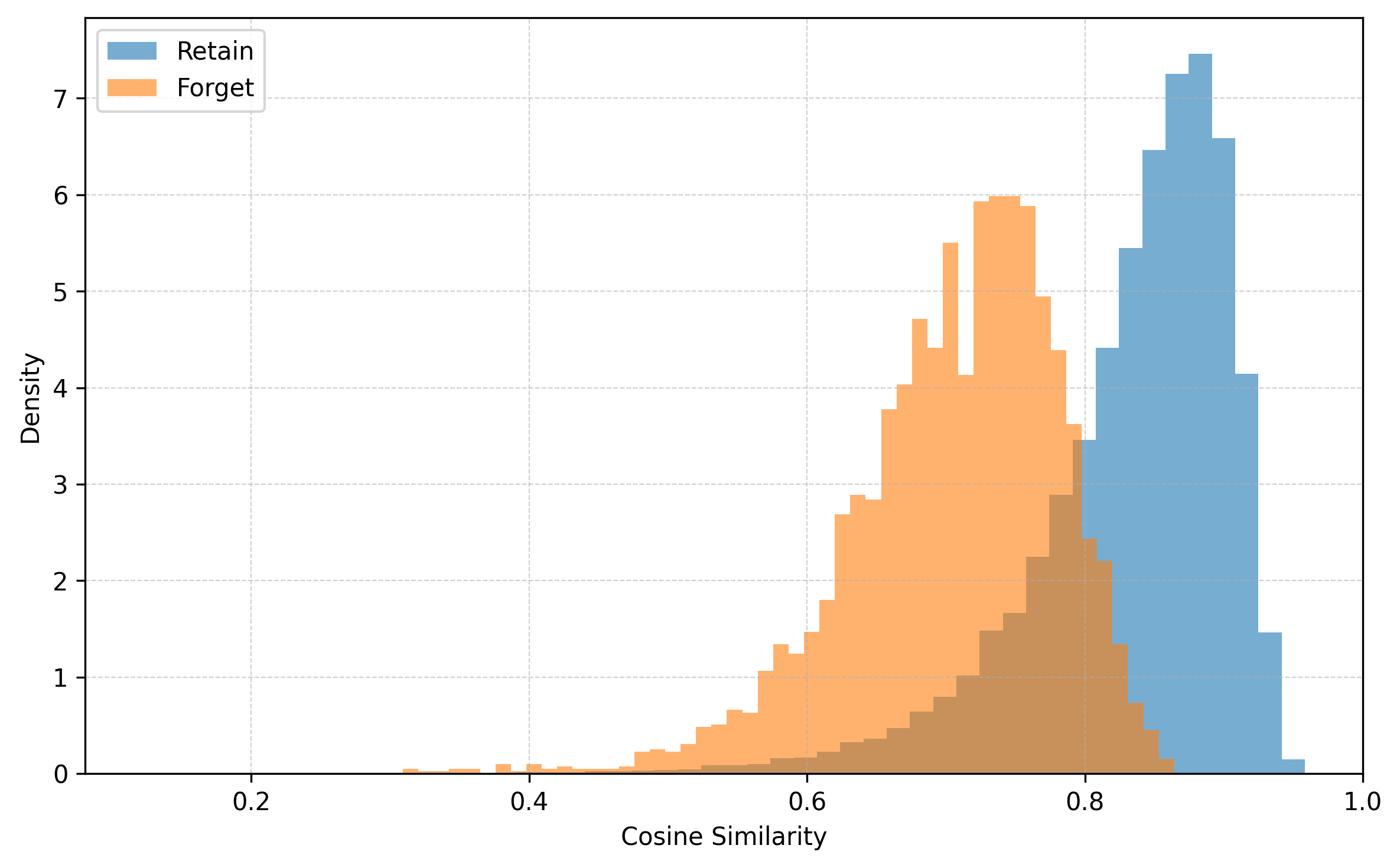} &
\includegraphics[width=0.22\textwidth]{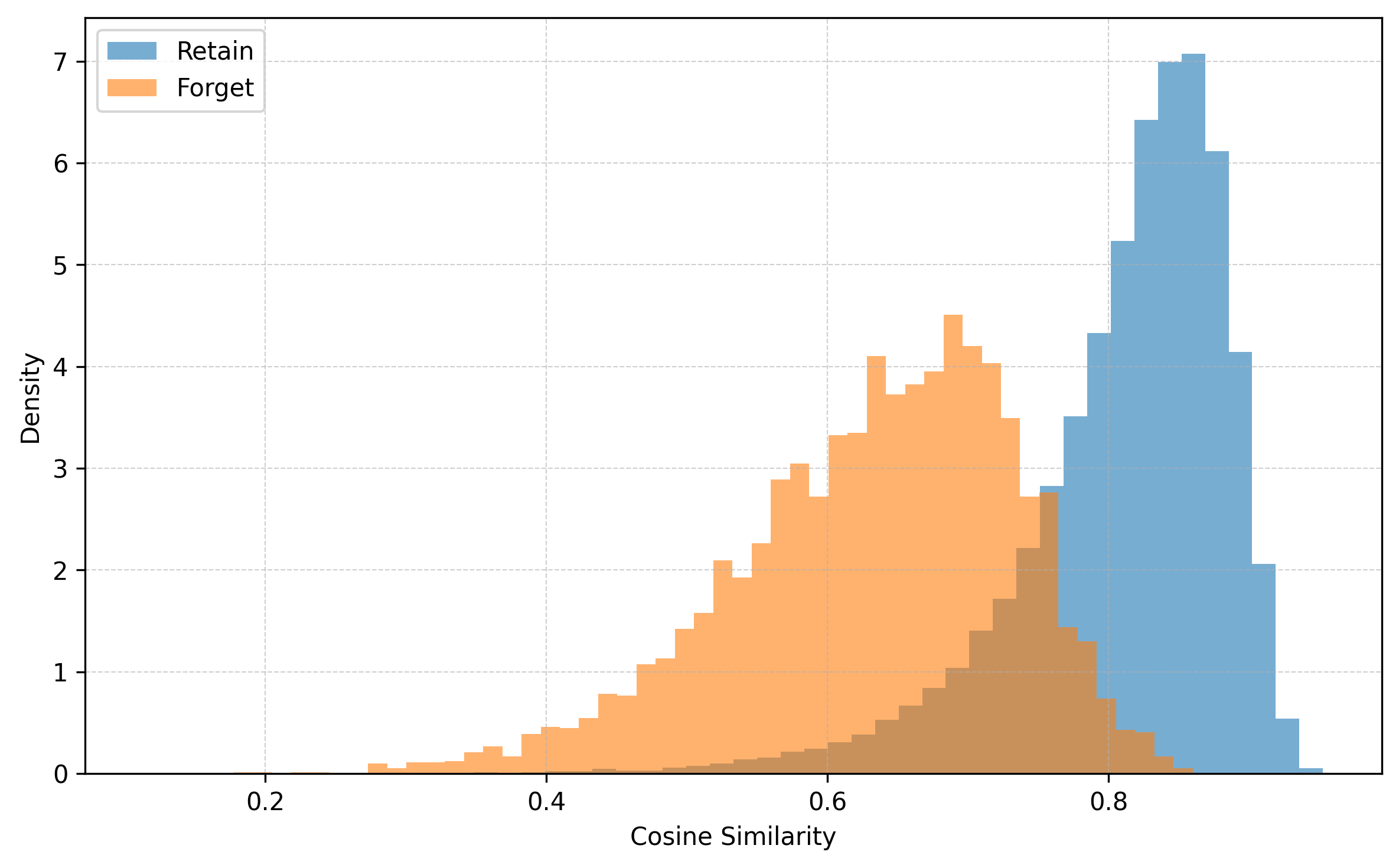} \\
\midrule

\textbf{SAFER} &
\includegraphics[width=0.22\textwidth]{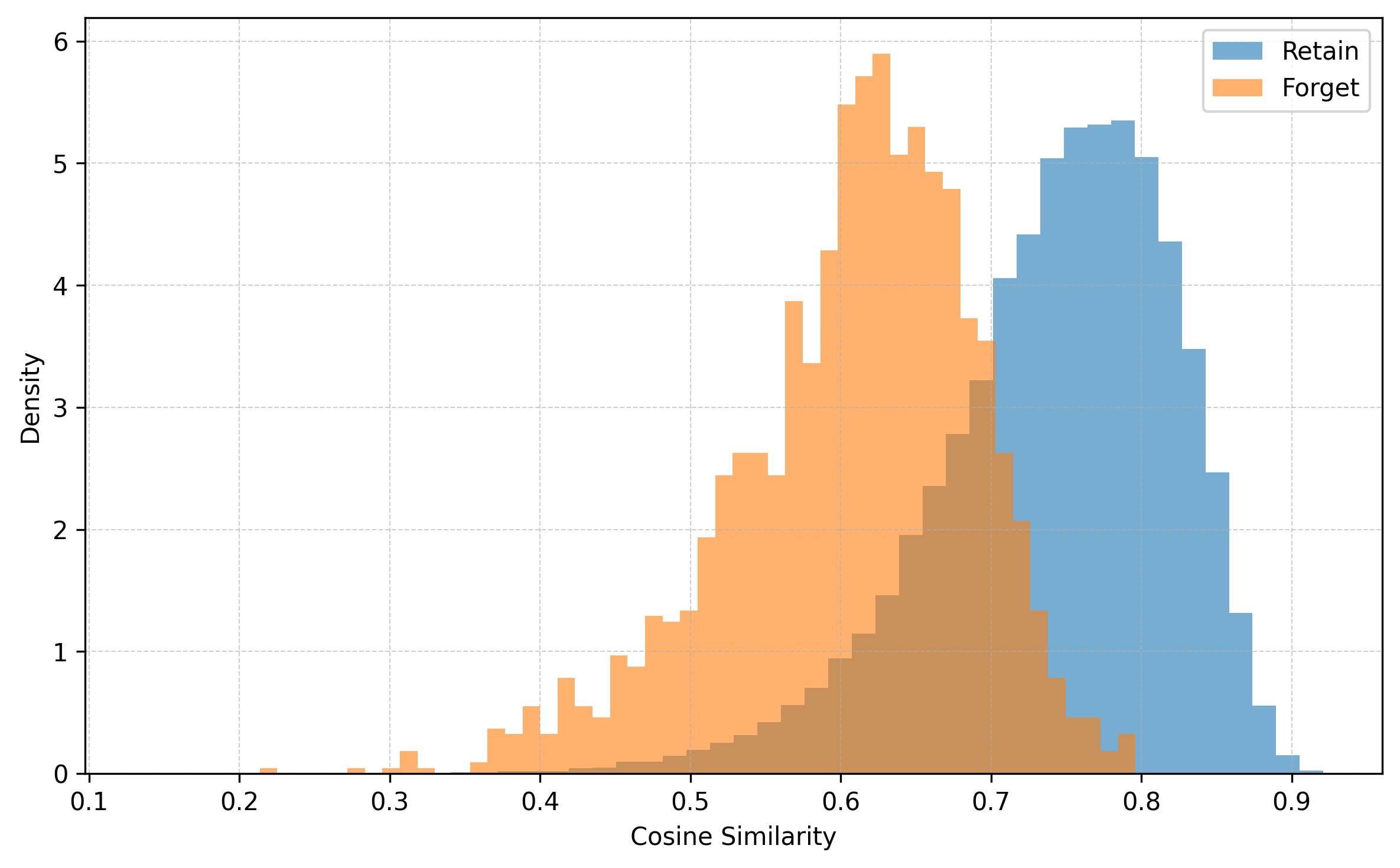} &
\includegraphics[width=0.22\textwidth]{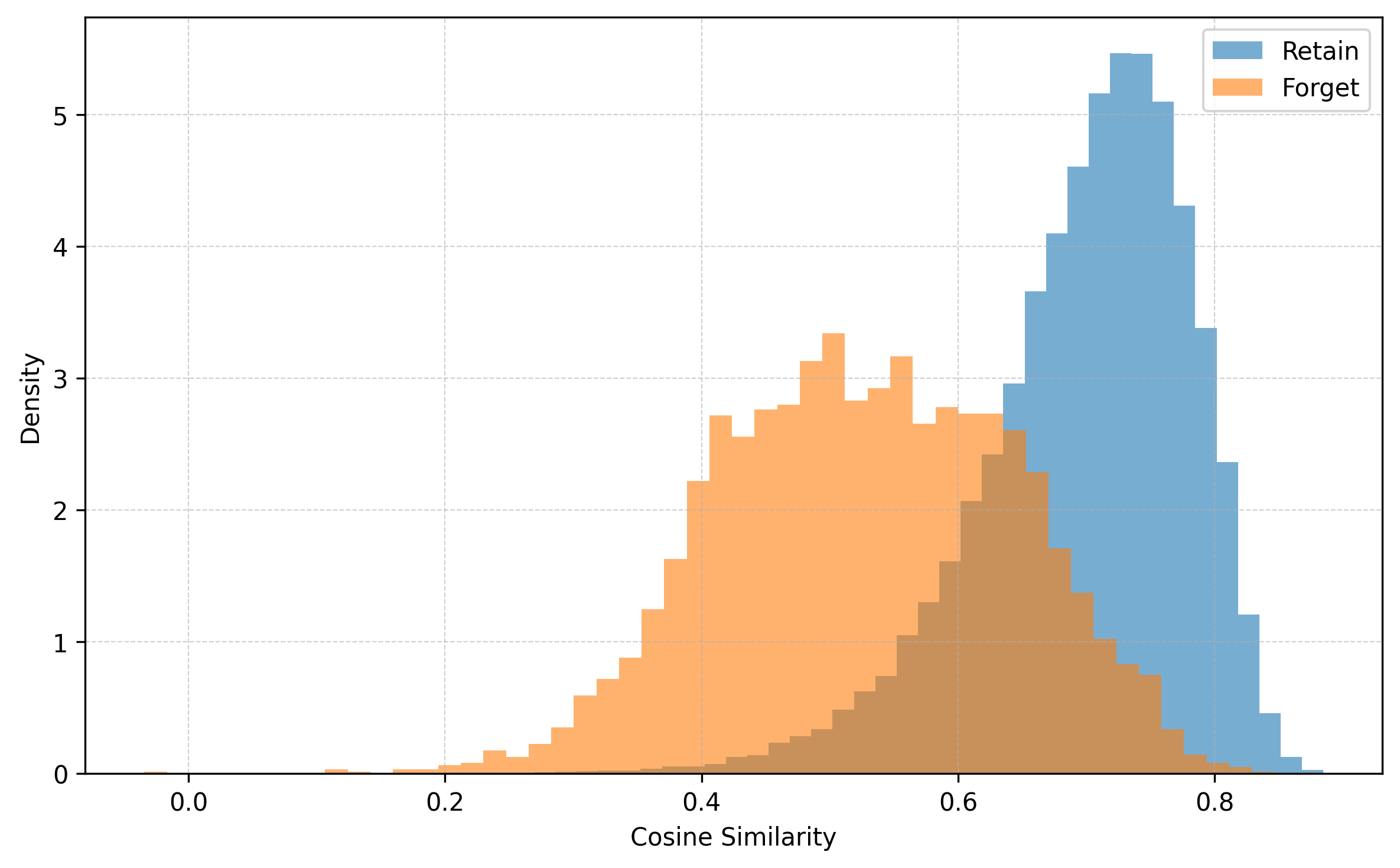} &
\includegraphics[width=0.22\textwidth]{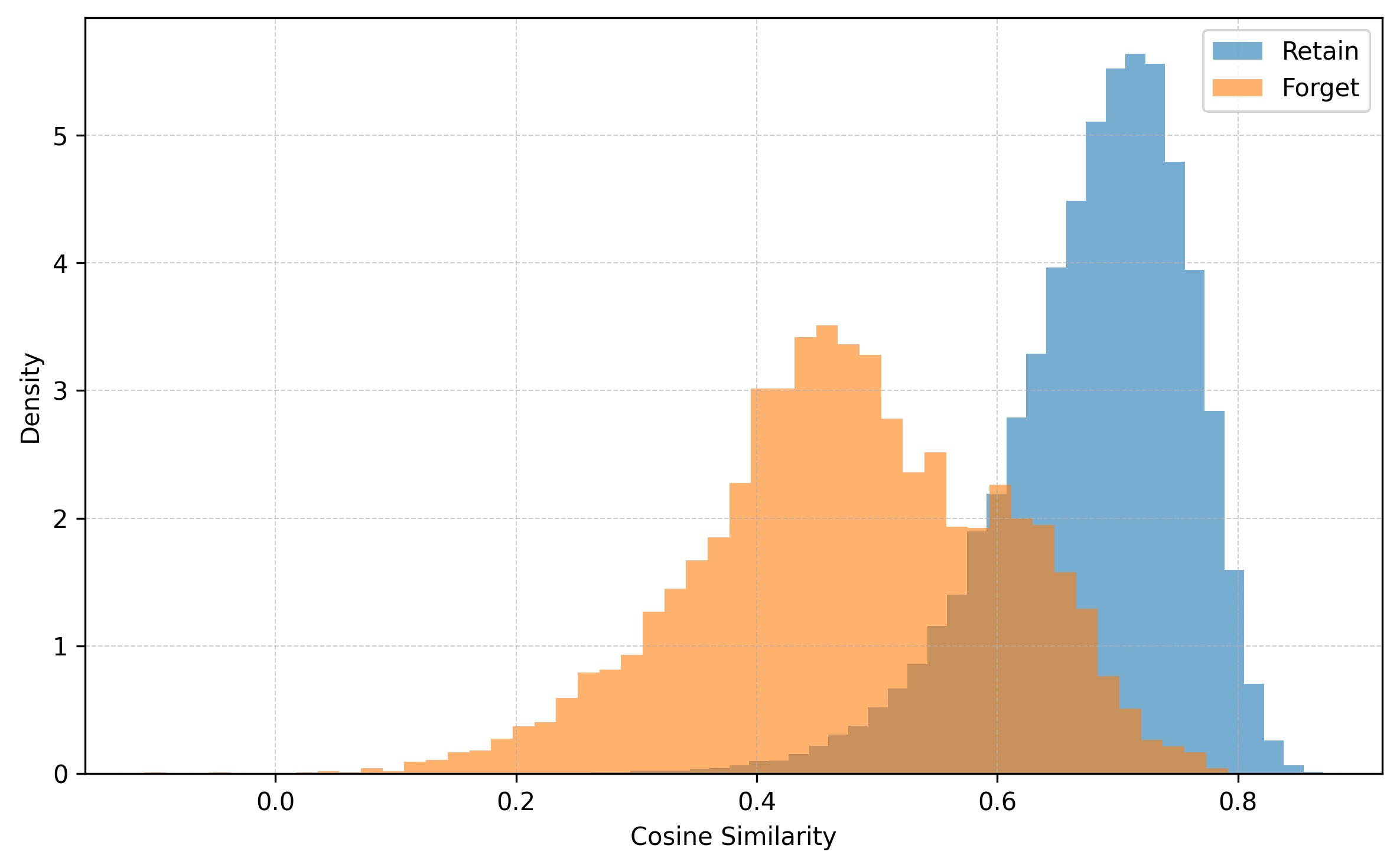} \\

\bottomrule

\end{tabular}
\caption{Representation similarity over the three-phase unlearning process on VGGFace2. The orange shows the retain data distribution, while the blue shows the forget data distribution.}
\label{fig:vgg2_similarity}
\end{figure*}

%% file: figs/fig_similarity_mufac.tex
\begin{figure*}[t]
\centering
\setlength{\tabcolsep}{1pt}
\begin{tabular}{c|ccc}
\toprule
Method & {Phase 1} & {Phase 2} & {Phase 3}\\
\midrule

\textbf{Retrain} &
\includegraphics[width=0.22\textwidth]{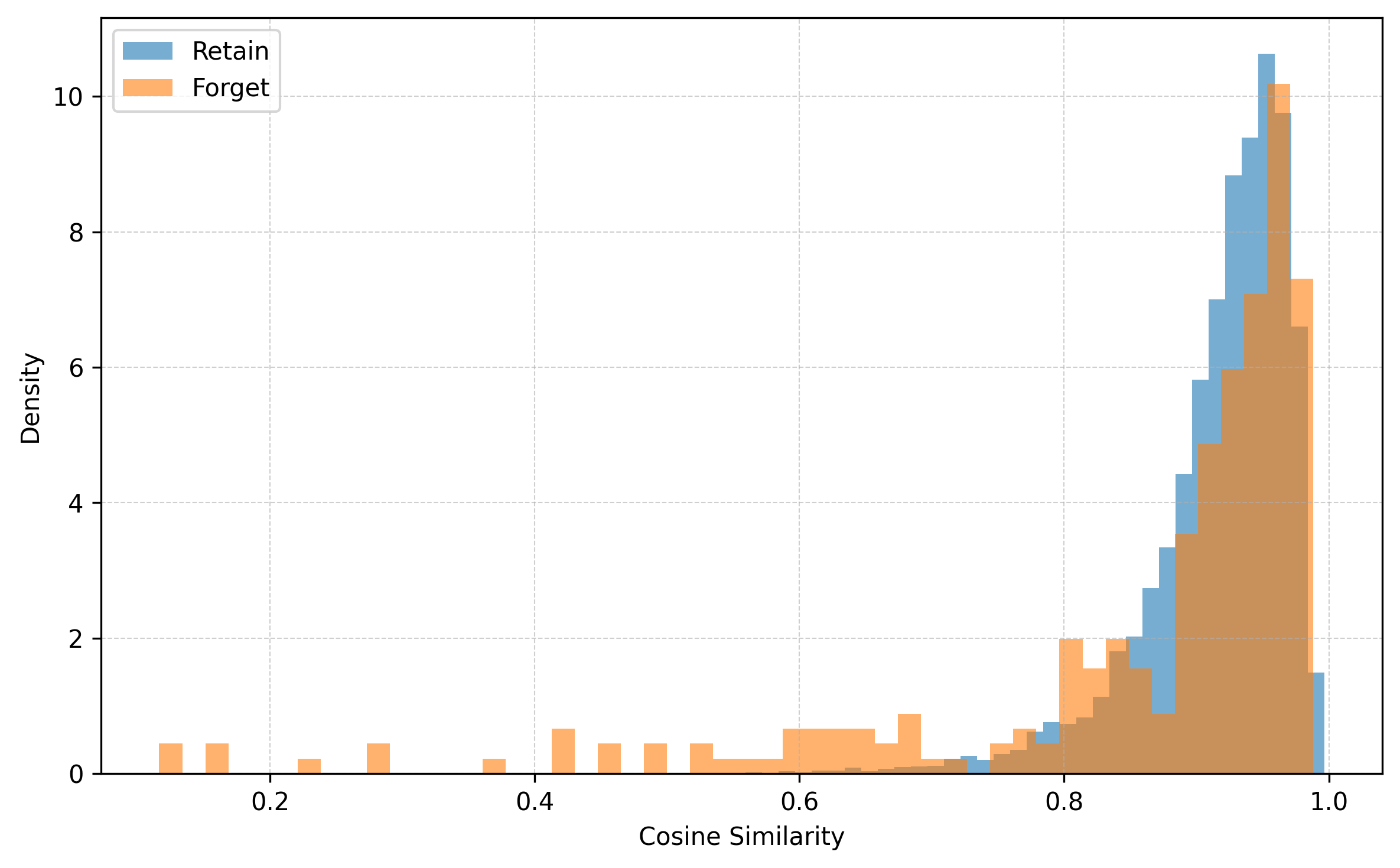} &
\includegraphics[width=0.22\textwidth]{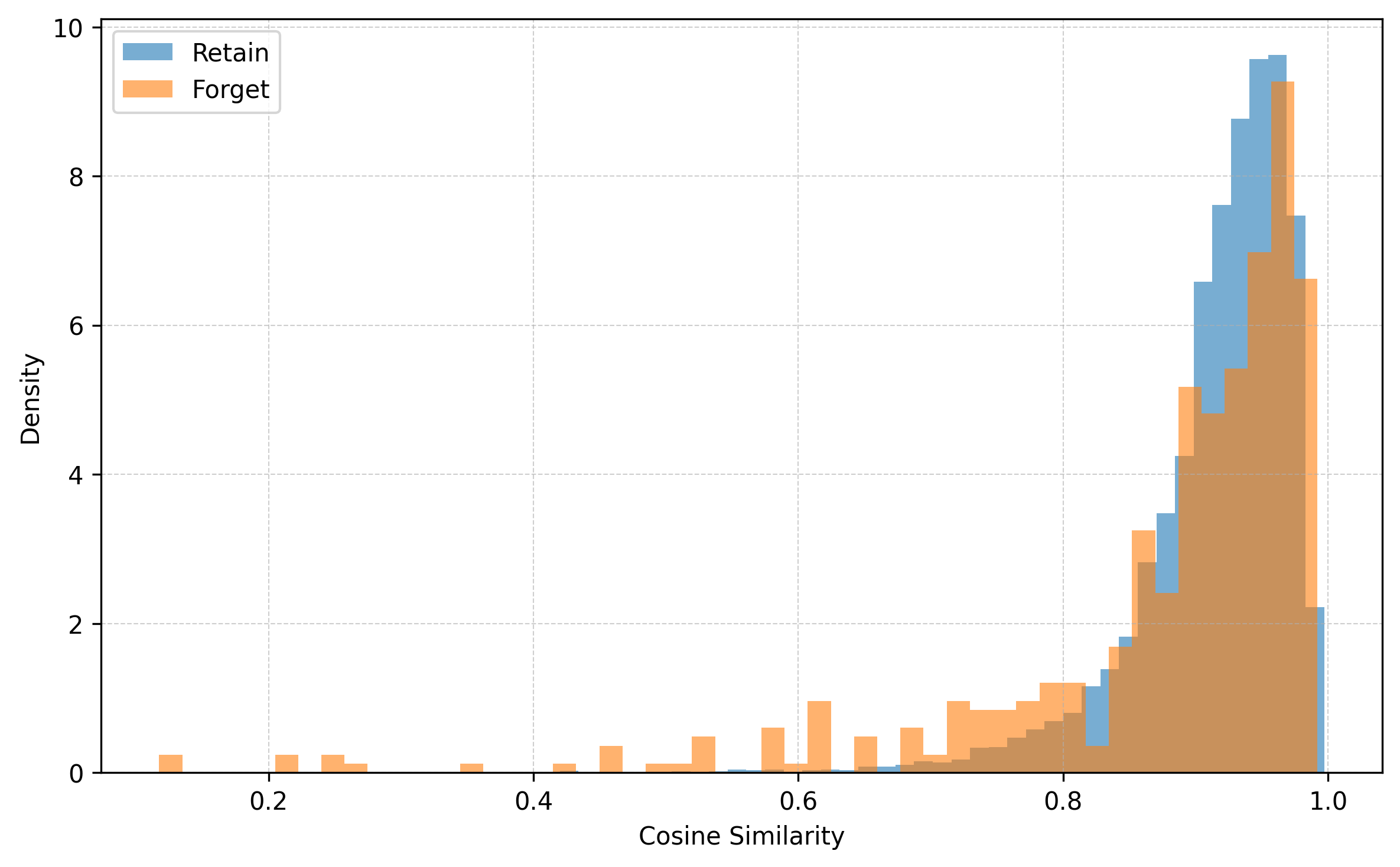} &
\includegraphics[width=0.22\textwidth]{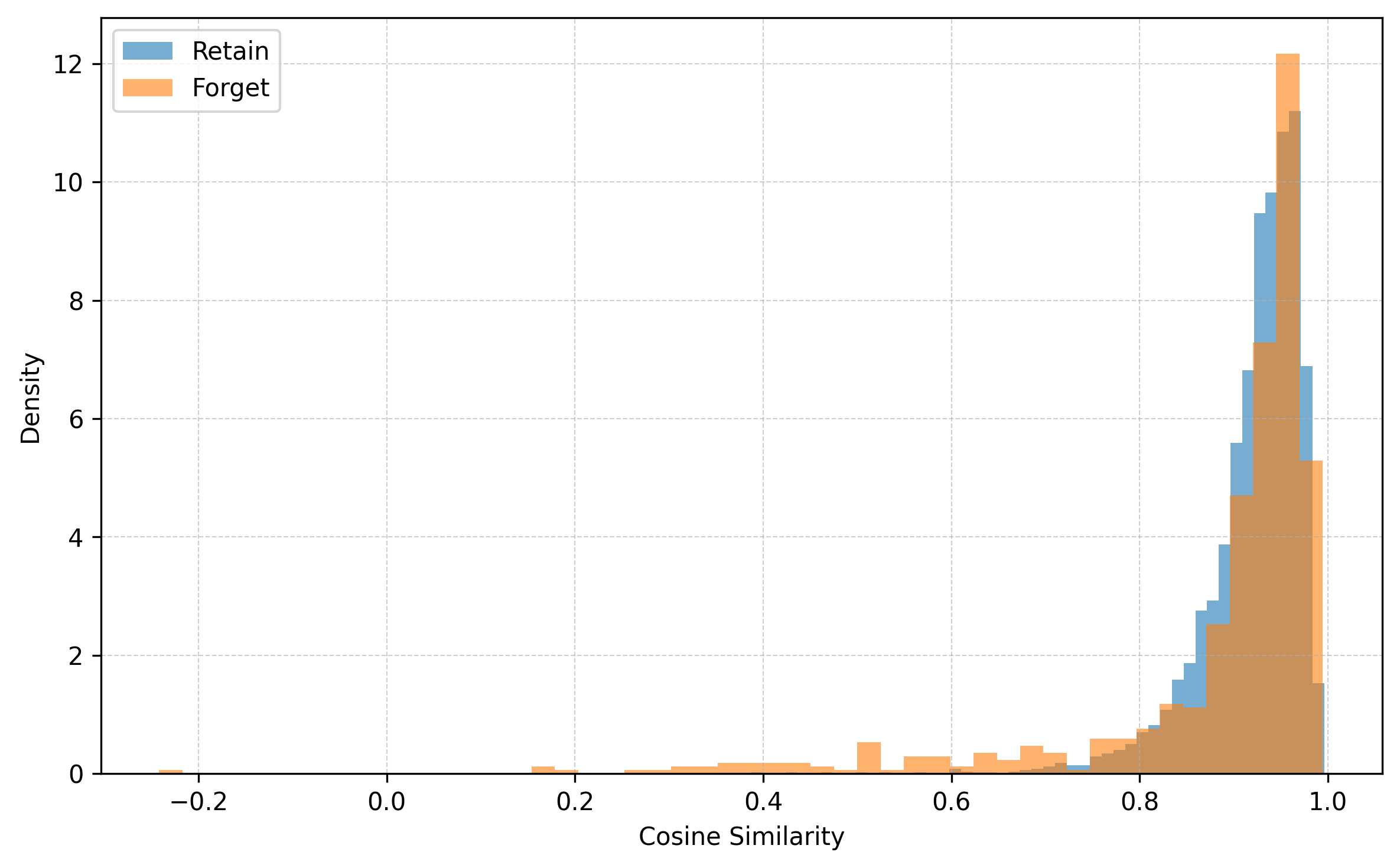} \\
\midrule

\textbf{SCRUB} &
\includegraphics[width=0.22\textwidth]{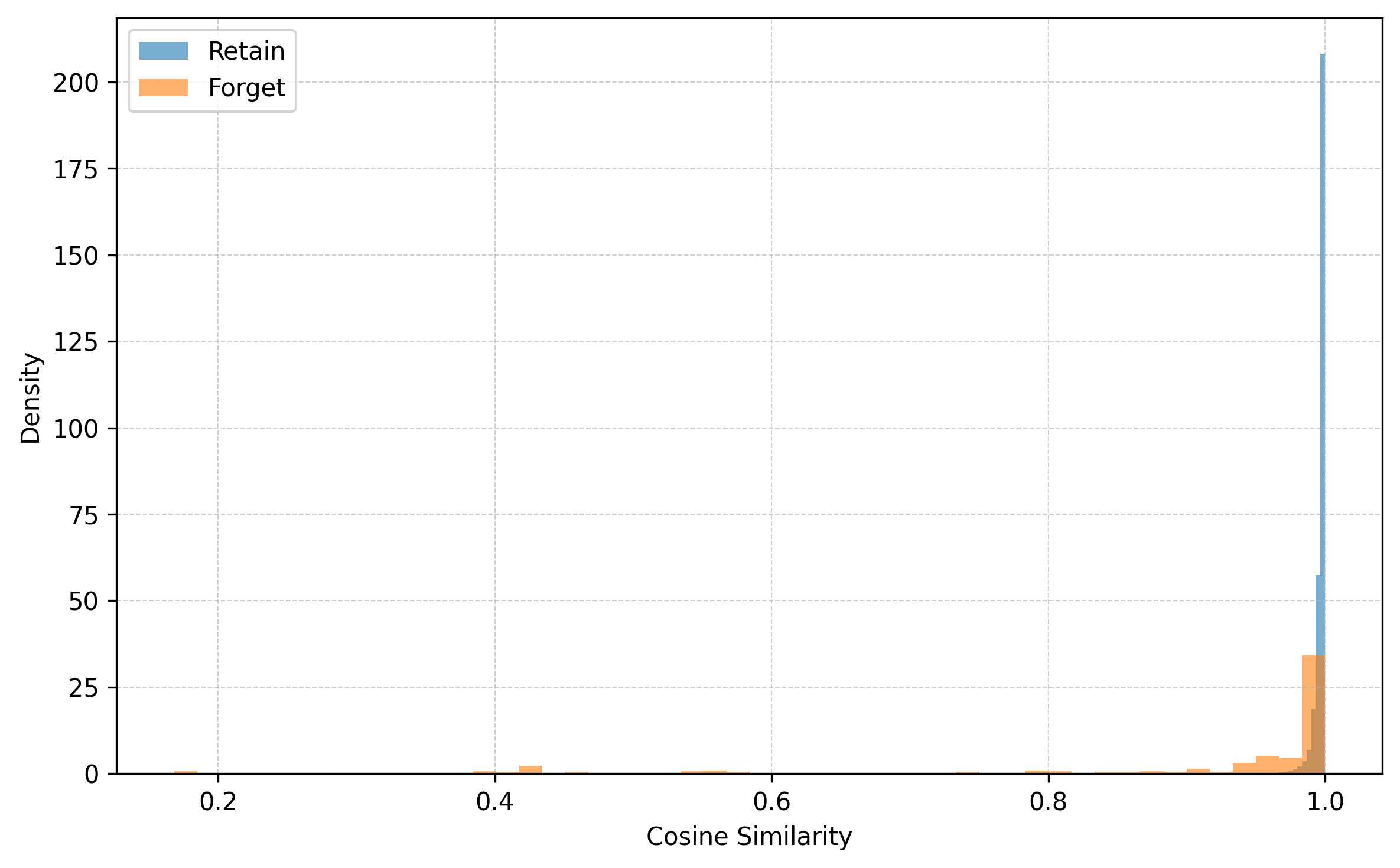} &
\includegraphics[width=0.22\textwidth]{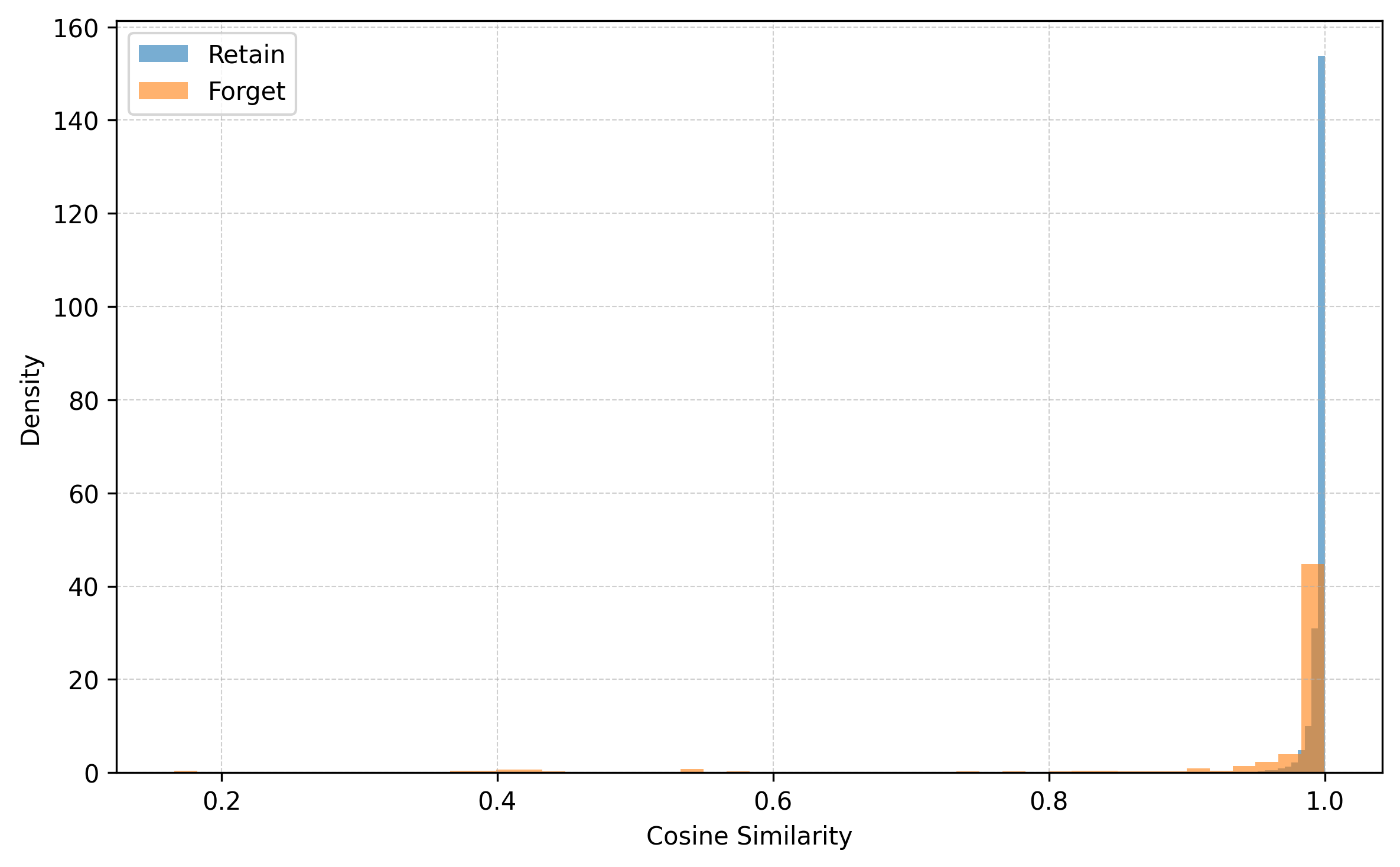} &
\includegraphics[width=0.22\textwidth]{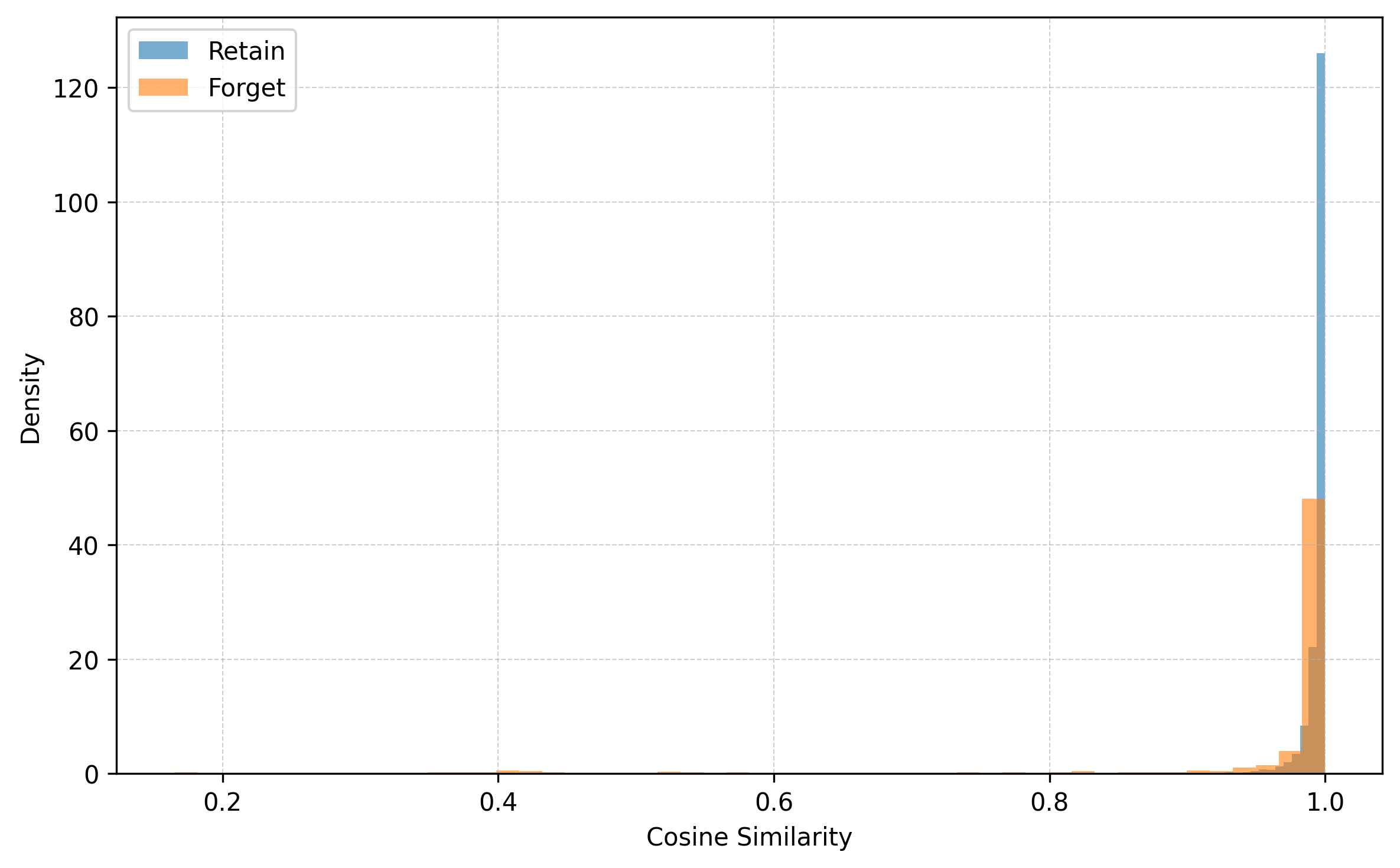} \\
\midrule

\textbf{SALUN} &
\includegraphics[width=0.22\textwidth]{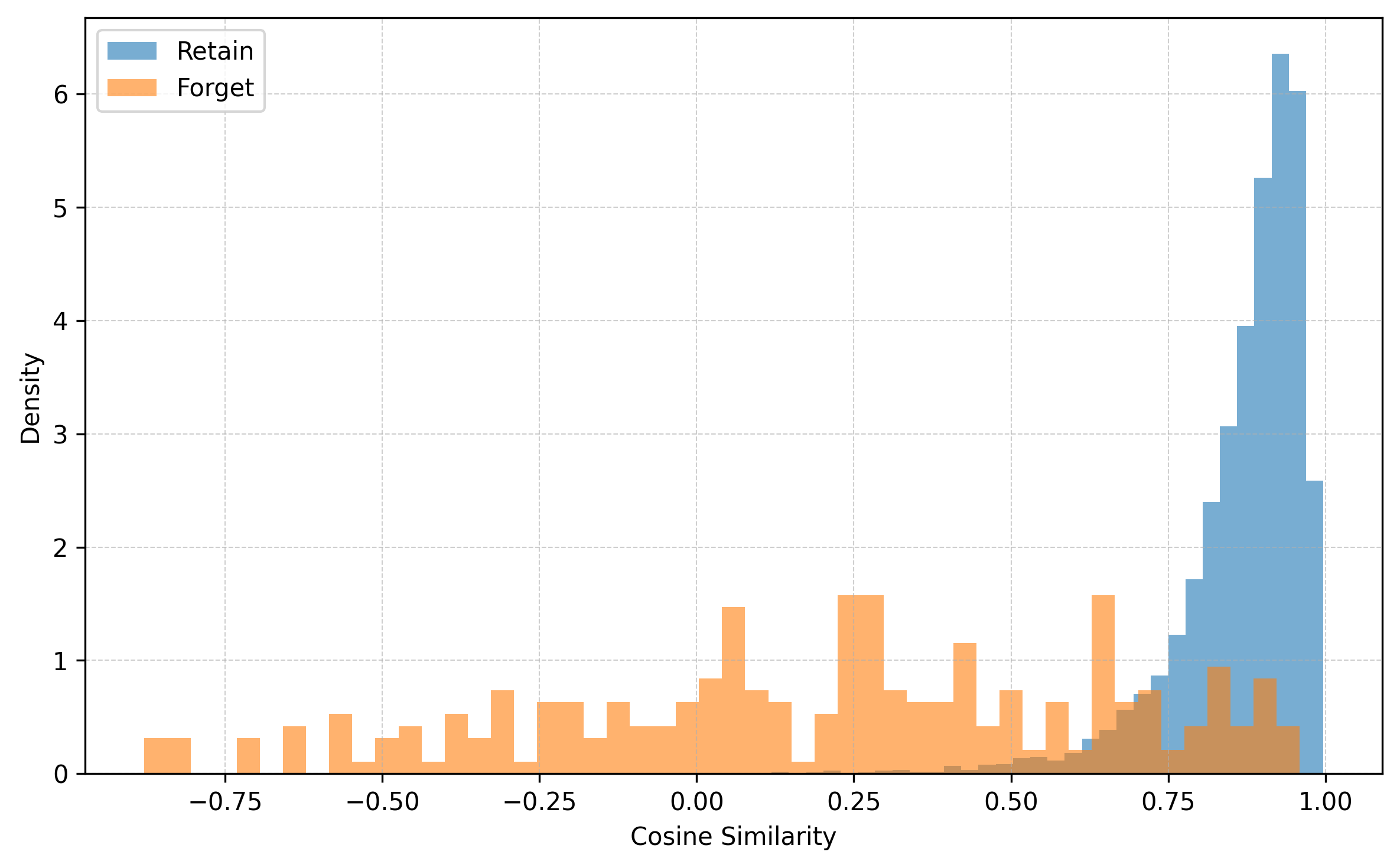} &
\includegraphics[width=0.22\textwidth]{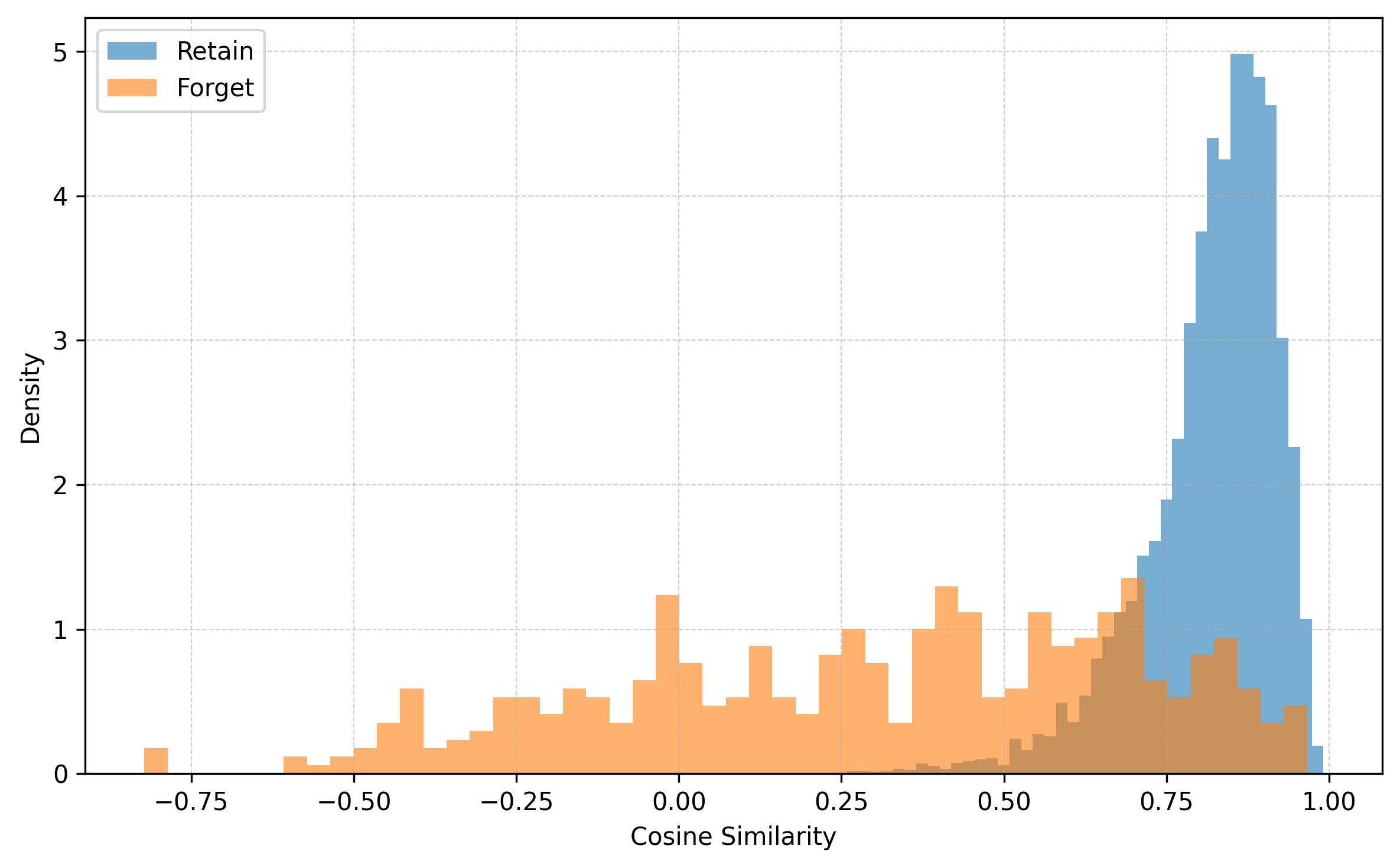} &
\includegraphics[width=0.22\textwidth]{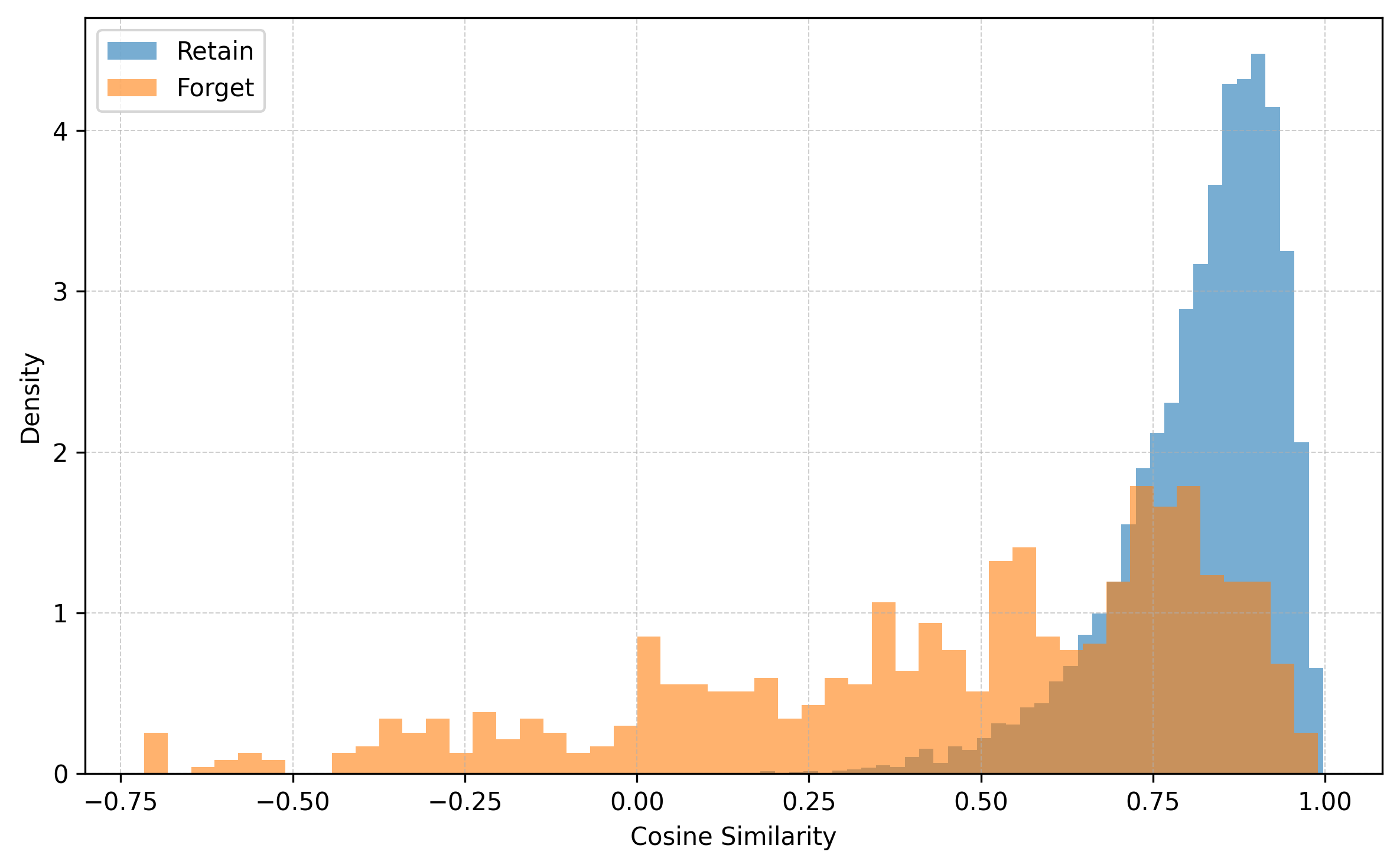} \\
\midrule

\textbf{SSD} &
\includegraphics[width=0.22\textwidth]{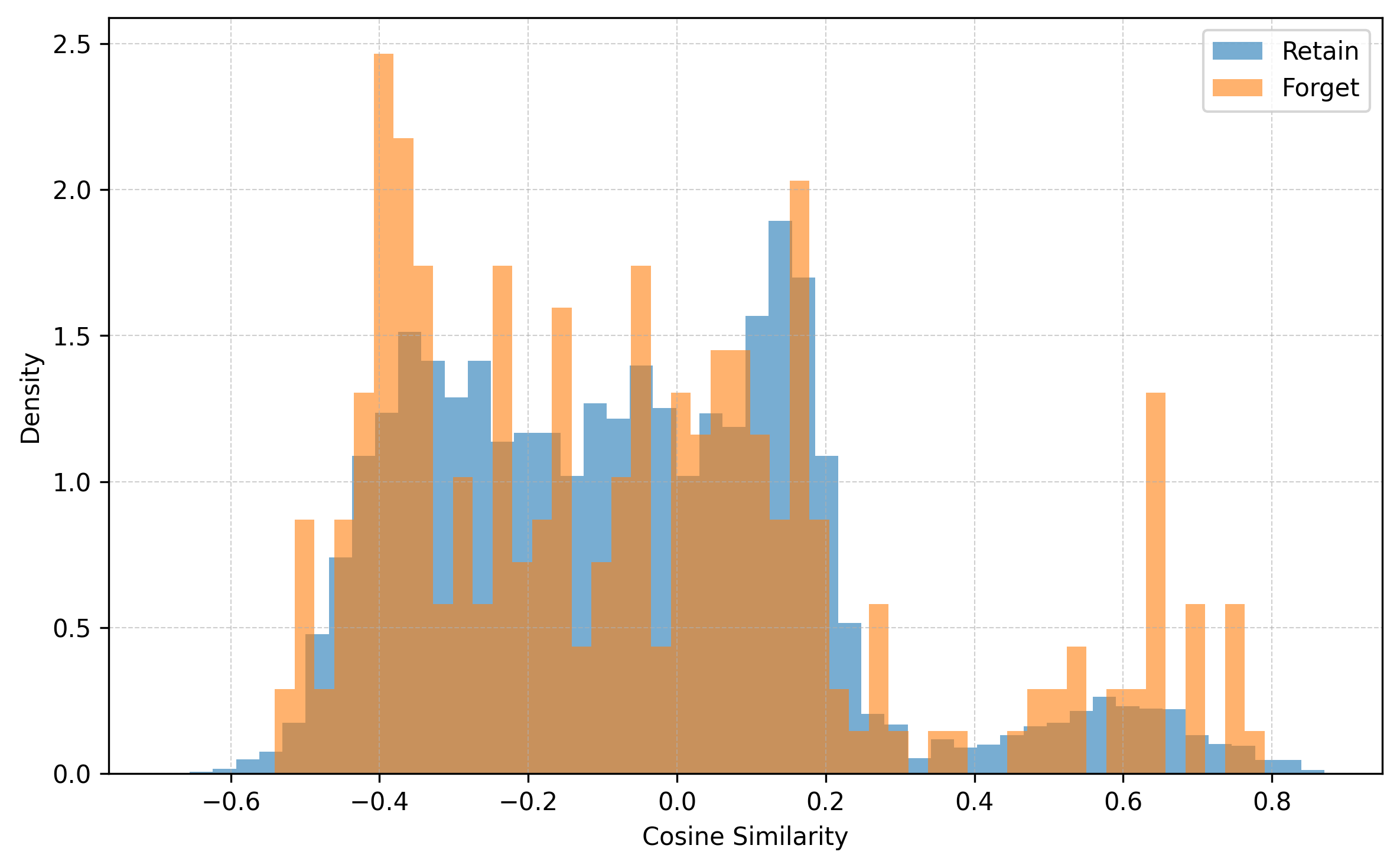} &
\includegraphics[width=0.22\textwidth]{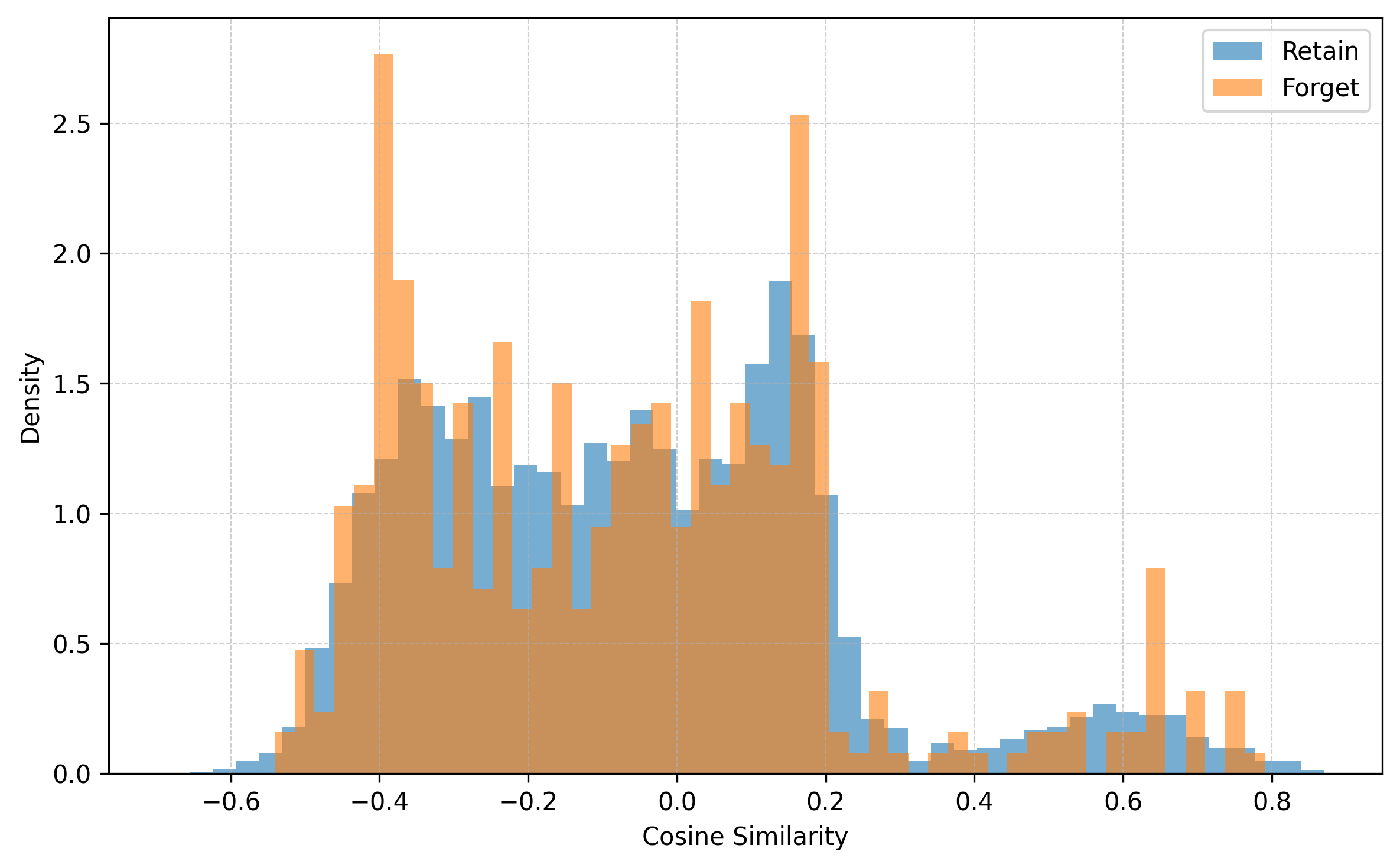} &
\includegraphics[width=0.22\textwidth]{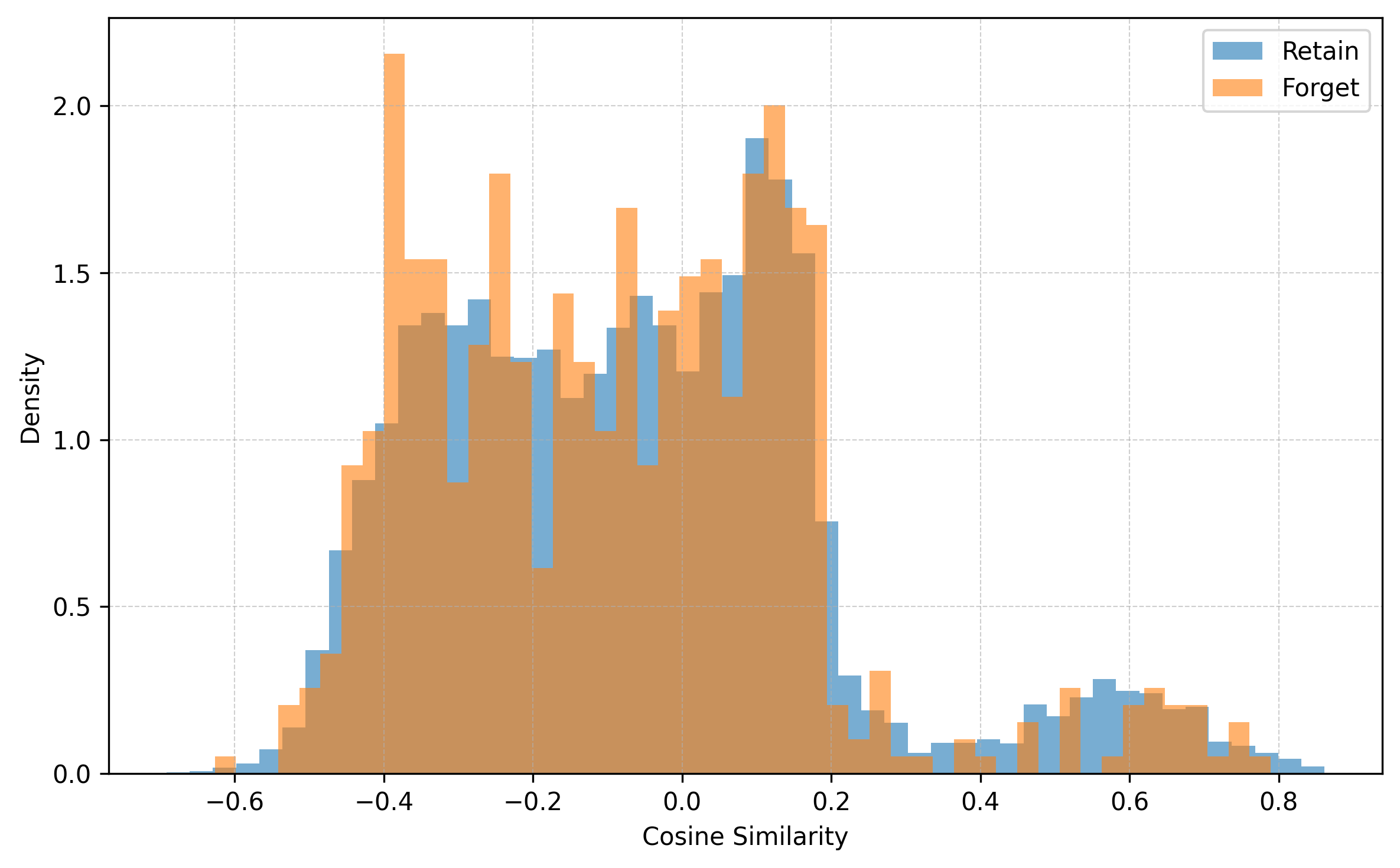} \\
\midrule

\textbf{BndShrink}  &
\includegraphics[width=0.22\textwidth]{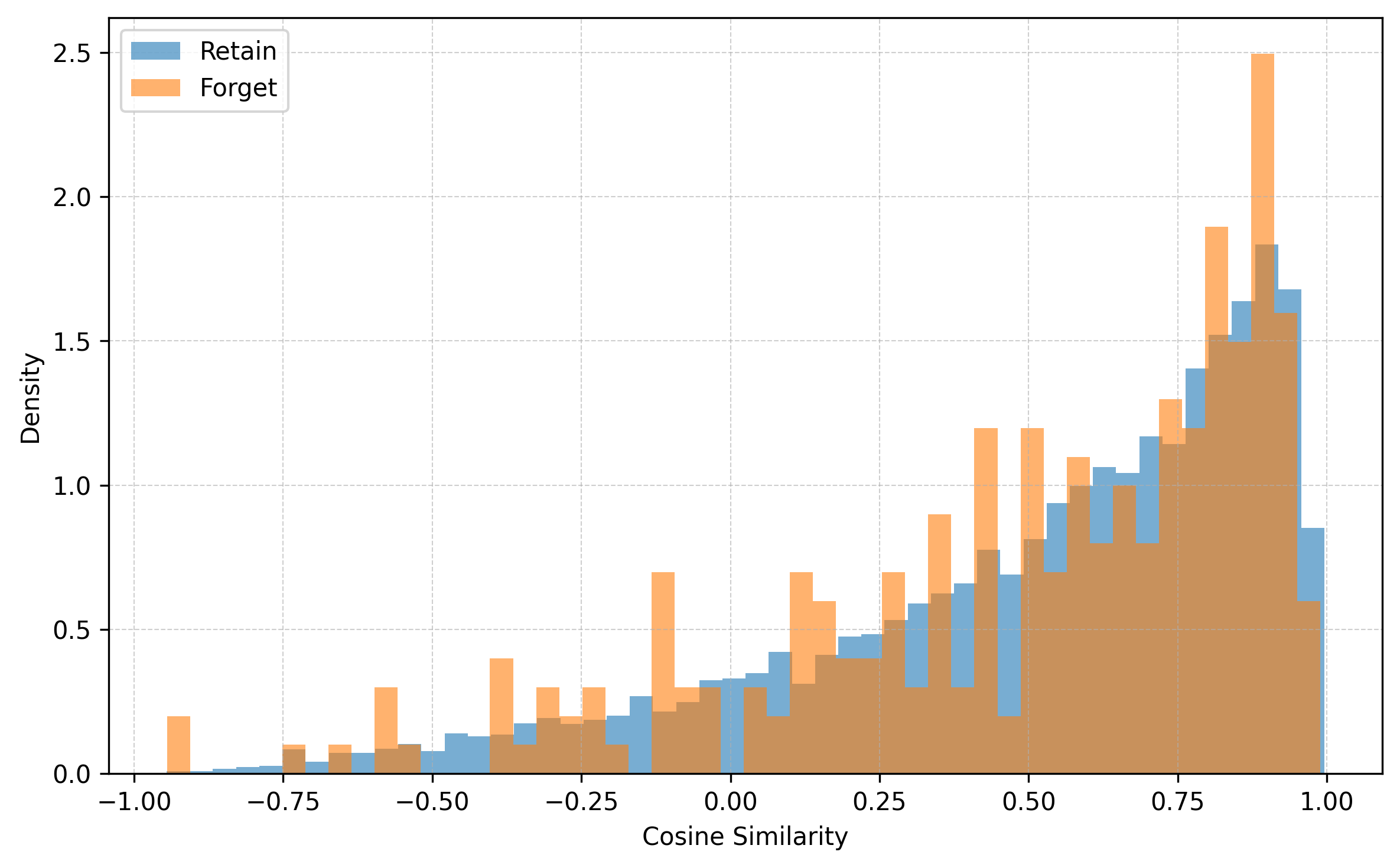} &
\includegraphics[width=0.22\textwidth]{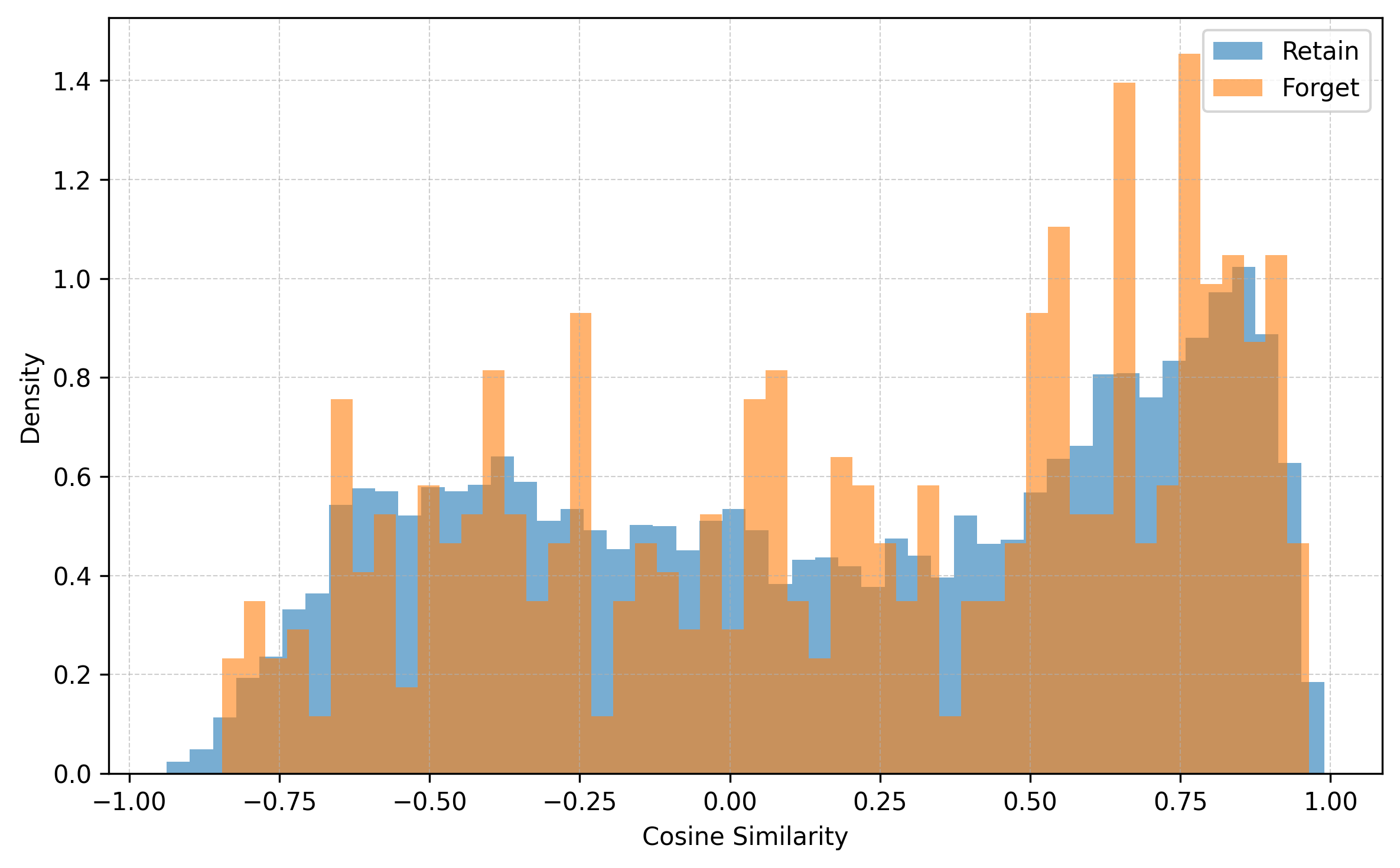} &
\includegraphics[width=0.22\textwidth]{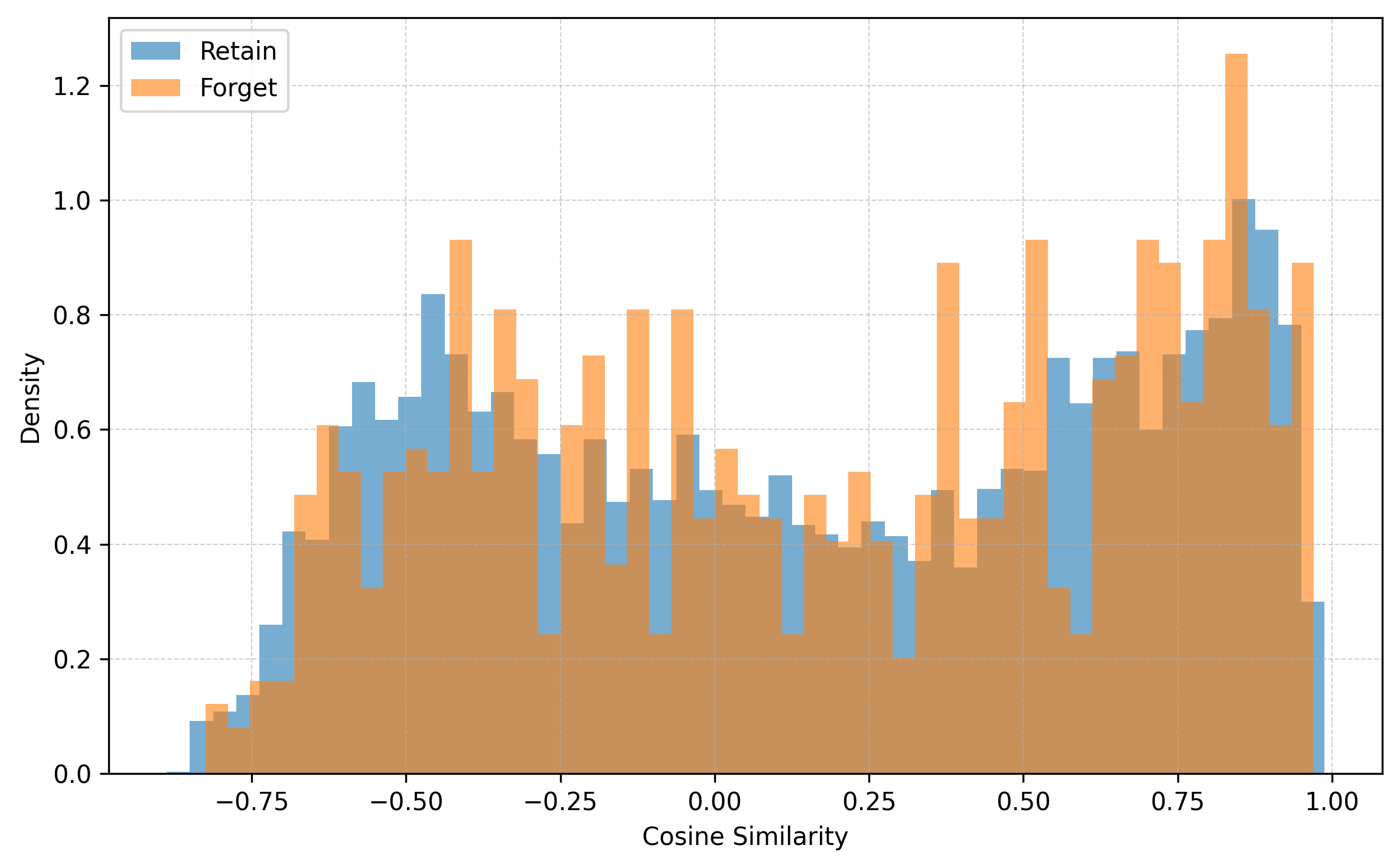} \\
\midrule

\textbf{NegGrad} &
\includegraphics[width=0.22\textwidth]{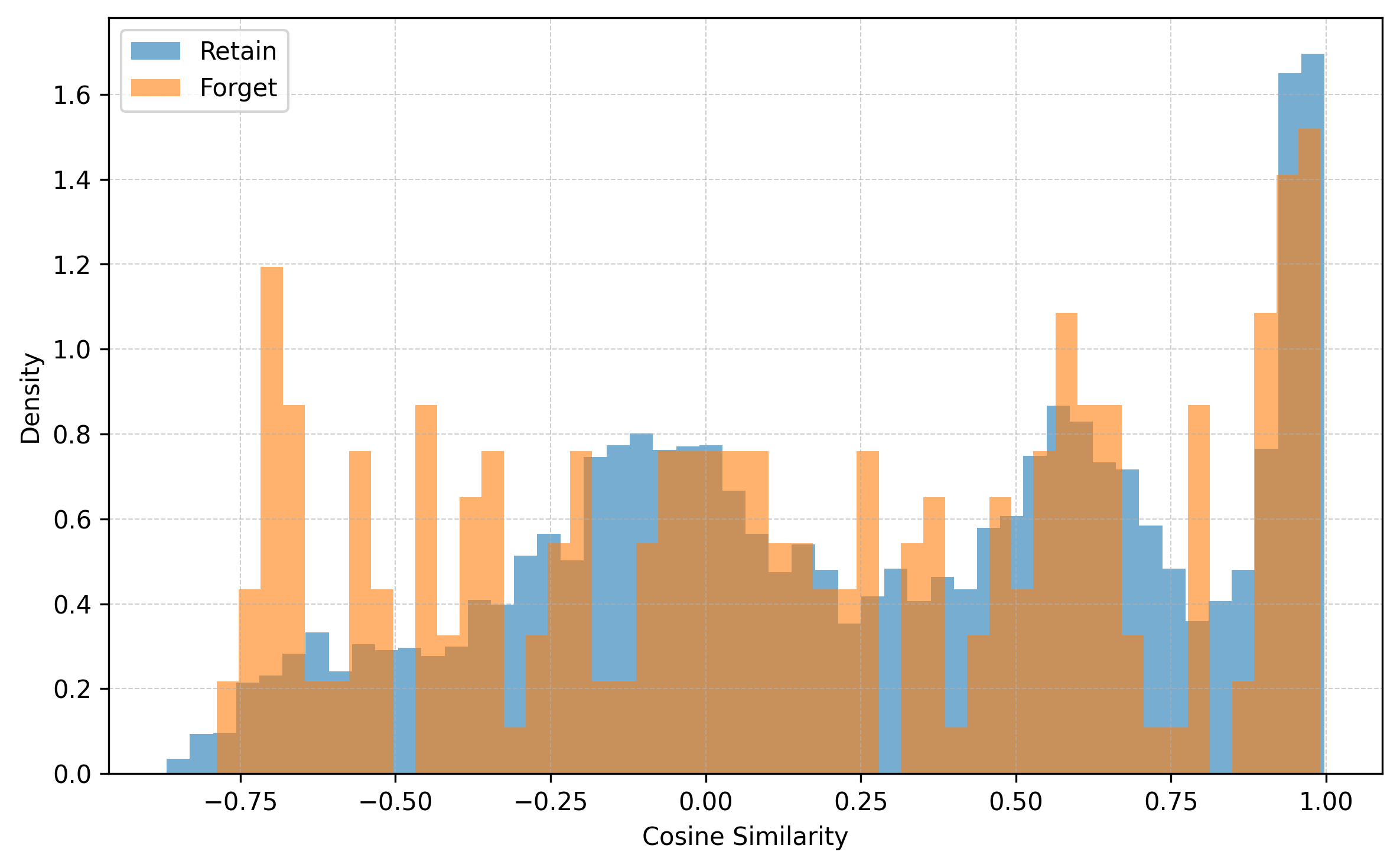} &
\includegraphics[width=0.22\textwidth]{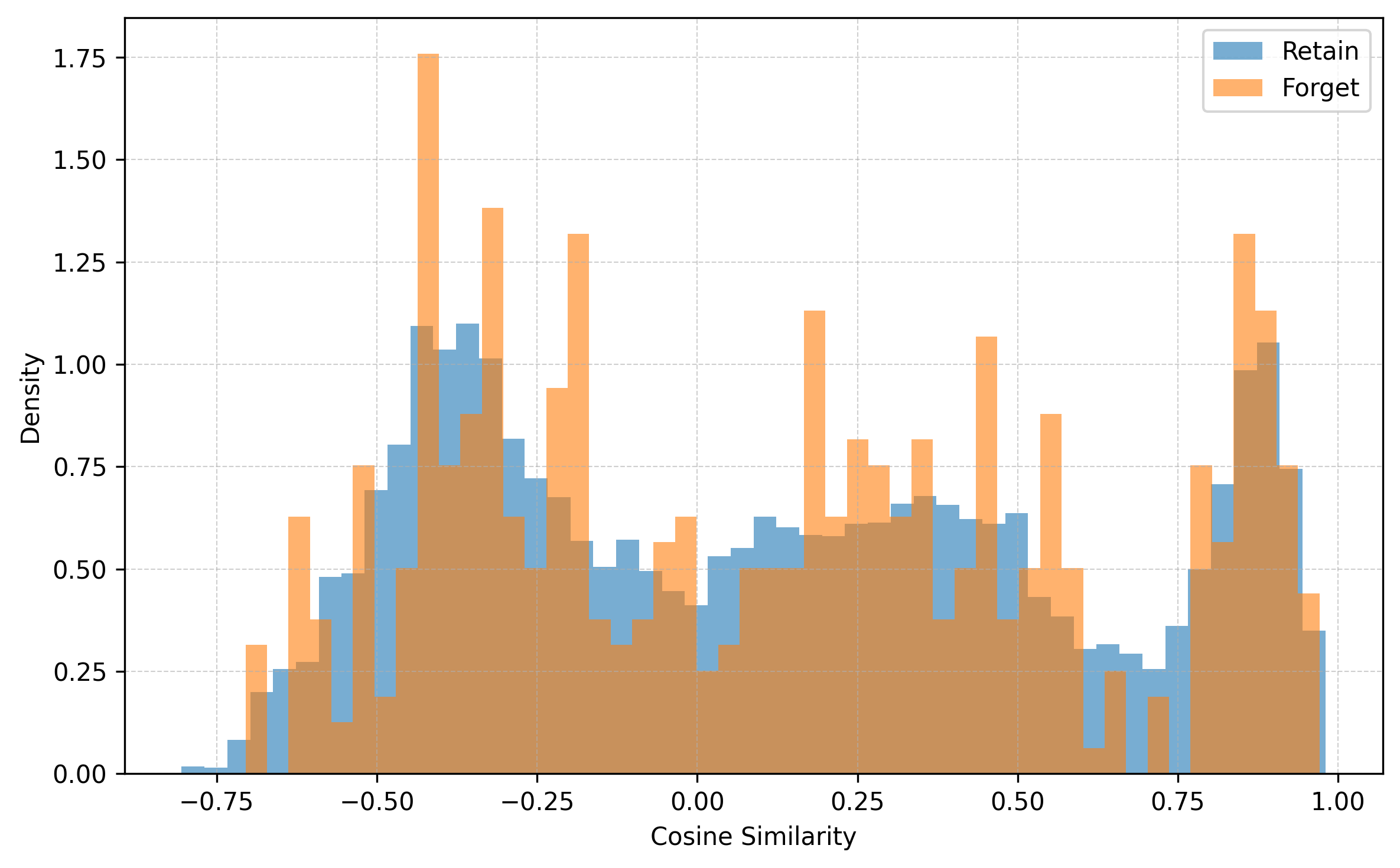} &
\includegraphics[width=0.22\textwidth]{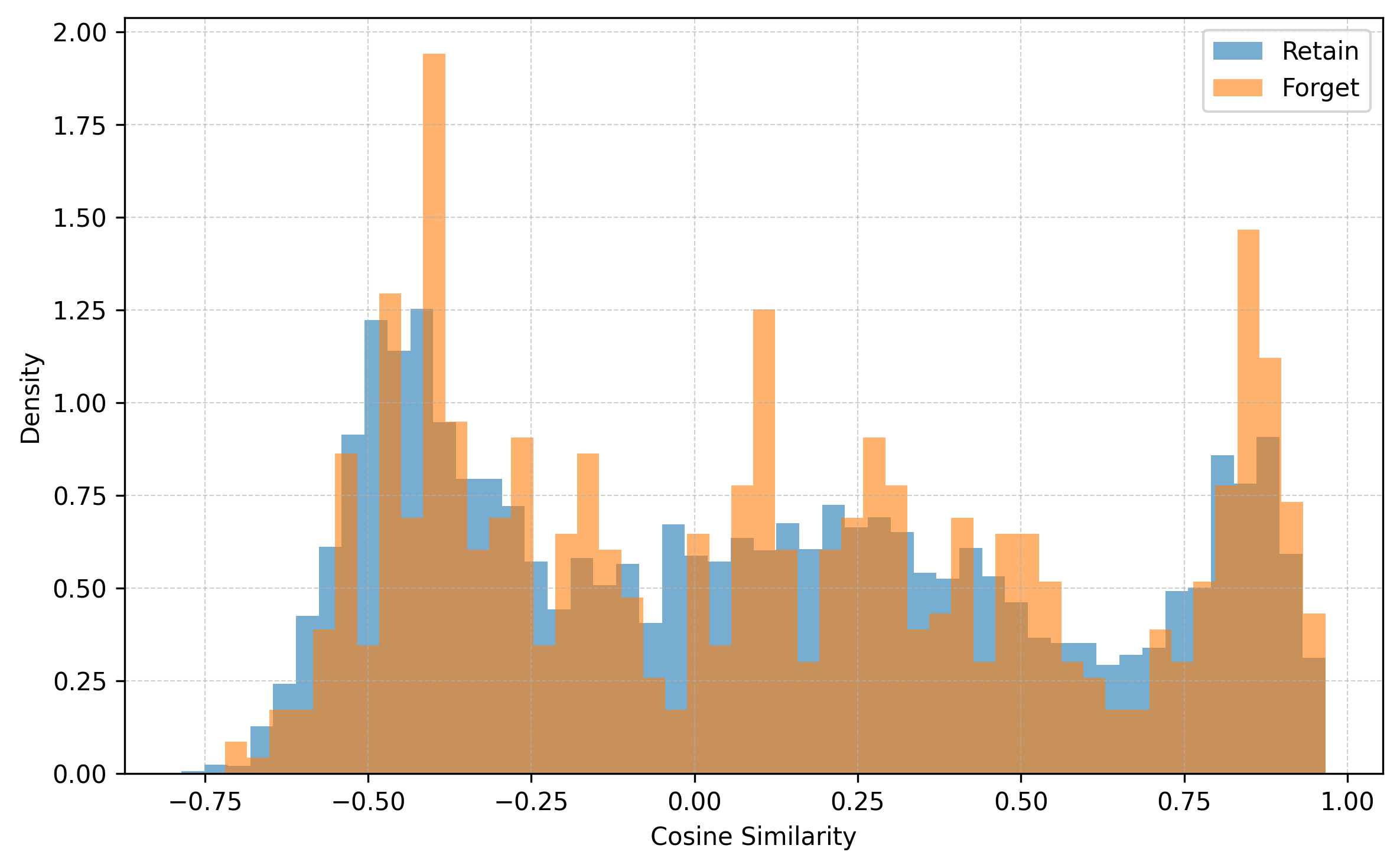} \\
\midrule

\textbf{Finetune} &
\includegraphics[width=0.22\textwidth]{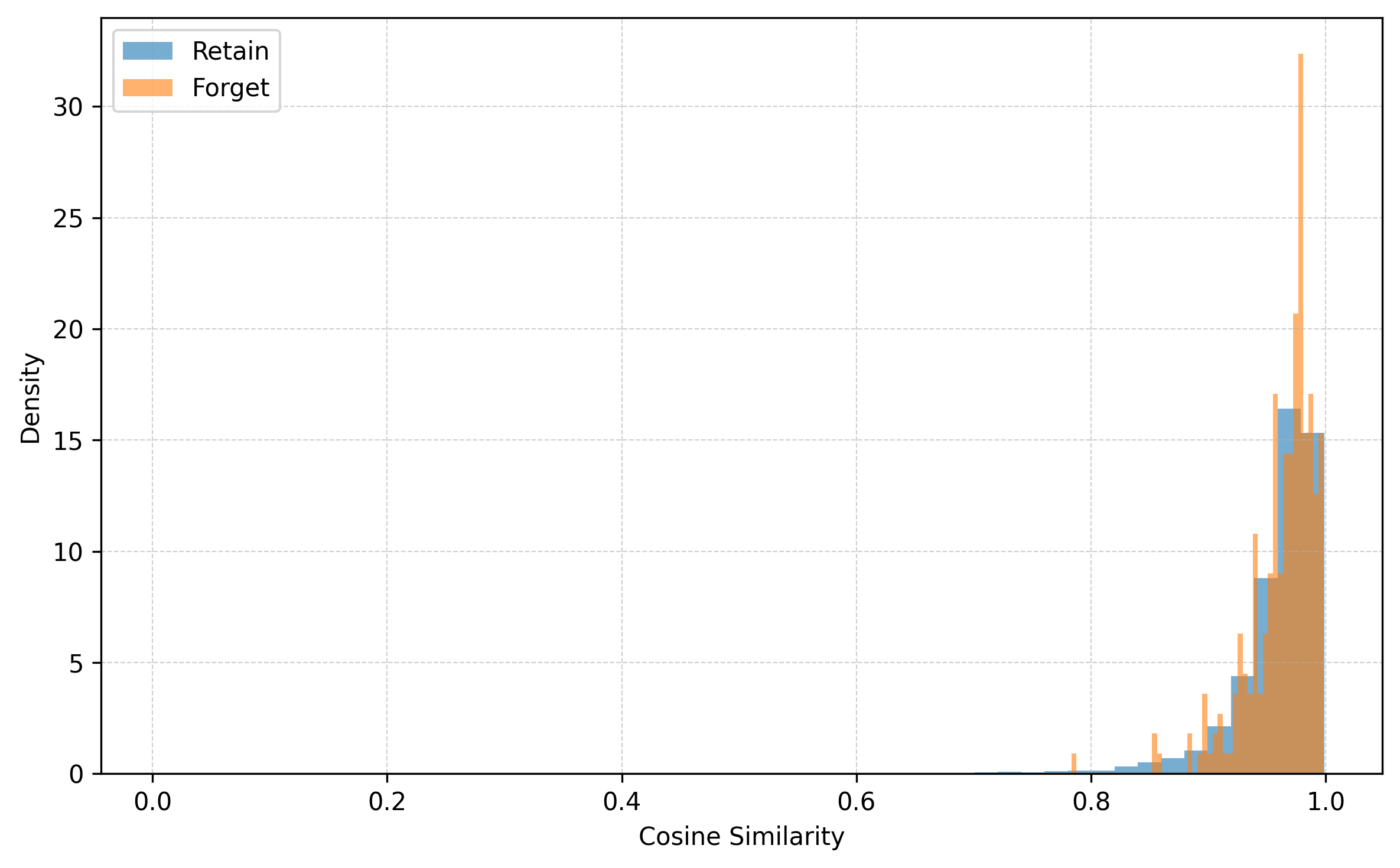} &
\includegraphics[width=0.22\textwidth]{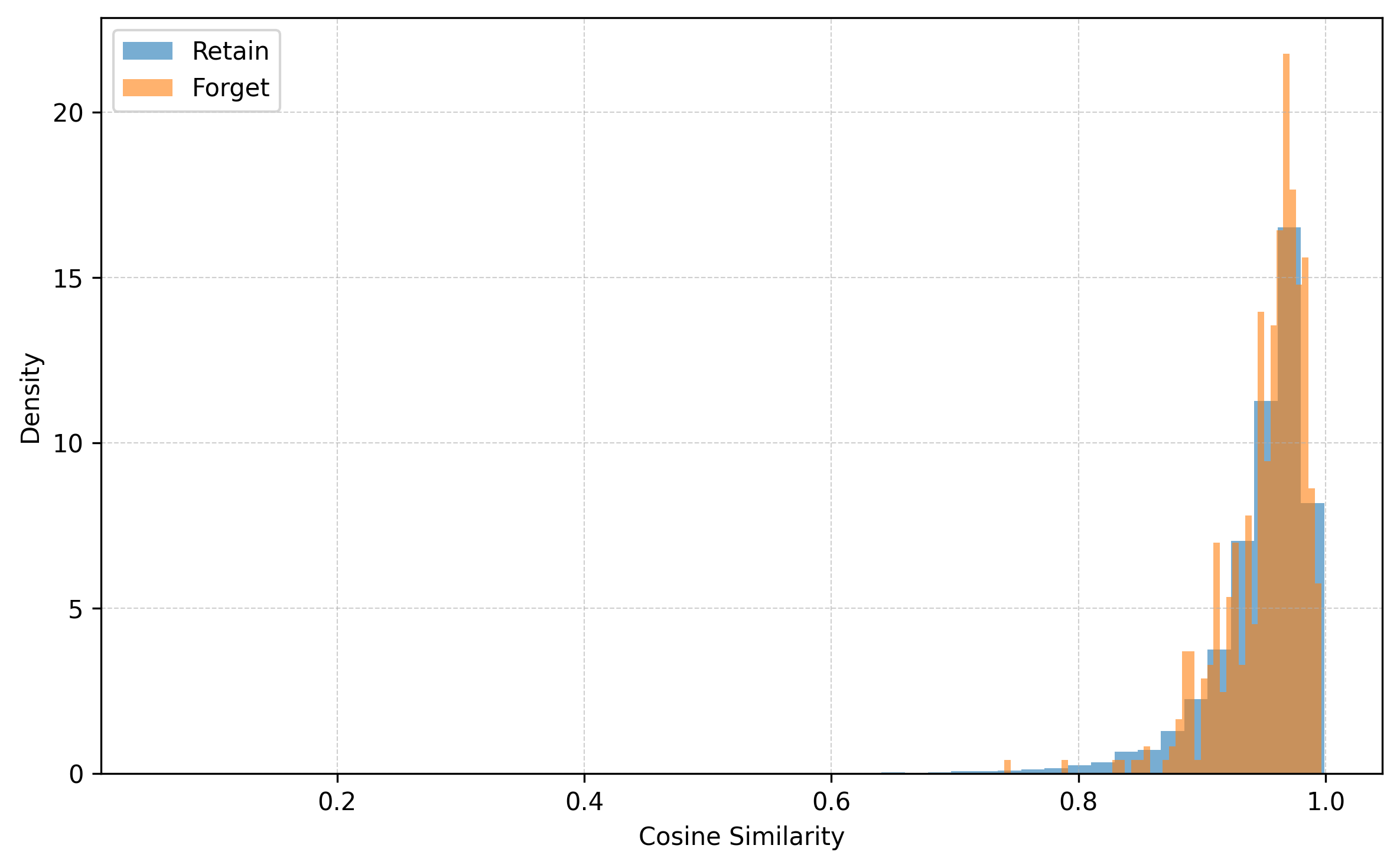} &
\includegraphics[width=0.22\textwidth]{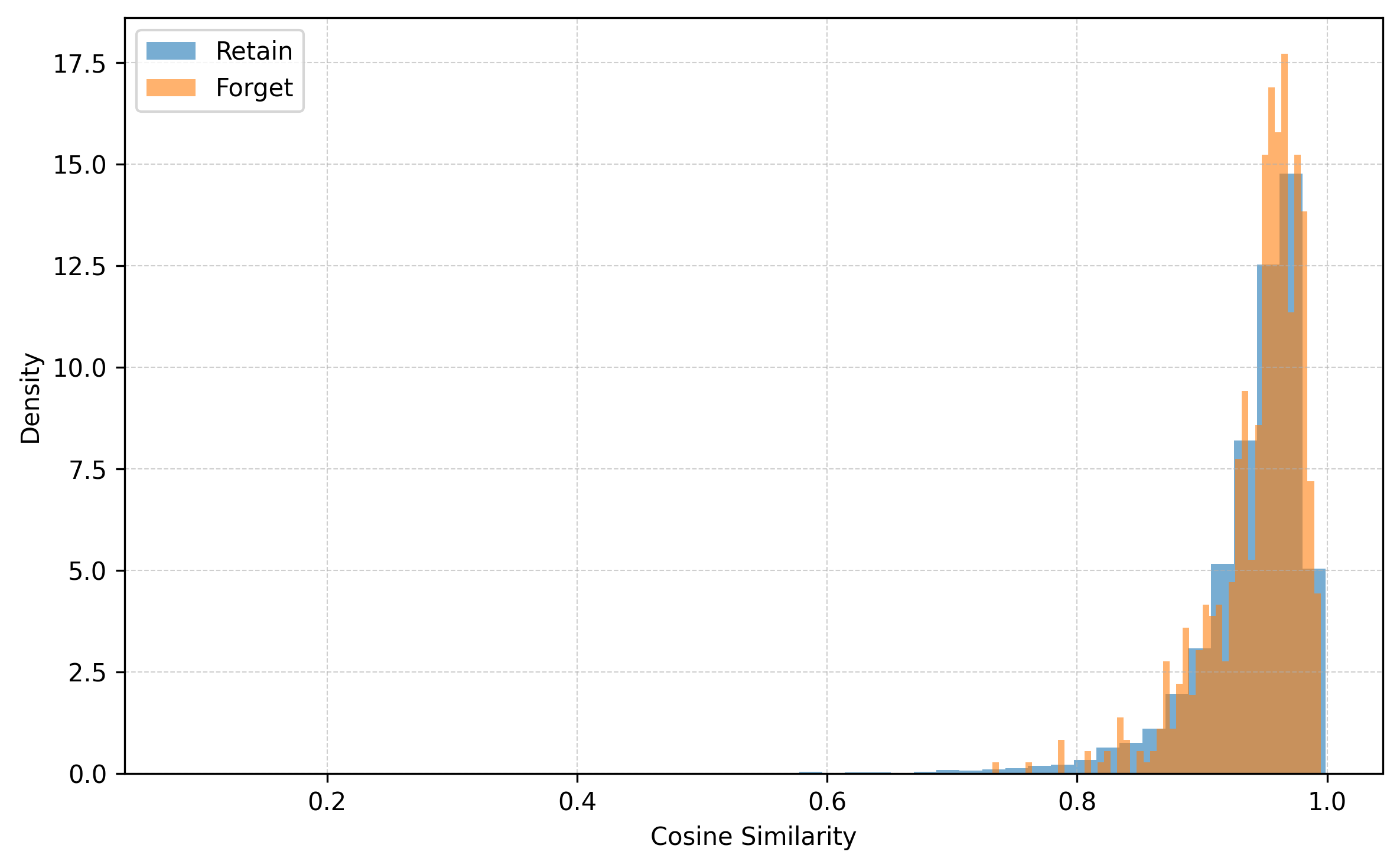} \\
\midrule

\textbf{SAFER} &
\includegraphics[width=0.22\textwidth]{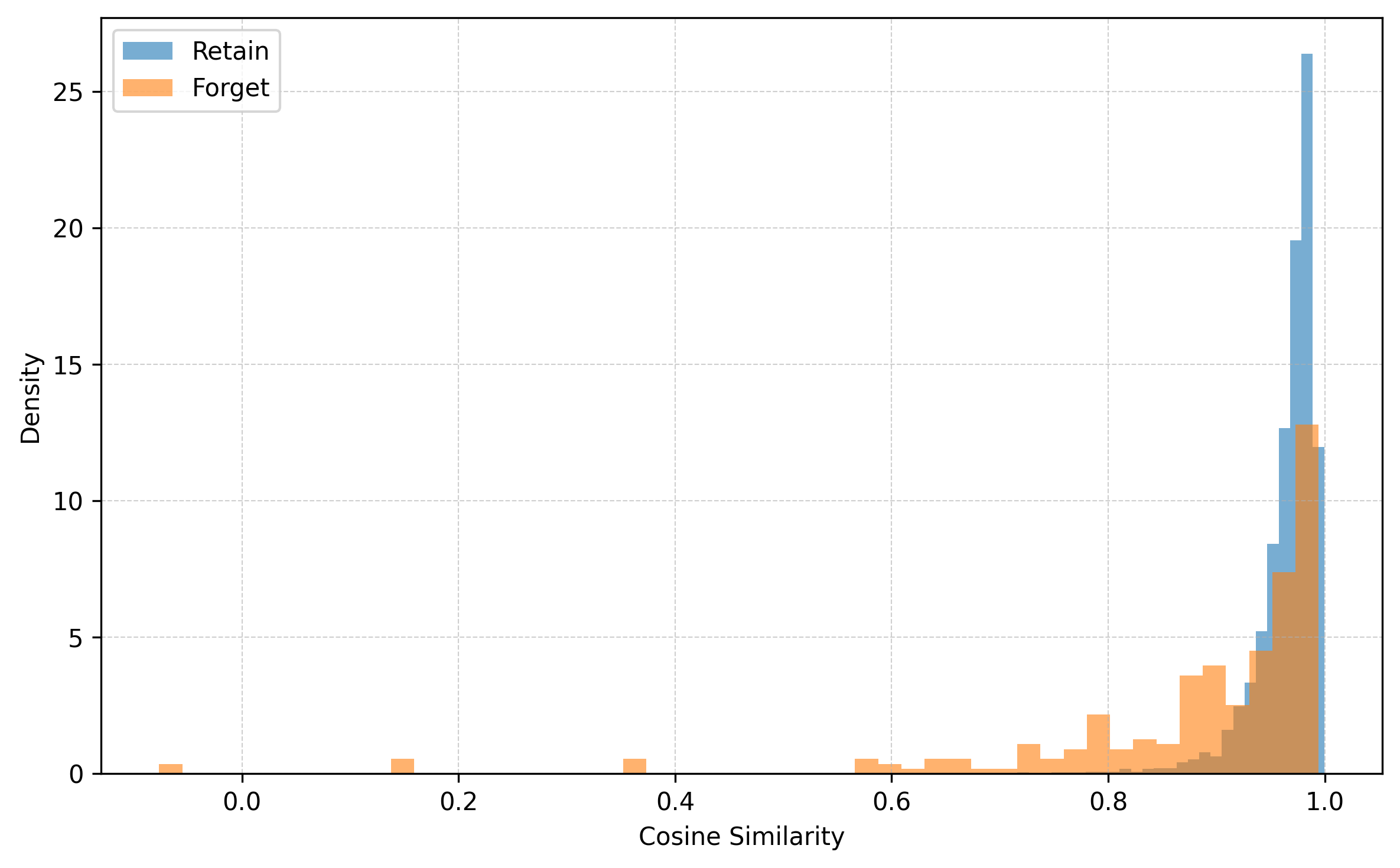} &
\includegraphics[width=0.22\textwidth]{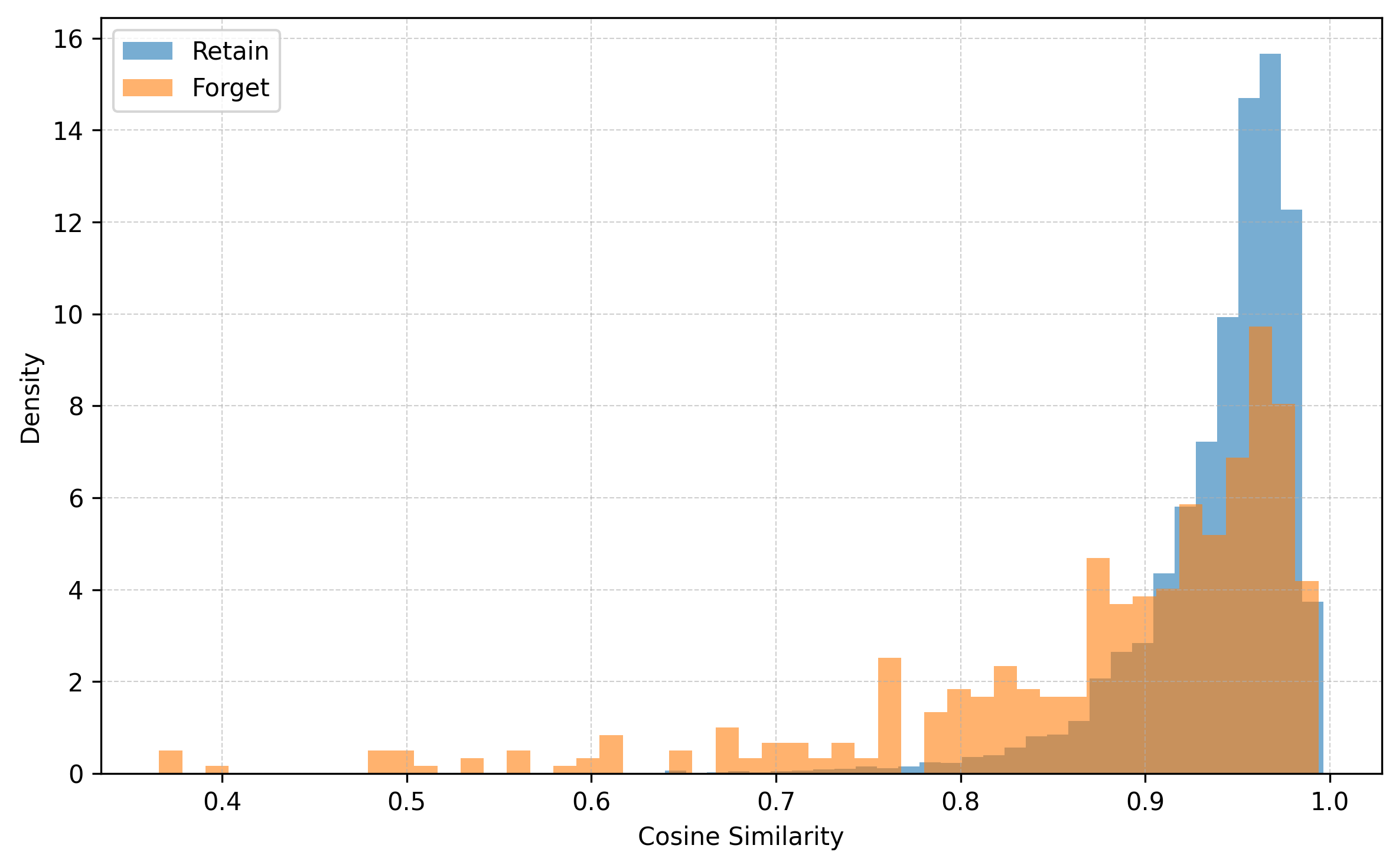} &
\includegraphics[width=0.22\textwidth]{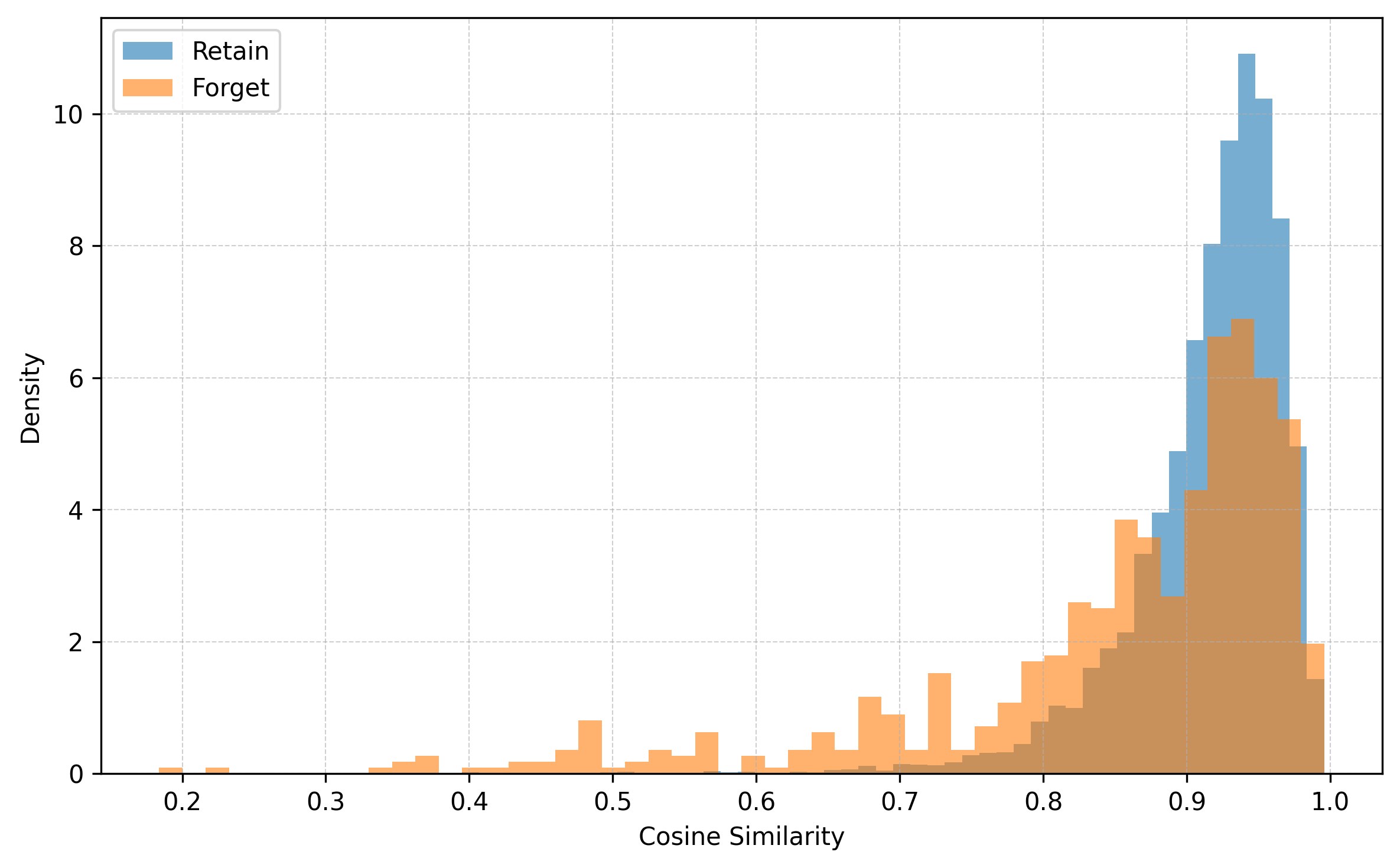} \\

\bottomrule

\end{tabular}
\caption{Representation similarity over the three-phase unlearning process on MUFAC. The orange shows the retain data distribution, while the blue shows the forget data distribution.}
\label{fig:mufac_similarity}
\end{figure*}

%% file: main.bbl
\begin{thebibliography}{31}
        \providecommand{\natexlab}[1]{#1}
        \providecommand{\url}[1]{\texttt{#1}}
        \expandafter\ifx\csname urlstyle\endcsname\relax
          \providecommand{\doi}[1]{doi: #1}\else
          \providecommand{\doi}{doi: \begingroup \urlstyle{rm}\Url}\fi
        
        \bibitem[Adhikari et~al.(2025)Adhikari, Kumaravelu, and Srijith]{adhikari2025unlearning}
        Sayanta Adhikari, Vishnuprasadh Kumaravelu, and PK Srijith.
        \newblock An unlearning framework for continual learning.
        \newblock \emph{arXiv preprint arXiv:2509.17530}, 2025.
        
        \bibitem[Akiba et~al.(2019)Akiba, Sano, Yanase, Ohta, and Koyama]{akiba2019optuna}
        Takuya Akiba, Shotaro Sano, Toshihiko Yanase, Takeru Ohta, and Masanori Koyama.
        \newblock {O}ptuna: A next-generation hyperparameter optimization framework.
        \newblock In \emph{The 25th ACM SIGKDD International Conference on Knowledge Discovery \& Data Mining}, pages 2623--2631, 2019.
        
        \bibitem[Bourtoule et~al.(2021)Bourtoule, Chandrasekaran, Choquette-Choo, Jia, Travers, Zhang, Lie, and Papernot]{bourtoule2021machine}
        Lucas Bourtoule, Varun Chandrasekaran, Christopher~A Choquette-Choo, Hengrui Jia, Adelin Travers, Baiwu Zhang, David Lie, and Nicolas Papernot.
        \newblock Machine unlearning.
        \newblock In \emph{2021 IEEE symposium on security and privacy (SP)}, pages 141--159. IEEE, 2021.
        
        \bibitem[Cao et~al.(2018)Cao, Shen, Xie, Parkhi, and Zisserman]{cao2018vggface2datasetrecognisingfaces}
        Qiong Cao, Li Shen, Weidi Xie, Omkar~M. Parkhi, and Andrew Zisserman.
        \newblock Vggface2: A dataset for recognising faces across pose and age, 2018.
        
        \bibitem[Cao and Yang(2015)]{cao2015towards}
        Yinzhi Cao and Junfeng Yang.
        \newblock Towards making systems forget with machine unlearning.
        \newblock In \emph{2015 IEEE symposium on security and privacy}, pages 463--480. IEEE, 2015.
        
        \bibitem[Carlini et~al.(2021)Carlini, Tramer, Wallace, Jagielski, Herbert-Voss, Lee, Roberts, Brown, Song, Erlingsson, et~al.]{carlini2021extracting}
        Nicholas Carlini, Florian Tramer, Eric Wallace, Matthew Jagielski, Ariel Herbert-Voss, Katherine Lee, Adam Roberts, Tom Brown, Dawn Song, Ulfar Erlingsson, et~al.
        \newblock Extracting training data from large language models.
        \newblock In \emph{30th USENIX security symposium (USENIX Security 21)}, pages 2633--2650, 2021.
        
        \bibitem[Chatterjee et~al.(2024)Chatterjee, Chundawat, Tarun, Mali, and Mandal]{chatterjee2024unified}
        Romit Chatterjee, Vikram Chundawat, Ayush Tarun, Ankur Mali, and Murari Mandal.
        \newblock A unified framework for continual learning and unlearning.
        \newblock \emph{arXiv preprint arXiv:2408.11374}, 2024.
        
        \bibitem[Chen et~al.(2023)Chen, Gao, Liu, Peng, and Wang]{chen2023boundary}
        Min Chen, Weizhuo Gao, Gaoyang Liu, Kai Peng, and Chen Wang.
        \newblock Boundary unlearning: Rapid forgetting of deep networks via shifting the decision boundary.
        \newblock In \emph{Proceedings of the IEEE/CVF Conference on Computer Vision and Pattern Recognition}, pages 7766--7775, 2023.
        
        \bibitem[Choi and Na(2024)]{choi2024towards}
        Dasol Choi and Dongbin Na.
        \newblock Towards machine unlearning benchmarks: Forgetting the personal identities in facial recognition systems.
        \newblock Presented at the AAAI Workshop on Privacy-Preserving Artificial Intelligence (PPAI-24), 2024.
        
        \bibitem[Chundawat et~al.(2023)Chundawat, Tarun, Mandal, and Kankanhalli]{chundawat2023can}
        Vikram~S Chundawat, Ayush~K Tarun, Murari Mandal, and Mohan Kankanhalli.
        \newblock Can bad teaching induce forgetting? unlearning in deep networks using an incompetent teacher.
        \newblock In \emph{Proceedings of the AAAI Conference on Artificial Intelligence}, pages 7210--7217, 2023.
        
        \bibitem[Davies and Bouldin(2009)]{davies2009cluster}
        David~L Davies and Donald~W Bouldin.
        \newblock A cluster separation measure.
        \newblock \emph{IEEE transactions on pattern analysis and machine intelligence}, \penalty0 (2):\penalty0 224--227, 2009.
        
        \bibitem[Deng et~al.(2009)Deng, Dong, Socher, Li, Li, and Fei-Fei]{deng2009imagenet}
        Jia Deng, Wei Dong, Richard Socher, Li-Jia Li, Kai Li, and Li Fei-Fei.
        \newblock Imagenet: A large-scale hierarchical image database.
        \newblock In \emph{2009 IEEE conference on computer vision and pattern recognition}, pages 248--255. Ieee, 2009.
        
        \bibitem[Dosovitskiy et~al.(2021)Dosovitskiy, Beyer, Kolesnikov, Weissenborn, Zhai, Unterthiner, Dehghani, Minderer, Heigold, Gelly, Uszkoreit, and Houlsby]{dosovitskiy2021an}
        Alexey Dosovitskiy, Lucas Beyer, Alexander Kolesnikov, Dirk Weissenborn, Xiaohua Zhai, Thomas Unterthiner, Mostafa Dehghani, Matthias Minderer, Georg Heigold, Sylvain Gelly, Jakob Uszkoreit, and Neil Houlsby.
        \newblock An image is worth 16x16 words: Transformers for image recognition at scale.
        \newblock In \emph{International Conference on Learning Representations}, 2021.
        
        \bibitem[Fan et~al.(2024)Fan, Liu, Zhang, Wong, Wei, and Liu]{fan2024salun}
        Chongyu Fan, Jiancheng Liu, Yihua Zhang, Eric Wong, Dennis Wei, and Sijia Liu.
        \newblock Salun: Empowering machine unlearning via gradient-based weight saliency in both image classification and generation.
        \newblock In \emph{The Twelfth International Conference on Learning Representations}, 2024.
        
        \bibitem[Foster et~al.(2024)Foster, Schoepf, and Brintrup]{foster2024fast}
        Jack Foster, Stefan Schoepf, and Alexandra Brintrup.
        \newblock Fast machine unlearning without retraining through selective synaptic dampening.
        \newblock In \emph{Proceedings of the AAAI Conference on Artificial Intelligence}, pages 12043--12051, 2024.
        
        \bibitem[Golatkar et~al.(2020)Golatkar, Achille, and Soatto]{golatkar2020eternal}
        Aditya Golatkar, Alessandro Achille, and Stefano Soatto.
        \newblock Eternal sunshine of the spotless net: Selective forgetting in deep networks.
        \newblock In \emph{Proceedings of the IEEE/CVF Conference on Computer Vision and Pattern Recognition}, pages 9304--9312, 2020.
        
        \bibitem[Graves et~al.(2021)Graves, Nagisetty, and Ganesh]{graves2021amnesiac}
        Laura Graves, Vineel Nagisetty, and Vijay Ganesh.
        \newblock Amnesiac machine learning.
        \newblock In \emph{Proceedings of the AAAI Conference on Artificial Intelligence}, pages 11516--11524, 2021.
        
        \bibitem[Grimes et~al.(2024)Grimes, Abidi, Frank, and Gallagher]{grimes2024gone}
        Keltin Grimes, Collin Abidi, Cole Frank, and Shannon Gallagher.
        \newblock Gone but not forgotten: Improved benchmarks for machine unlearning.
        \newblock \emph{arXiv preprint arXiv:2405.19211}, 2024.
        
        \bibitem[He et~al.(2016)He, Zhang, Ren, and Sun]{he2016deep}
        Kaiming He, Xiangyu Zhang, Shaoqing Ren, and Jian Sun.
        \newblock Deep residual learning for image recognition.
        \newblock In \emph{Proceedings of the IEEE conference on computer vision and pattern recognition}, pages 770--778, 2016.
        
        \bibitem[Huang et~al.(2025)Huang, Cheng, Zhang, Zheng, Wang, He, Li, and Huang]{huang2025unified}
        Zhehao Huang, Xinwen Cheng, Jie Zhang, Jinghao Zheng, Haoran Wang, Zhengbao He, Tao Li, and Xiaolin Huang.
        \newblock A unified gradient-based framework for task-agnostic continual learning-unlearning.
        \newblock \emph{arXiv preprint arXiv:2505.15178}, 2025.
        
        \bibitem[Krizhevsky et~al.(2009)Krizhevsky, Hinton, et~al.]{krizhevsky2009learning}
        Alex Krizhevsky, Geoffrey Hinton, et~al.
        \newblock Learning multiple layers of features from tiny images.
        \newblock 2009.
        
        \bibitem[Krizhevsky et~al.(2012)Krizhevsky, Sutskever, and Hinton]{krizhevsky2012imagenet}
        Alex Krizhevsky, Ilya Sutskever, and Geoffrey~E Hinton.
        \newblock Imagenet classification with deep convolutional neural networks.
        \newblock \emph{Advances in neural information processing systems}, 25, 2012.
        
        \bibitem[Kurmanji et~al.(2023)Kurmanji, Triantafillou, Hayes, and Triantafillou]{kurmanji2023towards}
        Meghdad Kurmanji, Peter Triantafillou, Jamie Hayes, and Eleni Triantafillou.
        \newblock Towards unbounded machine unlearning.
        \newblock In \emph{Thirty-seventh Conference on Neural Information Processing Systems}, 2023.
        
        \bibitem[Ngnaw{\'e} et~al.(2024)Ngnaw{\'e}, Sahoo, Pequignot, Precioso, and Gagn{\'e}]{ngnawe2024detecting}
        Jonas Ngnaw{\'e}, Sabyasachi Sahoo, Yann Pequignot, Fr{\'e}d{\'e}ric Precioso, and Christian Gagn{\'e}.
        \newblock Detecting brittle decisions for free: Leveraging margin consistency in deep robust classifiers.
        \newblock \emph{Advances in Neural Information Processing Systems}, 37:\penalty0 23301--23324, 2024.
        
        \bibitem[Tang et~al.(2025)Tang, Zhuang, Fang, Li, Han, Huang, Fan, Wang, Zhu, Zhang, et~al.]{tang2025acu}
        Jianheng Tang, Huiping Zhuang, Di Fang, Jiaxu Li, Feijiang Han, Yajiang Huang, Kejia Fan, Leye Wang, Zhanxing Zhu, Shanghang Zhang, et~al.
        \newblock Acu: Analytic continual unlearning for efficient and exact forgetting with privacy preservation.
        \newblock \emph{arXiv preprint arXiv:2505.12239}, 2025.
        
        \bibitem[Voigt and Bussche(2017)]{GDPR}
        Paul Voigt and Axel Bussche.
        \newblock \emph{The EU General Data Protection Regulation (GDPR): A Practical Guide}.
        \newblock 2017.
        
        \bibitem[Xu(2024)]{xu2024machine}
        Miao Xu.
        \newblock Machine unlearning: challenges in data quality and access.
        \newblock In \emph{Proceedings of the Thirty-Third International Joint Conference on Artificial Intelligence (IJCAI-24)}, 2024.
        
        \bibitem[Zhang et~al.(2025)Zhang, Shen, Chen, and Xu]{zhang2025toward}
        Chenhao Zhang, Shaofei Shen, Weitong Chen, and Miao Xu.
        \newblock Toward efficient data-free unlearning.
        \newblock In \emph{Proceedings of the AAAI Conference on Artificial Intelligence}, pages 22372--22379, 2025.
        
        \bibitem[Zhao et~al.(2024{\natexlab{a}})Zhao, Ni, Fan, Wang, Chen, Meng, and Zhang]{Zhao_2024_CVPR}
        Hongbo Zhao, Bolin Ni, Junsong Fan, Yuxi Wang, Yuntao Chen, Gaofeng Meng, and Zhaoxiang Zhang.
        \newblock Continual forgetting for pre-trained vision models.
        \newblock In \emph{Proceedings of the IEEE/CVF Conference on Computer Vision and Pattern Recognition (CVPR)}, pages 28631--28642, 2024{\natexlab{a}}.
        
        \bibitem[Zhao et~al.(2026)Zhao, Zhu, Ni, Zhu, Meng, and Zhang]{zhao2026practical}
        Hongbo Zhao, Fei Zhu, Bolin Ni, Feng Zhu, Gaofeng Meng, and Zhaoxiang Zhang.
        \newblock Practical continual forgetting for pre-trained vision models.
        \newblock \emph{IEEE Transactions on Pattern Analysis and Machine Intelligence}, 2026.
        
        \bibitem[Zhao et~al.(2024{\natexlab{b}})Zhao, Kurmanji, B{\u{a}}rbulescu, Triantafillou, and Triantafillou]{zhao2024makes}
        Kairan Zhao, Meghdad Kurmanji, George-Octavian B{\u{a}}rbulescu, Eleni Triantafillou, and Peter Triantafillou.
        \newblock What makes unlearning hard and what to do about it.
        \newblock \emph{Advances in Neural Information Processing Systems}, 37:\penalty0 12293--12333, 2024{\natexlab{b}}.
        
        \end{thebibliography}
